\documentclass{article}
\usepackage[utf8]{inputenc}

\usepackage{amsmath}
\usepackage{amsthm}
\usepackage{amsfonts}
\usepackage{amssymb}
\usepackage{graphicx}
\usepackage{mathrsfs}
\usepackage{enumerate}
\usepackage{a4wide}
\usepackage{bm, xcolor}
\usepackage{bbm}
\usepackage{soul}
\usepackage{natbib}
\usepackage[bookmarks=false]{hyperref}

\newcommand{\DD}{{\mathcal D}}
\newcommand{\EE}{{\mathbb E}}
\newcommand{\NN}{{\mathbb N}}
\newcommand{\RR}{{\mathbb R}}
\newcommand{\XX}{{\mathcal X}}
\newcommand{\YY}{{\mathcal Y}}
\newcommand{\ZZ}{{\mathcal Z}}
\newcommand{\Ho}{\ell^1}
\newcommand{\BB}{\mathcal B}
\newcommand{\HH}{\mathcal H}
\newcommand{\Ht}{\mathcal H}
\newcommand{\Hu}{\mathcal{H}_\mu}
\newcommand{\LL}{L^2(\XX,\mu;\YY)}
\def\PP{\mathcal P}
\def\pp{\mathfrak q}
\def\xx{\mathbf x}
\def\zz{\mathbf z}
\def\yy{\mathbf y}
\newcommand{\kk}{{\mathbf{k}}}

\newcommand{\fr}{f^\dagger}
\newcommand{\la}{\lambda}

\newcommand{\ella}{L_A}
\newcommand{\fz}{f_{\zz,\la}}
\newcommand{\fp}{f_\rho}
\newcommand{\gp}{g_\rho}
\newcommand{\sx}{S_\xx}
\newcommand{\tx}{T_\xx}
\newcommand{\ip }{S_\mu}
\newcommand{\tp}{T_\mu}
\newcommand{\VE}{\boldsymbol{\varepsilon}}
\newcommand{\tr}{\operatorname{tr}}

\newcommand{\abs}[1]{{\left\lvert{#1}\right\rvert}}
\newcommand{\paren}[1]{\left(#1\right)}
\newcommand{\brac}[1]{\left\{#1\right\}}
\newcommand{\sbrac}[1]{\left[#1\right]}
\newcommand{\norm}[1]{{\left\lVert{#1}\right\rVert}}
\newcommand{\inner}[1]{\left\langle#1\right\rangle}
\newcommand{\argmin}{\operatorname*{arg\,min}}
\newcommand{\err}{\operatorname*{err}}
\newcommand{\errone}{\operatorname*{err_1}}
\newcommand{\errtwo}{\operatorname*{err_2}}

\newtheorem{theorem}{Theorem}[section]
\newtheorem{lemma}[theorem]{Lemma}
\newtheorem{corollary}[theorem]{Corollary}
\newtheorem{proposition}[theorem]{Proposition}
\theoremstyle{definition}
\newtheorem{definition}[theorem]{Definition}
\newtheorem{assumption}{Assumption}
\theoremstyle{remark}
\newtheorem{remark}[theorem]{Remark}

\title{Statistical inverse learning and $\Ho$-regularization}
\date{}

\author{
Abhishake\thanks{
Department of Computational Engineering,
LUT University,
Lappeenranta, Finland.
\texttt{abhimaths88@gmail.com}}
\and
Tatiana A.~Bubba\thanks{
Department of Mathematics and Computer Science,
University of Ferrara,
Ferrara, Italy.
\texttt{tatiana.bubba@unife.it}}
\and
Tapio Helin\thanks{
Department of Computational Engineering,
LUT University,
Lappeenranta, Finland.
\texttt{tapio.helin@lut.fi}}
\and
Luca Ratti\thanks{
Department of Mathematics and Computer Science,
University of Ferrara,
Ferrara, Italy.
\texttt{luca.ratti5@unibo.it}}
}

\begin{document}
\maketitle

\begin{abstract}
We study the problem of recovering a sparse function from finite, noisy, and indirect observations in the framework of statistical inverse learning. The unknown is modeled as an element of $\ell^1$, and observations are generated via a (possibly nonlinear) forward operator $A$ from $\ell^1$ to a vector-valued Reproducing Kernel Hilbert Space (vv-RKHS) $\Ht$. We introduce an $\ell^1$-regularized empirical risk minimizer that promotes sparsity in the recovered solution, and we develop a comprehensive theoretical analysis of this estimator.

We establish almost-sure consistency of the estimator as the sample size grows under mild conditions. We then derive non-asymptotic, high-probability convergence rates in both the prediction norm and the $\ell^1$ reconstruction norm. The rates are expressed in terms of two intrinsic complexity parameters: the source smoothness index $r$, encoded through a variational source condition, and the effective dimension exponent $b$, which captures the polynomial spectral decay of the covariance operator associated with the vv-RKHS embedding. We further establish matching minimax lower bounds over a natural class of prior distributions, confirming that the derived upper rates are optimal. 
 
To connect the theory with concrete sparsity models, we develop a general 
framework for finitely smoothing operators of the form $A = G \circ S$, where $S$ is a 
synthesis operator, and show how approximation-space assumptions on the ground truth imply variational source conditions. In particular, we establish an equivalence between membership in the approximation space $k_t$ and polynomial decay of the best $n$-term approximation error, thereby linking sparse approximation properties directly to the convergence-rate exponents.

As applications, we verify the assumptions for two representative inverse problems: the identification of a reaction coefficient in an elliptic PDE and sparse recovery in computed tomography. For filtered Radon transforms, we further derive explicit effective-dimension asymptotics,
leading to concrete convergence rates for standard image models and sparsifying systems.
\end{abstract}

\medskip


\medskip


\section{Introduction}
Inverse problems constitute a central topic in modern applied mathematics and machine learning \cite{mueller2012linear}, where one aims to recover an unknown quantity from indirect and noisy observations. Typical examples include identifying material parameters in mechanics, reconstructing images in medical tomography, estimating reaction rates in biological systems, or drug evolution in pharmacokinetics \cite{mueller2012linear,khan2020,shukla2020physics,Hartung21}. Unlike standard supervised learning \cite{BauPerRos07, gerfo2008spectral, steinwart2009optimal, mendelson2010, Caponnetto07, blanchard2020kernel, lin2020optimal}, where observations directly reveal the value of a function at given inputs, inverse problems involve an additional layer of complexity due to the presence of the operator~$A$, which may be compact or ill-conditioned. In these problems, the goal is to recover an unknown function~$\fp$ from random, noisy and indirect observations of the form
\begin{equation}\label{eq:obs}
    y_i = A(\fp)(x_i) + \varepsilon_i, \qquad i = 1,\dots,n,
\end{equation}
where $A:\BB\to \Ht$ is a (possibly nonlinear) forward operator acting between the vector spaces $\BB$ and $\Ht$. We are going to assume that $\BB$ is a Banach space of sequences, whereas $\HH$ is a Hilbert space of functions from $\XX$ to $\YY$. Here,~$n$ is the
sample size and $\zz:=\{(x_i,y_i)\}_{i=1}^n$ are independent samples drawn from an underlying probability distribution $\rho$ on~$\XX\times \YY$. The observation points $x_i$ are random, and the noise terms $\varepsilon_i$ are random and centered, conditional on these design points. This stochastic nature of both the inputs and the noise adds additional complexity to the inverse problem of recovering the unknown sequence $\fp$. When both randomness and finite sampling are explicitly taken into account, the problem becomes statistical, giving rise to the framework of \emph{statistical inverse learning} \cite{BlaMuc16,Rastogi20,Rastogi23}. This viewpoint naturally connects classical inverse problems with modern machine learning, where the objective is to infer an underlying function from data through appropriate regularization and statistical principles.

In real-world scenarios, such inverse problems are typically ill-posed, meaning that they may lack a unique solution or exhibit strong sensitivity to data perturbations, i.e., small perturbations in the data may lead to large deviations in the reconstructed solution. This necessitates the use of regularization to obtain stable and meaningful estimates. A standard approach for such problems is to employ regularized empirical risk minimization \cite{Caponnetto07,BlaMuc16,Rastogi20}. For instance, the classical Tikhonov estimator with quadratic (Hilbert-norm) regularization is given by
\[
\fz 
  = \argmin_{f\in \DD(A)\cap \HH'} 
     \left\{ \frac{1}{n}\sum\limits_{i=1}^n\| A(f)(x_i) - y_i \|_\YY^2 
             + \la \|f\|_{\HH'}^2 \right\},
\]
where $\la>0$ is a regularization parameter and $\HH'$ is a Hilbert space encoding smoothness or structural prior information.
This regularization promotes smoothness and stability of the estimator, but it typically leads to dense solutions, i.e., every coordinate of~$f$ contributes to the reconstruction, regardless of its statistical relevance. 

In many modern learning problems, however, the underlying signal or function admits a \emph{sparse representation} with respect to some basis or dictionary. 
That is, although the data may live in a high- or even infinite-dimensional space, the essential information can often be captured by a small number of active components.
This observation motivates the use of sparsity-promoting regularization schemes, most prominently the $\Ho$ penalty.

In this work, we consider the \emph{sparse statistical inverse learning problem} defined through the functional
\begin{equation}
  J_\la(f) : = 
      \frac{1}{n} \sum_{i=1}^n \|A(f)(x_i) - y_i\|_\YY^2
      + \la \|f\|_{\Ho}.
\end{equation}
The corresponding estimator is then given by
\begin{equation}\label{eq:l1reg}
  \fz = \argmin_{f \in \DD(A)\cap\Ho} J_\la(f).
\end{equation}
Here, the regularizer $\|f\|_{\Ho}$ promotes sparsity in the learned solution by shrinking many coordinates of $f$ to zero.  
This mechanism effectively performs \emph{data-driven model selection} during the learning process, allowing the estimator to focus on the most relevant features or directions in the hypothesis space.

From a learning-theoretic perspective, $\Ho$-regularization acts as a form of \emph{implicit feature selection} that adaptively controls the model complexity in high-dimensional regimes. 
Unlike the $\ell^2$ penalty, which shrinks all coefficients uniformly, the $\Ho$ penalty creates a bias towards parsimonious models that are easier to interpret and potentially more robust to noise. 
This trade-off between stability and sparsity has been central to the success of methods such as the LASSO, compressed sensing, and sparse kernel regression.


In the statistical learning theory for nonparametric regression, the convergence rate of an estimator toward the true solution can be exceedingly slow as the sample size increases, a phenomenon commonly referred to as the \emph{No Free Lunch Theorem}~\cite{Devroye96}. To overcome this limitation, the complexity of data-generating distributions~$\rho$ is typically restricted by incorporating structural assumptions on the underlying data-generating process. Such \textit{a priori} assumptions are crucial, as they allow prior knowledge and regularity information to be incorporated into the learning process, thereby enabling faster rates of convergence.
These assumptions generally pertain to three key aspects: the smoothness of the source set of admissible solutions~$\fp$, the mapping properties of the forward operator~$A$ (e.g., smoothness or Lipschitz continuity), and the design or marginal probability measure~$\mu$ on~$\XX$, which determines the sampling points $(x_i)_{i=1}^{n} \subset \XX^{n}$. We study how the interplay between the forward operator~$A$, the sampling distribution~$\mu$, and the sparsity-inducing regularization affects generalization and convergence.
To ensure statistical consistency and to derive convergence rates, we impose the following standard conditions:
a sub-exponential noise assumption on the observational noise variables,
 Lipschitz continuity of the forward operator~$A$, 
and a variational source condition that quantifies the regularity of the true solution~$\fp$. 
Furthermore, we rely on the concept of the \emph{effective dimension}~$\mathcal{N}(\la)$, which encapsulates the interplay between the sample size~$n$ and the regularization strength~$\la$. 
Under a polynomial decay condition of the form $\mathcal{N}(\la) \le C \, \la^{-b}$ for $0<b<1$, we obtain quantitative estimates for the sample complexity and asymptotic convergence behavior of the estimator.


\medskip

The contributions of this paper can be summarized as follows:

\begin{enumerate}[(i)]
\item We propose a general framework for sparse statistical inverse learning in vector-valued reproducing kernel Hilbert spaces, extending classical kernel ridge regression to an $\ell^1$-regularized setting with nonlinear forward operators.
\item We establish almost-sure consistency and derive non-asymptotic high-probability convergence rates for the proposed estimator under variational source conditions, sub-exponential noise assumptions, and polynomial effective-dimension growth.
\item We prove matching minimax lower bounds over a natural family of probability distributions, thereby establishing the minimax optimality of the obtained convergence rates.
\item We develop a general approximation-space framework based on the spaces $k_t$ and show that polynomial decay of best $n$-term approximation errors implies the variational source conditions required for the statistical analysis.
\item We verify the assumptions for representative nonlinear inverse problems, including coefficient identification in elliptic PDEs and sparse computed tomography, and derive explicit convergence rates for filtered Radon transforms with standard reconstruction filters.
\end{enumerate}

\medskip

Our analysis unifies ideas from inverse problem regularization, compressed sensing, and statistical learning theory, providing a unified analysis of sparsity-driven estimators for operator-valued learning tasks. This study provides a unified theoretical foundation for analyzing inverse learning problems in functional and multi-dimensional output spaces, offering a rigorous bridge between deterministic regularization theory and modern statistical learning approaches.

\subsection{Comparison with Existing Results}
\label{sec:comparison}

We situate the present work within the broader literature on statistical inverse learning and sparsity-promoting regularization, organizing the comparison along four axes: (i) the regularization penalty, (ii) the linearity or nonlinearity of the forward operator, (iii) the underlying hypothesis space, and (iv) the type of convergence guarantees obtained.

\paragraph{Direct Supervised Learning and the Role of the Effective Dimension.}

The foundational reference for minimax-optimal rates in kernel-based regression is the work of Caponnetto and De Vito~\cite{Caponnetto07}, which establishes optimal rates for regularized least-squares for the direct learning ($A=I$ identity operator) in vector-valued reproducing kernel Hilbert spaces under an RKHS-norm penalty. A key contribution of~\cite{Caponnetto07} is the introduction of the effective dimension $\mathcal N(\lambda)$ as a measure of statistical complexity.
Under a source condition $\phi(t)=t^r$ and polynomial effective-dimension growth $\mathcal N(\lambda)\lesssim\lambda^{-b}$, they obtain minimax rates of order $n^{-(r+\frac{1}{2})/(2r+b+1)}$ in $L^2$-norm. 

Rastogi and Sampath~\cite{Rastogi17} extended this framework to general source conditions in vector-valued RKHSs, deriving optimal convergence rates $n^{-r/(2r+b+1)}$ in RKHS norm for a broad class of spectral regularization methods beyond the classical H\"older-type setting. Their analysis demonstrates that the attainable learning rates are determined by the interplay between the effective dimension and the index function governing the source condition, thereby substantially generalizing the results of~\cite{Caponnetto07}.

The present paper adopts the same effective-dimension framework but considers a fundamentally different setting. We study nonlinear inverse problems, replace Hilbert-space regularization by sparsity-promoting $\ell^1$ regularization, and work in a Banach-space reconstruction framework. The resulting minimax-optimal rate of order $n^{-r/(1+b-br)}$ reflects both the statistical complexity parameter $b$ and the sparse regularity parameter $r$ arising from the variational source condition.

\paragraph{Linear Statistical Inverse Learning with RKHS Regularization.}

The statistical inverse learning problem with linear forward operators was studied systematically by Blanchard and M\"ucke~\cite{BlaMuc16}, who derive minimax-optimal convergence rates of order $n^{-r/(2r+b+1)}$ in Hilbert-space reconstruction norm for a broad family of spectral regularization methods. 

These works are restricted to linear operators and Hilbert-space regularization schemes. By contrast, the present paper allows nonlinear forward operators and sparsity-promoting $\ell^1$ penalties. The resulting convergence analysis relies on variational source conditions and Banach-space techniques rather than spectral regularization theory.

\paragraph{Nonlinear Statistical Inverse Learning with RKHS Regularization.}

The closest antecedent of the present work is the nonlinear statistical inverse learning results of~\cite{Rastogi20}, which analyze Tikhonov regularization with Hilbert-space penalties in vector-valued reproducing kernel Hilbert spaces. Under suitable source and complexity assumptions, the same minimax-optimal convergence rates as in the corresponding linear setting are obtained in the Hilbert-space reconstruction norm.

Compared with these works, the present paper introduces several new ingredients:

\begin{enumerate}[(i)]
\item sparsity-promoting $\ell^1$ regularization in place of Hilbert-space norm regularization;
\item a Banach-space reconstruction framework based on $\ell^1$;
\item variational source conditions formulated directly in the reconstruction norm;
\item minimax-optimal rates characterized by the regularity parameter $r$ and the complexity parameter $b$ in this sparse setting.
\end{enumerate}

While the work of~\cite{Rastogi20} is formulated within a Hilbert-space framework and exploit the geometry of Hilbert-space norm regularization, the present paper studies sparse recovery via $\ell^1$ regularization. This transition from Hilbert-space regularization to a Banach-space setting fundamentally changes the analysis and requires variational source conditions adapted to sparsity, which do not arise in the classical RKHS-based theory.

\paragraph{Deterministic Sparsity Regularization.}

Deterministic sparsity-promoting inverse problems have been extensively studied; see Flemming~\cite{Flemming18}, Hohage and Miller~\cite{Hohage19}, and Miller and Hohage~\cite{Miller21}.

The present work builds upon these deterministic developments but addresses a fundamentally different setting. Rather than deterministic perturbations characterized by a noise level $\delta$, we consider statistical inverse learning under random sampling. Consequently, our convergence analysis must simultaneously control stochastic sampling fluctuations and operator ill-posedness through concentration inequalities and the effective dimension of the underlying vector-valued RKHS. Moreover, we establish matching minimax lower bounds, thereby extending Banach-space regularization theory into the statistical learning regime.

Miller and Hohage~\cite{Miller21} obtained the convergence rates of order ${\delta^{(2-2t)/(2-t)}}$ in $\ell^1$-norm for the nonlinear inverse problem.  The present paper instead considers the statistical regime $n\to\infty$, where both the sampling locations and observations are random. Consequently, the effective dimension and covariance concentration phenomena play a central role and have no direct analogue in deterministic analyses.



\paragraph{General Convex Regularization.}

Bubba et al.~\cite{bubba2023convex,bubba2022shearlet} study statistical inverse learning with \emph{general convex, $p$-homogeneous regularization functionals} and linear forward operators, deriving concentration rates in symmetric Bregman distances induced by the penalty. In contrast, the present work focuses specifically on sparsity-promoting $\ell^1$ regularization for possibly \emph{nonlinear} forward operators and establishes minimax-optimal convergence rates in reconstruction norms. Furthermore, our analysis connects approximation-space sparsity models to variational source conditions and provides matching minimax lower bounds.  Bubba et al.~\cite{bubba2023convex} also measures errors primarily through Bregman distances, whereas our results are formulated directly in $\ell^1$ and related sequence-space norms. 

\paragraph{Sparse Statistical Estimation and $\ell^1$ Regularization.}
The use of $\ell^1$ penalties for sparse estimation originates with the LASSO of Tibshirani~\cite{Tibshirani96}, whose statistical properties were subsequently analyzed by~\cite{CandesTao07,BRT09,RWY11}. These works establish minimax estimation rates, oracle inequalities, and variable-selection guarantees in finite-dimensional linear regression models under assumptions such as restricted eigenvalue or restricted isometry conditions.

The present work extends this line of research from finite-dimensional regression to nonlinear statistical inverse learning in infinite-dimensional Banach spaces. Instead of restricted eigenvalue assumptions, our analysis relies on variational source conditions and effective-dimension estimates associated with the covariance operator. Consequently, the convergence rates are governed jointly by the sparsity parameter $r$ and the statistical complexity parameter $b$, yielding minimax-optimal reconstruction guarantees for nonlinear inverse problems under random design.




\paragraph{Bayesian Sparse Inverse Problems.}
Recent work has also investigated sparsity-promoting Bayesian methods for inverse problems. In particular, Agapiou and Wang~\cite{AgapiouWang24} analyze Bayesian inverse problems with Laplace priors and establish posterior contraction properties for spatially inhomogeneous Besov-type models. Although Laplace priors and $\ell^1$ regularization are closely related through maximum a posteriori estimation, the objectives of the two approaches differ substantially. Their analysis focuses on posterior distributions and Bayesian uncertainty quantification, whereas the present paper develops a frequentist statistical learning framework based on empirical risk minimization. Furthermore, our analysis establishes explicit high-probability convergence rates and matching minimax lower bounds for nonlinear statistical inverse learning under random sampling.

\paragraph{Summary Comparison.}

We summarize the results discussed in this section in Table~\ref{tab:comparison}.

\begin{table}[ht]
\renewcommand{\arraystretch}{1.3}
\centering
\caption{Comparison with related literature.}
\label{tab:comparison}
\small
\begin{tabular}{llcp{2cm}p{1cm}p{1cm}}
\hline
Reference & Operator & Regularizer  & Reconstruction Norm Rate & Minimax lower bounds & Optimal rates\\
\hline
Caponnetto--De Vito~\cite{Caponnetto07}
& Identity & Tikhonov
& 
& \checkmark & \checkmark \\
Rastogi et al.~\cite{Rastogi17}
& Identity & General
& $n^{-r/(2r+b+1)}$
& \checkmark & \checkmark \\
Blanchard--M\"ucke~\cite{BlaMuc16}
& Linear & General
& $n^{-r/(2r+b+1)}$
& \checkmark & \checkmark \\
Rastogi et al.~\cite{Rastogi20}
& Nonlinear & Tikhonov
& $n^{-r/(2r+b+1)}$
& \checkmark & \checkmark \\
Miller--Hohage~\cite{Miller21}
& Nonlinear
& $\ell^1$
& ${\delta^{(2-2t)/(2-t)}}$
& \checkmark & \checkmark \\
{This work}
& {Nonlinear}
& ${\ell^1}$
& ${n^{-r/(1+b-br)}}$
& \checkmark  & \checkmark \\
\hline
\end{tabular}
\end{table}

\paragraph{Organization.} 
The paper is organized as follows.  
Section \ref{sec:setting} introduces the framework of statistical inverse learning under random design and presents the vector-valued Reproducing Kernel Hilbert Space (vv-RKHS) structure used in our analysis. In this section, we also state the key assumptions required for our results, including the Bernstein-type noise condition, kernel regularity, smoothness assumptions on the operator $A$ and the true solution $\fp$, and the polynomial decay condition on the effective dimension. Section \ref{sec:consistency} establishes consistency results for the proposed estimator. 
In  Section \ref{Sec:convergence.rates}, we present the main convergence results for the sparsity-promoting regularization scheme. These results provide upper convergence rates under \textit{a priori} smoothness assumptions on the true solution, expressed in terms of the variational source condition. 
In Section \ref{sec:lower.rates}, we derive minimax lower bounds in the class of probability measures satisfying the same assumptions as the ones employed to obtain the upper estimates: as a consequence, we deduce that the derived convergence rates are optimal.
Section \ref{sec:examples} supplements the theoretical discussion of the previous sections by showcasing some examples, fully motivated by applications, that satisfy all the introduced theoretical assumptions.
Finally, Section \ref{sec:discussion} provides a detailed discussion of the derived results, including comparisons with existing literature and an analysis of the implications of our approach. In addition, the appendix contains auxiliary results and technical lemmas used throughout the analysis.

\vspace{0.3cm} 

\paragraph{Notations.} 
For a Banach space $\BB$, we denote its norm by $\|\cdot\|_{\BB}$. For a Hilbert space~$\HH$, the norm and inner product are denoted by~$\|\cdot\|_{\HH}$ and~$\langle \cdot, \cdot \rangle_{\HH}$, respectively. The space of all bounded linear operators on a separable Hilbert space $\HH$ is denoted by $\mathcal{L}(\HH)$. 

\vspace{0.2cm}

We denote the adjoint of an operator $A$ by $A^*$ and its domain by $\DD(A)$. The operator norm of $A$ is written as $\|A\|$, while $\|A\|_{HS}$ denotes its Hilbert--Schmidt norm.

\section{Problem Setting and Main Assumptions}
\label{sec:setting}

We formulate the nonlinear statistical inverse learning problem in a vector-valued reproducing kernel Hilbert space (vv-RKHS) framework. In this section, we introduce the key assumptions required for our analysis, including conditions on the noise, regularity of the forward operator, smoothness of the true solution, and spectral properties of the associated covariance operators. These assumptions are essential to derive statistical guarantees, establish consistency, and obtain convergence rates for the proposed regularized estimators. Throughout, we adopt the notation and conventions introduced in the preceding section.

\subsection{General Framework}
Suppose $\XX \subset \RR^d$ denotes the input domain and $\YY$ is a separable Hilbert space representing the output space. We observe data points $(x_1, y_1), \ldots, (x_n, y_n)$ drawn according to an unknown probability measure~$\rho$ defined on the Borel $\sigma$-algebra of~$\XX \times \YY$. We denote by $\mu$ the marginal distribution of $\rho$ on $\XX$, and by $\rho(\cdot \mid x)$ the conditional distribution on $\YY$ given $x \in \XX$, whose existence is assumed.

We define the weighted Hilbert space $\Hu := \LL$, endowed with the inner product 
$\langle g_1, g_2 \rangle_{\mu} =\langle g_1, g_2 \rangle_{\Hu} := \int_\XX \langle g_1(x), g_2(x) \rangle_\YY \, \mu(dx)$.
Henceforth, we assume that $\mathcal{B}=\ell^1$ is the underlying sequence space.
Consider a bounded nonlinear operator
\[
    A : \DD(A) \cap \Ho\to \Ht,
\]
where $\Ht$ is a vector-valued reproducing kernel Hilbert space (vv-RKHS) that is continuously embedded into $\Hu$ via the inclusion operator 
\[
    \ip : \Ht \hookrightarrow \Hu.
\]
The embedding $\ip$ ensures that elements of $\Ht$ can be identified with square-integrable vector-valued functions while preserving the reproducing property essential for our analysis. 

This operator-valued kernel formulation naturally accommodates multi-output regression and structured prediction tasks, thereby extending the scope of scalar-valued learning theory to a broader class of nonlinear statistical inverse problems.

Given i.i.d.~samples $\xx: = \{x_i\}_{i=1}^n$ drawn from $\mu$, the sampling operator $\sx : \Ht \to \YY^n$ is defined by
\begin{equation}\label{def:sx}
   (\sx g)_i := g(x_i), \quad i=1,\ldots,n. 
\end{equation}
Then, the observed data can be expressed as
\begin{equation}
    \label{eq:main-model}
    \yy = \sx A(\fp) + \VE,
\end{equation}
where $\yy: = (y_i)_{i=1}^n \in \YY^n$ and $\VE := (\varepsilon_i)_{i=1}^n$ is a noise vector. The noise variables $\{\varepsilon_i\}_{i=1}^n$ are assumed to be independent and centered, satisfying $\int_\YY \varepsilon_i\, \rho(dy_i|x_i) = 0$ for all  $i = 1, \dots, n$.

Our objective is to reconstruct the function $\fp$ from the observed samples $(x_i, y_i)_{i=1}^n$
by minimizing a regularized empirical risk that incorporates $\Ho$-type penalization promoting sparsity in the representation of~$\fp$.

\subsection{True solution}
The probability distribution~$\rho$ is accessible only through a \emph{training set} 
$\zz$. 
The goal of \emph{supervised inverse learning} is to construct an estimator 
$f_{\zz}$ based on $\zz$ such that $[A(f_{\zz})](x)$ 
approximates the true label~$y$ for unseen samples $(x, y)$. 
To formalize this objective, we define the \emph{expected square loss}
\[
\mathcal{E}(f) = \int_{\XX \times \YY} \norm{[A(f)](x) - y}_{\YY}^{2}\, \rho(dx, dy),
\]
which measures the expected discrepancy between predictions and true labels. 
Using the marginal probability distribution~$\mu$ on~$\XX$ and the conditional distribution~$\rho(\cdot \mid \cdot)$ of $y$ given $x$, 
the expected loss can be rewritten as
\[
\mathcal{E}(f) = \|A(f) - \gp\|_{\mu}^{2} + \mathcal{E}(\gp),
\]
where $\gp(x) = \int_\YY y\, \rho(dy \mid x)$ denotes the \emph{conditional mean function}. 
Hence, minimizing the expected loss is equivalent to minimizing 
$\|A(f) - \gp\|_{\mu}^{2}$, which corresponds to solving a classical inverse problem 
where data fidelity is weighted by the design measure~$\mu$. 
In particular, if there exists $f^\dagger \in \DD(A)$ such that 
$\gp = A(f^\dagger)$, then $f^\dagger$ minimizes the expected risk. 
For the centered noise, this minimizer $f^\dagger$ coincides with the true solution~$\fp$ 
for an injective operator~$A$.

In what follows, we specify the assumptions concerning the true solution of the inverse problem~\eqref{eq:obs}; see also~\cite{BlaMuc16}.  

\begin{assumption}[True solution]\label{ass:fp}
The conditional expectation of \(y\) given \(x\), with respect to the probability distribution~\(\rho\), is assumed to exist almost surely. 
Moreover, there exists an element \(\fp \in \DD(A)\) such that  
\[
\int_{\YY} y \, d\rho(y \mid x) 
= [A(\fp)](x), 
\quad \text{for all } x \in \XX.
\]
\end{assumption}

\subsection{Noise Assumption}

To ensure statistical consistency, we impose a sub-exponential noise condition analogous to Bernstein-type inequalities commonly used in empirical process theory.

\begin{assumption}[Sub-exponential noise]
\label{ass:noise}
There exist constants $M,\Sigma > 0$ such that for almost all $x \in \XX$,
\[
    \int_\YY\!\left(e^{\| y - A(\fp)(x) \|_\YY / M} - \frac{\|y - A(\fp)(x)\|_\YY}{M} - 1\right) \rho(dy \mid x)
    \;\leq\; \frac{\Sigma^2}{2M^2}.
\]
\end{assumption}

This assumption remains valid under several settings, including cases where the noise~$\varepsilon$ is bounded 
or follows a sub-Gaussian distribution with zero mean and is independent of~$x$; 
see, e.g.,~\cite{van96}. 
It is worth noting that the case of Gaussian white noise in infinite-dimensional spaces 
is excluded from this setting.
It ensures concentration inequalities needed for non-asymptotic error analysis.

\subsection{RKHS Structure}

We begin by recalling the notion of a vector-valued reproducing kernel Hilbert space (RKHS). The RKHS framework provides a powerful foundation for kernel-based methods, enabling the development of efficient and theoretically grounded algorithms. Our focus lies on Hilbert spaces of vector-valued functions that admit a reproducing kernel \cite{alvarez2012kernels, carmeli2006vector, carmeli2010vector}. Such spaces have attracted considerable interest in recent years, particularly in machine learning theory, due to their effectiveness in modeling and learning from complex, structured data.

The concept of an RKHS originates from the seminal work of Aronszajn \cite{Aronszajn50}, who studied Hilbert spaces of functions associated with symmetric, positive semi-definite kernels. A defining feature of these spaces is the reproducing property, which allows pointwise evaluation of functions to be expressed as an inner product involving the kernel. This framework has been extended to vector-valued functions by Micchelli and Pontil \cite{Micchelli05}, thereby generalizing the classical scalar-valued RKHS to settings in which functions take values in a Hilbert space rather than in the real line~$\RR$. This generalization allows RKHS methods to model multiple, potentially correlated outputs simultaneously,
and provides a unifying framework for multi-task learning, structured regression, and functional data analysis.

\begin{definition}[Vector-valued RKHS]
Let $\XX$ be a non-empty set and $(\YY,\langle \cdot,\cdot \rangle_\YY)$ be a real separable Hilbert space. 
A Hilbert space $\HH$ of functions $g:\XX \to \YY$ is a \emph{vector-valued reproducing kernel Hilbert space (vv-RKHS)} if for all $x \in \XX$ and $y \in \YY$, the evaluation functional
\[
    F_{x,y}: \HH \to \RR, \qquad F_{x,y}(g) = \langle y, g(x)\rangle_\YY,
\]
is continuous.
\end{definition}

The Riesz representation theorem allows us to identify, for each $x \in \XX$ and $y \in \YY$, the continuous functional $F_{x,y}$ with a unique element of $\HH$ (denoted in \cite{Micchelli05} by $K(x|y)$) such that
\[
\langle K(x|y), g \rangle_\HH
=
F_{x,y}(g)
=
\langle y, g(x)\rangle_\YY,
\qquad \forall g \in \HH.
\]

We observe that the dependence of $K(x|y)$ on $y$ is linear. Hence, for every $x \in \XX$, we define the linear operator
\[
K_x \colon \YY \to \HH,
\qquad
K_x y = K(x|y).
\]
With this notation, the reproducing property becomes
\[
\langle K_x y, g \rangle_\HH
=
F_{x,y}(g)
=
\langle y, g(x)\rangle_\YY,
\qquad \forall g \in \HH.
\]

The operator-valued kernel associated with the RKHS $\HH$ is the map
\[
K: \XX \times \XX \to \mathcal{L}(\YY),
\]
which assigns to each pair $x,x' \in \XX$ the linear operator defined by
\[
K(x,x') y = (K_{x'} y)(x),
\qquad \forall y \in \YY.
\]

According to \cite[Proposition 2.1]{Micchelli05}, the kernel $K$ satisfies the following properties:
\begin{enumerate}[(i)]
\item For every $y,y' \in \YY$,
\[
\langle y, K(x,x') y' \rangle_\YY
=
\langle K_{x'} y', K_x y \rangle_\HH;
\]

\item \textit{Hermitian symmetry}:
\[
K(x,x')^* = K(x',x);
\]

\item \textit{Positive semi-definiteness}: for every finite family
$\{x_i\}_{i=1}^n \subset \XX$ and $\{y_i\}_{i=1}^n \subset \YY$,
\[
\sum_{i,j=1}^n
\langle y_i, K(x_i,x_j)y_j \rangle_\YY
\ge 0.
\]
\end{enumerate}


%

Conversely, any kernel satisfying the three conditions above is associated with a unique vv-RKHS, defined as
\[
\overline{\operatorname{span}}\{K_x y \;:\; x \in \XX,\ y \in \YY\}.
\]
For theoretical analysis, we impose the following mild regularity condition on the kernel.

\begin{assumption}[Kernel regularity]
\label{ass:kernel}
Let $\Ht$ be a vector-valued reproducing kernel Hilbert space. The operator-valued kernel $K:\XX\times \XX \to \mathcal{L}(\YY)$ associated with $\Ht$ satisfies:
\begin{enumerate}[(i)]
  \item For all $x \in \XX$, $K_x:\YY \to \HH$ is a Hilbert--Schmidt operator with
  \[
      \kappa^2 := \sup_{x \in \XX} \| K_x \|_{HS}^2 
      = \sup_{x \in \XX}\tr(K_x^* K_x) < \infty.
  \]
  \item For all $y,t \in \YY$, the real-valued function 
  $\varsigma(x,s) := \langle K_x y, K_s t\rangle_{\HH}$ 
  is measurable with respect to $(x,s) \in \XX \times \XX$.
\end{enumerate}
\end{assumption}

Assumption~\ref{ass:kernel} ensures that the kernel induces a well-behaved operator-valued feature map. 
The Hilbert--Schmidt condition guarantees boundedness and square-integrability, 
while the measurability requirement ensures that all integrals involving $K$ 
(such as covariance operators or empirical risks) are well-defined. 

\subsection{Smoothness Assumptions}

We impose two structural assumptions on the forward operator~$A$ and one on the unknown solution~$\fp$. 
Assumptions \ref{ass:Lipschitz} and \ref{Ass:A-Lip} concern the \emph{regularity and stability} of the operator $A$, while Assumption \ref{Ass:var.sour} characterizes the \emph{smoothness} of the true solution in relation to~$A$.

The following assumption ensures the \emph{well-posedness} of the forward mapping. 
Injectivity prevents ambiguity in the inverse problem, while Lipschitz continuity guarantees that the mapping is not overly sensitive to small changes in~$f$. 
This property enables control over the propagation of perturbations from the data space~$\Ht$ to the parameter space~$\Ho$, which is essential for the local stability of the inverse problem.

\begin{assumption}[Lipschitz continuity]
\label{ass:Lipschitz}
We assume that $\DD(A)$ has nonempty interior. The operator
\(
A:\DD(A)\cap \ell^1 \to \HH
\)
is injective.
Moreover, $A$ is Lipschitz continuous on $\DD(A)\cap \ell^1$, namely,
there exists a constant $\ella<\infty$ such that
\[
\|A(f)-A(\tilde f)\|_{\HH}
\le
\ella\,\|f-\tilde f\|_{\ell^1},
\qquad
\forall f,\tilde f\in \DD(A)\cap \ell^1 .
\]
\end{assumption}

Next, we specify the smoothness or regularity of the true solution~$\fp$. 
In inverse problems, such assumptions describe how well $\fp$ can be approximated 
by elements related to the range of the operator~$A$, 
and they play a crucial role in determining attainable rates of convergence.

\begin{assumption}[Variational source condition]
\label{Ass:var.sour}
There exists a concave, non-decreasing index function 
$\phi: [0,\infty) \to [0,\infty)$ with $\phi(0) = 0$ 
such that for all $f \in \DD(A)\cap\Ho$,
\[
    \|f - \fp\|_{\Ho}
    \le 
    \|f\|_{\Ho} - \|\fp\|_{\Ho}
    + \phi\!\left(\|A(f) - A(\fp)\|_{\Hu}^2\right).
\]
\end{assumption}
The variational source condition captures the \emph{intrinsic smoothness} of the true solution. It establishes a quantitative link between the discrepancy in the parameter space and the residual in the data space. 
It connects the regularity of the true solution~$\fp$ with the sensitivity of the forward operator~$A$. 
The function~$\phi$ determines the rate at which approximation errors decay, and thus directly influences the achievable convergence rate of regularized estimators.

\medskip
We now introduce a family of weighted sequence spaces that will serve as model domains for the operator~$A$.

\begin{definition}[Weighted sequence spaces]
Let $\underline{w} = (w_j)_{j=1}^\infty$ be a sequence of positive real numbers such that $w_j \to 0$ as $j \to \infty$.  
For $p \in (0,2]$, the weighted space $\ell^p_{\underline{w}}$ is defined as
\[
\ell^p_{\underline{w}}
:= \Big\{ f \in \mathbb{R}^\infty : \|f\|_{\underline{w},p} < \infty \Big\},
\qquad
\|f\|_{\underline{w},p} := \Big(\sum_{j=1}^{\infty} w_j^p |f_j|^p \Big)^{1/p}.
\]
\end{definition}

\noindent

The next assumption provides a quantitative stability property of~$A$ with respect to the weighted norm.

\begin{assumption}[Weighted bi-Lipschitz property]\label{Ass:A-Lip}
For a sequence of positive weights $\underline{w}$ vanishing at $\infty$, assume that $D(A) \subset \ell_{\underline{w}}^2$ is closed and 
\begin{enumerate}[(i)]
    \item there exists a constant $L > 0$ such that for all $f, \tilde{f} \in D(A)$,
    \[
    \norm{A(f)-A(\tilde{f})}_{\Hu}
    \le 
    L \|f - \tilde{f}\|_{\underline{w},2},
    \]
    \item there exists a constant $L>0$ such that for all $f, \tilde{f} \in D(A)$,
    \[
    \|f - \tilde{f}\|_{\underline{w},2}
    \le 
    L \norm{A(f)-A(\tilde{f})}_{\Hu}.
    \]
\end{enumerate}
\end{assumption}

The first part of the assumption is needed to obtain uniform lower convergence rates, as discussed in Theorem \ref{err.lower.bound}. The second part is crucial to provide uniform upper convergence rates (see Corollary~\ref{cor:err.upper.weighted} and Theorem \ref{thm:uniform.rate.lambda*}). Moreover, as shown in Theorem \ref{kt.vsc}, the second part is a very useful tool to verify the variational source condition (Assumption \ref{Ass:var.sour}). Finally, verifying both parts of Assumption \ref{Ass:A-Lip} with the same weight $\underline{w}$ allows to prove the minimax optimality of the derived bounds, as discussed in Corollary \ref{cor:minimax-optimality}.



\subsection{Effective Dimension and Spectral Decay}

We now formalize the notions of effective dimension and spectral decay through the covariance operator associated with the feature map of the kernel $K$.

\begin{definition}[Uncentered covariance operator]
\label{def:covop}
The \emph{uncentered covariance operator}\index{Uncentered covariance operator} associated with the kernel $K$ is defined as the operator $\tp: \HH \to \HH$ given by
\begin{equation*}
    \tp := \int_{\XX}  K_x K_x^*  \, \mu(dx).
\end{equation*}
\end{definition}

The operator $\tp$ is positive, self-adjoint, compact, and in particular of trace class by Assumption~\ref{ass:kernel}.  
Let $\ip : \Ht \hookrightarrow \Hu$ denote the canonical embedding. Then, the covariance operator can equivalently be expressed as $\tp = \ip^* \ip$.

\medskip

In statistical learning theory, assumptions on the marginal distribution $\mu$ are often characterized in terms of the \emph{effective dimension}\index{Effective dimension} (or \emph{statistical dimension}\index{Statistical dimension}), which captures the number of effective degrees of freedom of the learning problem.  
Intuitively, when data are high-dimensional, only a subset of features significantly contributes to the variation in the data. The effective dimension thus provides a measure of model complexity, reflecting how many directions in feature space remain active at a given regularization level.  

Formally, the effective dimension is defined as
\[
    \mathcal{N}(\la) := \tr\!\big((\tp + \la I)^{-1} \tp\big), \qquad \la > 0,
\]
which quantifies the number of ``active'' degrees of freedom at scale $\la$.  

\medskip

A common way to control the complexity of the hypothesis space is through a polynomial bound on $\mathcal{N}(\la)$.

\begin{assumption}[Polynomial spectral decay]
\label{ass:polydecay}
There exist constants $C_\beta > 0$ and $0 < b < 1$ such that
\[
    \mathcal{N}(\la) \le C_\beta \la^{-b}, \qquad \forall \la > 0.
\]
\end{assumption}

This spectral decay condition links the eigenvalue distribution of $\tp$ to the statistical complexity of the learning problem.  
It ensures that the hypothesis space does not grow too rapidly as $\la \to 0$, allowing for meaningful generalization and stable regularization.  

\bigskip

Together, the definitions and assumptions introduced above provide the analytical foundation for the convergence analysis of $\Ho$-regularized estimators.  
They guarantee well-posedness, statistical consistency, and enable the derivation of non-asymptotic convergence rates under mild regularity and sparsity assumptions.

\subsection{Class of Probability Measures}

We now introduce the class of probability measures that describe the statistical framework of the nonlinear inverse problem.  
These measures encode both the smoothness of the true solution~$\fp$ and the spectral characteristics of the forward operator~$A$ through the induced covariance structure.

\begin{definition}[Class of Probability Measures]\label{def:classP}
Let $M, \Sigma, R$, and $C_\beta$ be fixed positive constants.
Given parameters $0 < b < 1$ and $0 \le r \le 1$, we define $\PP=\PP_{r,b}$ as the set of probability distributions $\rho$ on the sample space $\ZZ=\XX\times\YY$ such that:
\begin{enumerate}[(i)]
    \item The true solution $\fp$ satisfies Assumption~\ref{ass:fp}, ensuring that it belongs to the admissible domain of the operator.
   
    \item The noise condition (Assumption~\ref{ass:noise}) holds with the prescribed constants $M$ and $\Sigma$, controlling the stochastic variability in the observations.
   
    \item The kernel $K$ associated with the RKHS $\HH$ satisfies Assumption~\ref{ass:kernel}, and the range of the forward operator $A$ is contained in $\HH$.
   
    \item Assumption~\ref{ass:Lipschitz} holds, so that the domain $\DD(A)$ has nonempty interior. Moreover, the operator $A:\DD(A)\cap\Ho\to\Ht$ is injective and Lipschitz continuous. In addition, the weighted bi-Lipschitz property stated in Assumption~\ref{Ass:A-Lip} holds, providing quantitative stability of $A$ with respect to the weighted norm $\|\cdot\|_{\underline{w},2}$.
   
    \item The true solution $\fp$ satisfies the variational source condition (Assumption~\ref{Ass:var.sour}) with index function $\phi(t)=t^r$, capturing its intrinsic smoothness relative to the operator $A$.
   
    \item The effective dimension of the operator $\tp$ satisfies the polynomial growth condition (Assumption~\ref{ass:polydecay}), which controls the complexity of the learning problem through the spectral decay of $A$.
\end{enumerate}
\end{definition}

\medskip

\noindent
The class of probability measures $\PP_{r,b}$ thus depends on two fundamental parameters:
\begin{enumerate}[(i)]
    \item The smoothness parameter $r$, describing the \emph{smoothness} of the true solution~$\fp$ via the variational source condition with $\phi(t) = t^r$;
    \item The effective dimension parameter $b$, which captures the number of effective degrees of freedom of the learning problem.
\end{enumerate}
Accordingly, we consider throughout this work the family
\[
\PP_{r,b}
:=
\Big\{
\rho \text{ on } \ZZ : \rho \text{ satisfies Assumptions~\ref{ass:fp}, \ref{ass:noise}, \ref{ass:kernel}, \ref{ass:Lipschitz}, \ref{Ass:var.sour}, \ref{ass:polydecay}}
\text{ with } \phi(t) = t^r
\Big\}.
\]
This class defines a notion of prior distributions for RKHS-based, sparsity-promoting nonlinear inverse learning problems,  
encapsulating both the spectral decay of the operator $A$ and the intrinsic regularity of the underlying solution~$\fp$.

\subsection{Upper Rates, Lower Rates, and Minimax Optimality}

Here, we introduce the notions of asymptotic upper rates, lower rates, and minimax optimality used throughout this paper. The goal is to describe precisely the asymptotic behaviour of learning procedures in terms of the sample size $n$.

To this end, we consider a family of probability measures $(\mathcal{P}_{r,b})$, indexed by
regularity and complexity parameters, where each $\mathcal{P}_{r,b}$ is a collection of
Borel probability measures on $\XX \times \YY$ defined in Definition~\ref{def:classP}. In the following, all expectations are taken with respect to $\rho^n$
for $\rho \in \mathcal{P}_{r,b}$.

We measure performance in terms of the $\pp$-th moment of the reconstruction error in the
interpolation norm $\|\cdot\|_{\underline{u},p}$, where $\pp\in[1,\infty)$, $p\in (0,2]$ and $\underline{u}$ is defined in~\eqref{eq:omega_t}.

\begin{definition}[\textbf{Upper rate of convergence}]
A family of positive sequences $(a_n)_{n\in\mathbb{N}}$ is called an upper rate of convergence
in $L^\pp$ for the interpolation norm $\|\cdot\|_{\underline{u},p}$ over the model class
$\mathcal{P}_{r,b}$, if there exists a learning procedure $l$
producing estimators $(f_{\zz,\la_n})_{n\in\mathbb{N}}$, such that
\[
\limsup_{n \to \infty}\sup_{\rho \in \mathcal{P}_{r,b}}
\frac{
\mathbb{E}_{\rho^n}
\big[
\|f_{\zz,\la_n} - \fp\|_{\underline{u},p}^{\pp}
\big]^{1/\pp}
}{a_n}
< \infty.
\]
\end{definition}

\begin{definition}[\textbf{Lower rates of convergence}]
A family of positive sequences $(a_n)_{n\in\mathbb{N}}$ is called a lower rate of convergence in
$L^\pp$ for the interpolation norm $\|\cdot\|_{\underline{u},p}$ over the model class
$\mathcal{P}_{r,b}$, if
\[
\liminf_{n \to \infty} \inf_{l}
\sup_{\rho \in \mathcal{P}_{r,b}}
\frac{
\mathbb{E}_{\rho^n}
\big[
\|f_{\zz}^{l} - \fp\|_{\underline{u},p}^{\pp}
\big]^{1/\pp}
}{a_n}
> 0.
\]
where the infimum is taken over all measurable learning algorithms
$l$ producing estimators $(f_{\zz}^{l})_{n\in\mathbb{N}}$.
\end{definition}

\medskip

\noindent

\begin{definition}[\textbf{Minimax optimal rate}]
A sequence of estimators $(f_{\zz,\la_n})_{n\in\mathbb{N}}$ is called minimax optimal in $L^\pp$ for the interpolation
norm $\|\cdot\|_{\underline{u},p}$ over $\mathcal{P}_{r,b}$, with rate of convergence $(a_n)$, if $(a_n)$
is both an upper rate of convergence and a lower rate of convergence for the same model class.
\end{definition}

\section{Consistency}
\label{sec:consistency}

In this section, we establish the consistency of the proposed estimator defined by the minimization problem~\eqref{eq:l1reg}. 
Our goal is to show that, under suitable assumptions on the forward operator~$A$, the kernel~$K$, and the data-generating distribution~$\rho$, the regularized estimator~$\fz$ converges to the true solution~$\fp$ as the sample size~$n$ increases. 

We now state the main high-probability consistency theorem.

\begin{theorem}\label{thm:asconsistency}
Suppose Assumptions~\ref{ass:fp}-\ref{ass:Lipschitz} hold.
Assume further that the composite operator \(S_\mu \circ A : \ell^1 \to \Hu\) is injective and weak-to-weak sequentially continuous. Let~$\fz$ denote a solution to the minimization problem~(\ref{eq:l1reg}) and choose the regularization parameter~$\la=\la_n>0$ such that
\begin{equation}\label{la.choice}
\la_n\to 0,~~~\frac{\log n}{\la_n\sqrt{n}}\to 0 \text{   as   } n\to\infty.
\end{equation}
Then the estimator~$\fz$ satisfies
\begin{equation}\label{consistency.rate}
\|\fz-\fp\|_{\Ho}\to 0   \text{  almost surely as   } n\to\infty.
\end{equation}
\end{theorem}

\begin{proof}
The proof proceeds in several steps. 

\noindent {\bf Basic variational inequalities.} 
First, by definition of $\fz$ as a minimizer of~\eqref{eq:l1reg}, we have 
\[
\norm{\sx A(\fz)-\yy}_n^2 + \la \norm{\fz}_{\Ho} \le \norm{\sx A(\fp)-\yy}_n^2 + \la \norm{\fp}_{\Ho},
\]
where $\norm{\mathbf{y}}_n^2 = \frac{1}{n}\sum_{i=1}^n\| y_i \|_\YY^2$. Expanding and rearranging terms gives
\begin{align}\label{eq01}
  &\norm{\sx \brac{A(\fz)-A(\fp)}}^2_n 
  +\la\|\fz\|_{\Ho}  \\ \nonumber
  \leq & 2\inner{\sx \brac{A(\fp)-A(\fz)},\sx A(\fp)-\yy}_n +\la\|\fp\|_{\Ho}.
\end{align}

Using the observation model \(\yy=S_{\xx}A(\fp)+\VE,\) and applying Cauchy--Schwarz we obtain
\begin{align}
  \label{c1}
  &\norm{\sx \brac{A(\fz)-A(\fp)}}^2_n 
  +\la\|\fz\|_{\Ho}  \\ \nonumber
  \leq & 2 \norm{\sx^*\VE}_{\Ht}\norm{A(\fz)-A(\fp)}_{\Ht} +\la\|\fp\|_{\Ho}.
\end{align}

For the sampling operator defined in \eqref{def:sx}, let $\tx := \sx^*\sx$ denote the empirical covariance operator. 
Adding and subtracting the term $\| S_\mu \{A(\fz)- A(\fp)\}\|^2_{\Hu}$ in \eqref{eq01} gives
\begin{align*}
&\norm{\ip \brac{A(\fz)-A(\fp)}}_{\Hu}^2+\la\|\fz\|_{\Ho}  \\  
\leq &  2\inner{A(\fp)-A(\fz),\sx^*\VE}_{\Ht}+\la\|\fp\|_{\Ho} \\ 
&+\inner{\paren{\tp-\tx}\brac{A(\fz)-A(\fp)},A(\fz)-A(\fp)}_{\Ht}.
\end{align*}

Applying the Cauchy--Schwarz inequality gives
\begin{align}\label{c11}
&\norm{\ip \brac{A(\fz)-A(\fp)}}_{\Hu}^2+\la\|\fz\|_{\Ho}  \\  \nonumber
\leq & 2 \norm{\sx^*\VE}_{\Ht}\norm{A(\fz)-A(\fp)}_{\Ht}\\  \nonumber
&+\norm{\tp-\tx}_{\mathcal{L}(\Ht)}\norm{A(\fz)-A(\fp)}_{\Ht}^2+\la\|\fp\|_{\Ho} 
\end{align}

\noindent {\bf High-probability bounds and choice of \(\eta_n\).} 
Apply Proposition~\ref{main.bound} with the choice
\[
\eta_n := \frac{1}{n^2}\qquad (n\ge 2).
\]
Proposition~\ref{main.bound} then yields constants \(C_1,C_2\) (depending only on model parameters \(\kappa,M,\Sigma\)) so that with probability at least \(1-\eta_n\) the following hold simultaneously:
\begin{align}
\|\sx^*\VE\|_{\Ht} &\le C_1 \frac{\log(1/\eta_n)}{\sqrt{n}} = C_1\frac{2\log n}{\sqrt{n}},\label{hp1}\\
\|\tx-\tp\|_{\mathcal L(\Ht)} &\le C_2\frac{\log(1/\eta_n)}{\sqrt{n}}= C_2\frac{2\log n}{\sqrt{n}}.\label{hp2}
\end{align}

Because \(\sum\limits_{n=1}^{\infty} \eta_n=\sum\limits_{n=1}^{\infty} n^{-2}<\infty\), the Borel--Cantelli lemma implies that, almost surely, for all sufficiently large \(n\) the high-probability inequalities \eqref{hp1}--\eqref{hp2} hold.

\noindent {\bf Almost sure boundedness of the estimator.}  
From \eqref{c1}, we get the bound
\[
\lambda\|f_{\zz,\lambda}\|_{\Ho}
\le
2\|S_{\xx}^*\VE\|_{\Ht}\|A(f_{\zz,\lambda})-A(\fp)\|_{\Ht}
+\lambda\|\fp\|_{\Ho}.
\]

Hence
\begin{equation}\label{ineq:fz-bound}
\norm{\fz}_{\Ho} \le \frac{2}{\la} \norm{\sx^* \VE}_{\Ht} \norm{A(\fz)-A(\fp)}_{\Ht} + \norm{\fp}_{\Ho}.
\end{equation}

Using Assumption \ref{ass:Lipschitz}, we obtain
\begin{equation*}
\norm{\fz}_{\Ho} \leq \frac{2\ella}{\la} \norm{\fz - \fp}_{\Ho} \norm{\sx^* \VE}_{\Ht} + \norm{\fp}_{\Ho}.
\end{equation*}

Applying the triangle inequality yields
\begin{align*}
 \norm{\fz}_{\Ho}  
&\leq \frac{2\ella}{\la} \norm{\sx^* \VE}_{\Ht} (\norm{\fz }_{\Ho} + \norm{\fp}_{\Ho}) +  \norm{\fp}_{\Ho}.
\end{align*}

From \eqref{hp1}, for almost every sample $\zz$ and all sufficiently large $n$, we have
\[
\|\fz \|_{\Ho} \le C'_1 \frac{\log n}{\lambda \sqrt{n}} \, (\norm{\fz }_{\Ho} + \norm{\fp}_{\Ho}) +  \|\fp\|_{\Ho},
\]
which implies
\[
\paren{1-C'_1 \frac{\log n}{\lambda \sqrt{n}}}  \|\fz \|_{\Ho} \le \paren{1+C'_1 \frac{\log n}{\lambda \sqrt{n}} } \,  \|\fp\|_{\Ho}.
\]

By the choice \eqref{la.choice}, we have $\frac{\log n}{\lambda \sqrt{n}} \to 0$ as $n \to \infty$. 
Therefore, it follows that
\begin{equation}\label{eq:limsup-bound}
 \limsup_{n \to \infty} \|\fz \|_{\Ho} \le  \|\fp\|_{\Ho} \qquad \text{almost surely.}   
\end{equation}

Consequently, $(f_{\zz,\lambda})$ is eventually bounded in $\ell^1$  almost surely.

\noindent
{\bf Weak-* compactness.}
Fix a sample $\zz$ for which~\eqref{eq:limsup-bound} holds.
Then,
\[
\sup_n \|f_{\zz,\lambda}\|_{\ell^2} \leq \sup_n \|f_{\zz,\lambda}\|_{\ell^1}<\infty.
\]
By the Banach--Alaoglu--Bourbaki theorem, we know that there exist a subsequence $f_{n_k}:=f_{\zz(n_k),\lambda}$ and an element $\tilde f\in \ell^2$ such that $f_{n_k}\rightharpoonup \tilde f$ in $\ell^2$. This theorem also implies that $\tilde{f}\in \ell^1$ and that
\begin{equation}\label{eq:f-to-hf}
f_{n_k}\rightharpoonup^* \tilde f
\qquad\text{in }\ell^1 .
\end{equation}
By the lower-semicontinuity of the $\| \cdot \|_{\ell^1}$ norm (which can be proved by Fatou's lemma) and \eqref{eq:limsup-bound}
we can conclude that
\begin{equation}\label{eq:limsup-bound.ineq}
\|\tilde f\|_{\ell^1}
\le
\liminf_{k\to\infty}\|f_{n_k}\|_{\ell^1}
\le
\limsup_{k\to\infty}\|f_{n_k}\|_{\ell^1}
\le
\|\fp\|_{\ell^1}
\qquad \text{almost surely}.
\end{equation}

\medskip
\noindent
{\bf Almost sure convergence of the residual.}
Substituting the high-probability bounds \eqref{hp1} and \eqref{hp2} into \eqref{c11}, we obtain that for almost every sample $\zz$ and all sufficiently large $n$,
\begin{align}
\|\ip(A(f_{\zz,\lambda})-A(\fp))\|_{\Hu}^2
\le &
\Big(2C_1\frac{2\log n}{\sqrt n}\Big)\|A(f_{\zz,\lambda})-A(\fp)\|_{\Ht}
\nonumber\\
&
+
\Big(C_2\frac{2\log n}{\sqrt n}\Big)\|A(f_{\zz,\lambda})-A(\fp)\|_{\Ht}^2
+
\lambda\|\fp\|_{\Ho}.
\end{align}

Since \((f_{\zz,\lambda})\) is bounded in \(\ell^1\), Assumption~\ref{ass:Lipschitz} implies that
\(
\|A(f_{\zz,\lambda})-A(\fp)\|_{\Ht}
\)
is uniformly bounded. Further, we have $\lambda\to0$ and $\frac{\log n}{\lambda\sqrt n}\to0$ by \eqref{la.choice}, hence, the right-hand side converges to zero. Consequently,
\begin{equation}\label{eq:residual-conv}
\|\ip(A(f_{\zz,\lambda})-A(\fp))\|_{\Hu}
\to 0
\qquad\text{almost surely as } n\to\infty .
\end{equation}
\medskip
\noindent
{\bf Identification of the limit.}
Consider again the subsequence $f_{n_k}$ which converges to a limit $\tilde{f}$ both in the weak $\ell^2$ and in the weak-* $\ell^1$ topologies.
Since $S_\mu \circ A$ is weak-to-weak continuous, we have that
\[
S_\mu A f_{n_k} \rightharpoonup S_\mu A \tilde{f} \qquad \text{ in } \Hu.
\]
On the other hand, the equation \eqref{eq:residual-conv} gives strong convergence
\[
S_\mu A f_{n_k} \to S_\mu A \fp \qquad \text{ in } \Hu, 
\]
Since strong convergence implies weak convergence and weak limits are unique, this implies that $S_\mu A \tilde{f} = S_\mu A \fp$. Finally, 
since the operator $S_\mu\circ A$ is injective, we conclude
\[
\tilde f=\fp .
\]

\noindent {\bf Strong a.s.\ convergence.}  
Using $\tilde f=\fp$ in~\eqref{eq:limsup-bound.ineq} we get almost surely
\[
\lim_{k\to\infty}\|f_{n_k}\|_{\Ho} = \|\fp\|_{\Ho}.
\]
Via Scheffé's lemma, the weak-* convergence and the norm convergence in $\ell^1$ imply the strong convergence, i.e.,
\[
\|f_{\zz(n_k),\lambda}-\fp\|_{\ell^1}\to0.
\]
Finally, since every subsequence of $f_{\zz,\la}$ has a further subsequence $f_{\zz(n_k),\lambda}$ that converges strongly to \(\fp\), we conclude that the entire sequence converges strongly almost surely:
\begin{equation*}
  \|f_{\zz,\la}-\fp\|_{\ell^1} \to 0 \quad \text{a.s.}  
\end{equation*}

This proves \eqref{consistency.rate} and completes the proof.
\end{proof}



\begin{remark}
Since $\ip:\Ht\to\Hu$ is linear and bounded, with
$\|\ip\|_{\Ht\to\Hu}\le \kappa$, it is weak-to-weak sequentially
continuous. Suppose that the (Lipschitz continuous) operator
\(A:\DD(A)\cap\Ho\to\Ht\)
is either linear and bounded, or nonlinear, continuous, and compact.
In the latter case, $A$ is weak-to-strong sequentially continuous. Since strong convergence implies
weak convergence, it follows that the composite operator
$\ip\circ A:\DD(A)\cap\Ho\to\Hu$ is weak-to-weak sequentially
continuous. Moreover, if the domain $\DD(A)$ is weakly closed, then
$\ip\circ A$ is weakly sequentially closed.\footnote{That is, if
$(f_m)_{m\in\NN}\subset\DD(A)$ satisfies
$f_m\rightharpoonup f$ in $\Ho$ and
$\ip A(f_m)\rightharpoonup g$ in $\Hu$, then
$f\in\DD(A)$ and $g=\ip A(f)$.}

Consequently, the regularization functional in~\eqref{eq:l1reg}
admits at least one global minimizer under the above assumptions,
although uniqueness is generally not guaranteed because of the possible
nonlinearity of $A$ (see~\cite[Section~4.1.1]{Schuster}).
\end{remark}

\section{Upper Convergence Rates}
\label{Sec:convergence.rates}

In this section, we establish non-asymptotic convergence rates for the proposed estimator under suitable regularity and complexity assumptions. The convergence analysis presented here provides a structured framework to understand how the estimator $\fz$ approximates the true solution $\fp$ in both the reconstruction norm $\|\fz - \fp\|_{\Ho}$ and the prediction norm $\|\ip[A(\fz) - A(\fp)]\|_{\Hu}$.  
We begin by relating the regularization parameter $\la$ and the sample size $n$ through the following condition:
\begin{equation}\label{l.la.condition}
    \mathcal{N}(\la) \le n\la, 
    \qquad 0 < \la \le 1.
\end{equation}
This condition ensures that the regularization level is compatible with the effective sample complexity determined by the covariance operator $\tp$.

\medskip
We introduce the following auxiliary quantities that characterize the stochastic and approximation errors:
\begin{align}
\Theta_{\zz} &:= \big\|(\tp +\la I)^{-1/2}\sx^*\VE\big\|_{\Ht}, 
\qquad \text{where } \VE = \yy - \sx\!\big[A(\fp)\big], 
\label{theta.z} \\
\Psi_{\xx} &:= \big\|(\tp +\la I)^{-1/2}(\tp - \tx)\big\|_{\mathcal{L}(\Ht)}.
\label{psi}
\end{align}
Here, $\Theta_{\zz}$ captures the fluctuations in the solution induced by noise, while $\Psi_{\xx}$ quantifies the error arising from finite and random sampling, representing the deviation between the empirical and true covariance operators.

\medskip

We define the error functional appearing in subsequent estimates as
\begin{equation}\label{eqn:err.fun}
	\err(f) := 
    \big\|\sx \!\big[A(\fp)\big] - \yy \big\|_n^2
    - \big\|\sx \!\big[A(f)\big] - \yy \big\|_n^2
    + \big\|\ip \!\big[A(f) - A(\fp)\big]\big\|_{\Hu}^2.
\end{equation}

The following lemma establishes a fundamental inequality relating the estimation error of $\fz$ to the error functional $\err(\fz)$ and the variational source condition $\phi$. This inequality serves as the starting point for all subsequent bounds.  This lemma formalizes how regularization interacts with the source smoothness to control the estimation error.

\begin{lemma}\label{lemma:pre}
Suppose that Assumption~\ref{Ass:var.sour} holds. Then, we obtain
\begin{align*}
\big\|\ip \!\big[A(\fz) - A(\fp)\big]\big\|_{\Hu}^2 
+ \la \big\|\fz - \fp\big\|_{\Ho}
\leq \err(\fz)
+ \la\,\phi\!\left(
    \big\|\ip \!\big[A(\fz) - A(\fp)\big]\big\|_{\Hu}^2
\right).
\end{align*}
\end{lemma}

\begin{proof}
Since $\fz$ minimizes $J_\la$, we have $J_\la(\fz) \le J_\la(\fp)$, which implies
\begin{align*}
    \norm{\fz-\fp}_{\Ho}
    \leq \frac{1}{\la}\paren{\norm{\sx \sbrac{A(\fp)}-\yy}_n^2-\norm{\sx \sbrac{A(\fz)}-\yy}_n^2}+\norm{\fz-\fp}_{\Ho}  + \norm{\fp}_{\Ho} - \norm{\fz}_{\Ho}.
\end{align*}

Applying the variational source condition, the last term can be bounded as
\begin{align*}
    \norm{\fz-\fp}_{\Ho}
    \leq \frac{1}{\la}\paren{\norm{\sx \sbrac{A(\fp)}-\yy}_n^2-\norm{\sx \sbrac{A(\fz)}-\yy}_n^2}+\phi\paren{\norm{\ip  \paren{A(\fz)-A(\fp)}}_{\Hu}^2}.
\end{align*}

Adding $\| \ip[A(\fz) - A(\fp)] \|_{\Hu}^2$ to both sides yields
\begin{align*}
&\norm{\ip\sbrac{A(\fz)-A(\fp)}}_{\Hu}^2  +\la \norm{\fz-\fp}_{\Ho}  \\  \nonumber
\leq & \norm{\ip \sbrac{A(\fz)-A(\fp)}}_{\Hu}^2 + \norm{\sx \sbrac{A(\fp)}-\yy}_n^2- \norm{\sx \sbrac{A(\fz)}-\yy}_n^2  \\
& + \la\phi\paren{\norm{\ip  \paren{A(\fz)-A(\fp)}}_{\Hu}^2}\\
= &\err(\fz)
+ \la\,\phi\!\left(
    \big\|\ip \!\big[A(\fz) - A(\fp)\big]\big\|_{\Hu}^2
\right),
\end{align*}
which establishes the claim.    
\end{proof}

\medskip
 
We now derive an upper bound for the error term $\err(\fz)$ defined in \eqref{eqn:err.fun} in terms of the perturbation quantities $\Psi_{\xx}$ and $\Theta_{\zz}$.
This bound separates the contributions from operator deviations, finite sampling, and noise, emphasizing how the error bound depends on these perturbation errors as well as on the reconstruction and prediction errors.

\begin{lemma}\label{lemma:err}
Suppose that Assumption~\ref{ass:Lipschitz} holds. Then, we have
\begin{equation*}
\err(\fz)
\leq
\Big(2\Theta_{\zz} + \Psi_{\xx}\,\ella \|\fz - \fp\|_{\Ho}\Big)
\Big(\ella \sqrt{\la}\,\|\fz - \fp\|_{\Ho}
+ \|\ip \!\big[A(\fz) - A(\fp)\big]\|_{\Hu}\Big).
\end{equation*}
\end{lemma}

\begin{proof}
We start by rewriting the error functional $\err(\fz)$ as
\begin{align*}
    \err(\fz) = &\norm{\ip\sbrac{A(\fz)-A(\fp)}}_{\Hu}^2-\norm{\sx \sbrac{A(\fz)-A(\fp)}}_n^2\\
    &-2\inner{\sx\sbrac{A(\fz)-A(\fp)}, \sx \sbrac{A(\fp )}-\yy}_n\\
    =&\inner{A(\fz)-A(\fp),(\tp-\tx)\sbrac{A(\fz)-A(\fp)}}_{\Ht}+2\inner{A(\fz)-A(\fp),\sx^*\VE}_{\Ht}\\
    =& \langle A(\fz)-A(\fp),(\tp -\tx)\sbrac{A(\fz)-A(\fp)} +2\sx^*\VE \rangle_{\Ht}.
\end{align*}
To bound this term, we use the decomposition
\begin{align*}
\inner{g,h}_{\Ht}= & \la\inner{g,(\tp+\la I)^{-1}h}_{\Ht}+\inner{g,\tp(\tp+\la I)^{-1}h}_{\Ht}\\
\leq & \brac{\sqrt{\la}\norm{g}_{\Ht}+\norm{\ip  g}_{\Hu}}\norm{(\tp+\la I)^{-1/2}h}_{\Ht}.
\end{align*}

Applying this inequality with~$g=A(\fz)-A(\fp)$, and~$h=(\tp-\tx)\sbrac{A(\fz)-A(\fp)}+2\sx^*\VE$ we obtain
\begin{align}\label{B3.1}
\err(\fz) = &\big\langle A(\fz)-A(\fp), (\tp - \tx)[A(\fz)-A(\fp)] + 2 \sx^* \VE \big\rangle_{\Ht}  \\
&\le \big( 2 \Theta_{\zz} + \Psi_{\xx} \| A(\fz)-A(\fp) \|_{\Ht} \big)  \nonumber  \\
&\quad \times \big( \sqrt{\lambda} \| A(\fz)-A(\fp) \|_{\Ht} + \| \ip[A(\fz)-A(\fp)] \|_{\Hu} \big) \nonumber \\
&\le \big( 2 \Theta_{\zz} + \Psi_{\xx} \ella \| \fz - \fp \|_{\Ho} \big) \nonumber \\
&\quad \times \big( \ella \sqrt{\lambda} \| \fz - \fp \|_{\Ho} + \| \ip[A(\fz)-A(\fp)] \|_{\Hu} \big), \nonumber
\end{align}
where~$\ella$,~$\Theta_\zz$,~$\Psi_{\xx}$ are defined in~Assumption~\ref{ass:Lipschitz},~\eqref{theta.z},~\eqref{psi}, respectively.
\end{proof}

\medskip

This theorem integrates Lemmas~\ref{lemma:pre} and~\ref{lemma:err} and establishes non-asymptotic, high-probability bounds for both the prediction and reconstruction errors.    
The term involving $\phi$ reflects the bias arising from the variational source condition, while the term $\lambda^{-b}/n$ encodes the variance due to finite sampling and the spectral complexity of the covariance operator.  
Overall, this theorem provides a clear quantitative relationship between the smoothness, sample size, and the effective degrees of freedom of the model.

In the following theorem, we make use of the Fenchel conjugate of the convex function $-\phi$, namely, $\phi^*:\mathbb{R} \to \mathbb{R} \cup \{+\infty\}$ defined as
\begin{equation}\label{Fenchel.conjugate}
  (-\phi)^*(s) := \sup_{t \ge 0} \big( st + \phi(t) \big),  
\end{equation}

%

%

\begin{theorem}\label{err.upper.bound.p.1}
Let Assumptions~\ref{ass:fp}--\ref{Ass:var.sour}, \ref{ass:polydecay} and condition~\eqref{l.la.condition} hold.  
Then, for all $0 < \eta < 1$, the following bounds hold with confidence $1 - \eta$:
\begin{align}\label{err1.3}
\big\|\ip[A(\fz) - A(\fp)]\big\|_{\Hu}^2
&\le
\left(
4\lambda(-\phi)^*\Big(-\frac{1}{4\lambda}\Big)
+ C''^2\,\frac{\lambda^{-b}}{n}
\right)
\log^2\!\left(\frac{4}{\eta}\right),
\\[2mm] \label{ps.3}
\big\|\fz - \fp\big\|_{\Ho}
&\le
\frac{1}{2\lambda}
\left(
4\lambda(-\phi)^*\Big(-\frac{1}{4\lambda}\Big)
+ C''^2\,\frac{\lambda^{-b}}{n}
\right)
\log^2\!\left(\frac{4}{\eta}\right),
\end{align}
where $C''$ depends on $\kappa$, $M$, $\Sigma$, $\ella$, and $\|\fp\|_{\Ho}$.
\end{theorem}

\begin{proof}
Starting from the estimates in Lemmas~\ref{lemma:pre} and~\ref{lemma:err}, we have
\begin{align*}
&\| \ip[A(\fz)-A(\fp)] \|_{\Hu}^2 + \lambda \| \fz - \fp \|_{\Ho} \\
&\le \lambda \phi\big(\| \ip[A(\fz)-A(\fp)] \|_{\Hu}^2 \big) \\
&\quad + (\ella \| \fz-\fp \|_{\Ho} \Psi_{\xx} + 2 \Theta_{\zz}) \| \ip[A(\fz)-A(\fp)] \|_{\Hu} \\
&\quad + \ella \sqrt{\lambda} (\ella \| \fz-\fp \|_{\Ho} \Psi_{\xx} + 2 \Theta_{\zz}) \| \fz-\fp \|_{\Ho}.
\end{align*}

Applying the inequality $ab \le \frac{a^2}{2} + \frac{b^2}{2}$ to the cross terms gives
\begin{align*}
&\| \ip[A(\fz)-A(\fp)] \|_{\Hu}^2 + \lambda \| \fz - \fp \|_{\Ho} \\
&\le \lambda \phi\big(\| \ip[A(\fz)-A(\fp)] \|_{\Hu}^2 \big) \\
&\quad + \frac{1}{2} (\ella \| \fz-\fp \|_{\Ho} \Psi_{\xx} + 2 \Theta_{\zz})^2 + \frac{1}{2} \| \ip[A(\fz)-A(\fp)] \|_{\Hu}^2 \\
&\quad + \frac{\ella^2}{2} (\ella \| \fz-\fp \|_{\Ho} \Psi_{\xx} + 2 \Theta_{\zz})^2 \| \fz-\fp \|_{\Ho} + \frac{\lambda}{2} \| \fz-\fp \|_{\Ho}.
\end{align*}

Rearranging terms yields
\begin{align*}
&\frac{1}{2} \| \ip[A(\fz)-A(\fp)] \|_{\Hu}^2 + \frac{\lambda}{2} \| \fz - \fp \|_{\Ho} \\
&\le \lambda \phi\big(\| \ip[A(\fz)-A(\fp)] \|_{\Hu}^2 \big) \\
&\quad + \frac{1}{2} (\| \fz-\fp \|_{\Ho} \ella^2 + 1)(\ella \| \fz-\fp \|_{\Ho} \Psi_{\xx} + 2 \Theta_{\zz})^2.
\end{align*}

Multiplying both sides by~$4$ and rearranging the resulting inequality yields
\begin{align*}
&\norm{\ip \sbrac{A(\fz)-A(\fp)}}_{\Hu}^2+2\la\norm{\fz-\fp}_{\Ho} \\
\leq &\paren{4\la \phi\paren{\norm{\ip \sbrac{A(\fz)-A(\fp)}}_{\Hu}^2}-\norm{\ip \sbrac{A(\fz)-A(\fp)}}_{\Hu}^2}\\
&+2(\norm{\fz-\fp}_{\Ho}\ella^2+1)(\ella\norm{\fz-\fp}_{\Ho}\Psi_{\xx}+2\Theta_{\zz})^2.
\end{align*}

Using the dual formulation of $\phi$ defined in \eqref{Fenchel.conjugate}, we get
\begin{align*}
&\norm{\ip \sbrac{A(\fz)-A(\fp)}}_{\Hu}^2+2\la\norm{\fz-\fp}_{\Ho} \\
\leq & \sup_{\tau\geq 0}\paren{4\la \phi(\tau)-\tau}+2(\norm{\fz-\fp}_{\Ho}\ella^2+1)(\ella\norm{\fz-\fp}_{\Ho}\Psi_{\xx}+2\Theta_{\zz})^2 \\
= & 4\la\paren{-\phi}^*\paren{-\frac{1}{4\la}}+2(\ella\norm{\fz-\fp}_{\Ho}\ella+1)(\ella\norm{\fz-\fp}_{\Ho}\Psi_{\xx}+2\Theta_{\zz})^2.
\end{align*}

Using the bound \eqref{Psi.bound} for the perturbation terms and the fact that $\|\fz - \fp\|_{\Ho}$ is bounded by Proposition~\ref{prop:bounded.solution}, we obtain, with probability $1-\eta$,

\begin{align*}
   &\norm{\ip \sbrac{A(\fz)-A(\fp)}}_{\Hu}^2+2\la\norm{\fz-\fp}_{\Ho}  \\
\leq & 4\la\paren{-\phi}^*\paren{-\frac{1}{4\la}}+C'\frac{\mathcal{N}(\la)}{n}\log^2\left(\frac{4}{\eta}\right)\\
\leq & 4\la\paren{-\phi}^*\paren{-\frac{1}{4\la}}+C''^2\frac{\la^{-b}}{n}\log^2\left(\frac{4}{\eta}\right).
\end{align*}

From this, the individual bounds follow:
\begin{equation*}
\norm{\ip \sbrac{A(\fz)-A(\fp)}}_{\Hu}^2\leq  \brac{4\la\paren{-\phi}^*\paren{-\frac{1}{4\la}}+C''^2\frac{\la^{-b}}{n}}\log^2\left(\frac{4}{\eta}\right)
\end{equation*}
and
\begin{equation*}
\norm{\fz-\fp}_{\Ho}\leq \frac{1}{2\la}\brac{4\la\paren{-\phi}^*\paren{-\frac{1}{4\la}}+C''^2\frac{\la^{-b}}{n}}\log^2\left(\frac{4}{\eta}\right),
\end{equation*}
where~$C',~C''$ depends on~$\kappa$,~$M$,~$\Sigma$,~$\ella$,~$\norm{\fp}_{\Ho}$.
\end{proof}

\medskip

The following corollary shows how to choose the regularization parameter $\lambda_*$ optimally based on the dual of the function $\psi^*$.  
By selecting $\lambda_*$ such that 
$\frac{1}{(\lambda_*)^{b+1}} \in \partial \psi^*\Big(-\frac{C''}{4n}\Big),
$
the bound simultaneously balances the bias and variance contributions to the estimation error, leading to the tightest high-probability guarantees for a given sample size.  
This step is crucial in practice since it prescribes a data-dependent way to tune the regularization parameter according to both the smoothness of $\fp$ and the spectral decay of $\tp$.

It is now straightforward to derive an explicit upper bound on the rate of convergence expressed directly in terms of the sample size $n$. 

\begin{corollary}\label{err.upper.bound}
Under the assumptions of Theorem~\ref{err.upper.bound.p.1}, let $\la = \la^*$ be chosen such that
\[
\frac{1}{(\la^*)^{b+1}} \in \partial \psi^*\!\left(-\frac{C''}{4n}\right),
\quad \text{for } 
\psi(x) = (-\phi)^*\!\left(-\frac{x^{\frac{1}{b+1}}}{4}\right).
\]
Then, for all~$0 < \eta < 1$, the following bounds hold with confidence~$1 - \eta$:
\begin{align*}
\big\|\ip \!\big[A(\fz) - A(\fp)\big]\big\|_{\Hu}^2
&\le
-4\la^*\,\psi^*\!\left(-\frac{C''}{4n}\right)
\log^2\!\left(\frac{4}{\eta}\right),\\
\big\|\fz - \fp\big\|_{\Ho}
&\le
-2\,\psi^*\!\left(-\frac{C''}{4n}\right)
\log^2\!\left(\frac{4}{\eta}\right),
\end{align*}
where $C''$ depends on $\kappa$, $M$, $\Sigma$, $\ella$, and $\|\fp\|_{\Ho}$.
\end{corollary}

\begin{proof}


By taking the infimum over the regularization parameter~$\la$ in the inequality~\eqref{ps.3} we get
$$\norm{\fz-\fp}_{\Ho}\leq  \inf\limits_{\la>0} 2\brac{\paren{-\phi}^*\paren{-\frac{1}{4\la}}+\frac{C''^2}{4\la^{b+1}n}}\log^2\left(\frac{4}{\eta}\right).$$

By assuming~$x=\frac{1}{\la^{b+1}}$ and~$\psi(x)=\paren{-\phi}^*\paren{-\frac{x^{\frac{1}{b+1}}}{4}}$ we obtain
\begin{align*} 
    \norm{\fz-\fp}_{\Ho}\leq &  \inf\limits_{x>0} 2\brac{\paren{-\phi}^*\paren{-\frac{x^{\frac{1}{b+1}}}{4}}+\frac{C''x}{4n}}\log^2\left(\frac{4}{\eta}\right)\\
    = & -\sup\limits_{x>0} 2\brac{-\frac{C''x}{4n}-\psi(x)}\log^2\left(\frac{4}{\eta}\right)\\
    = &-2\psi^*\paren{-\frac{C''}{4n}}\log^2\left(\frac{4}{\eta}\right).
\end{align*}
The infimum is attained at $\la=\la_*$ corresponds to $\frac{1}{\la_*^{b+1}}\in\partial\psi^*\paren{-\frac{C''}{4n}}$.
\end{proof}

\medskip
 
Under the polynomial variational source condition $\phi(x) = x^r$, the preceding results yield explicit convergence rates.
These expressions clearly show the influence of source smoothness $r$ and spectral decay $b$ on the learning rate. As $r \to 1$ (smoother solutions), the convergence accelerates, reflecting reduced bias.  As $b \to 0$ (faster spectral decay), the effective model complexity decreases, reducing variance.  
Thus, the rates interpolate naturally between well-posed and mildly ill-posed inverse problems.  
The logarithmic factors $\log^2(4/\eta)$ account for high-probability deviations due to sampling noise.

\begin{corollary}
\label{cor:err.upper.power}
Under the assumptions of Theorem~\ref{err.upper.bound.p.1},  
suppose the variational source condition holds with an index function of the form 
$\phi(x) = x^r$ for some $0 < r < 1$.  
If the regularization parameter is chosen as 
\[
\lambda_* = n^{-\frac{1 - r}{1 + b - br}},
\]
then for all $0 < \eta < 1$, the following bounds hold with probability at least $1 - \eta$:
\begin{align*}
\big\| A(\fz) - A(\fp) \big\|_{\Hu}^2
&\le
C\, n^{-\frac{1}{1 + b - br}} 
\log^2\!\left(\frac{4}{\eta}\right),\\[4pt]
\| \fz - \fp \|_{\Ho}
&\le
C\, n^{-\frac{r}{1 + b - br}} 
\log^2\!\left(\frac{4}{\eta}\right),
\end{align*}
where $C$ is a constant depending on $\kappa$, $M$, $\Sigma$, $\ella$, and $\|\fp\|_{\Ho}$.
\end{corollary}

\medskip

The results above establish explicit high-probability convergence rates for the estimator in both the prediction and reconstruction norms.  
The rates depend on the source smoothness parameter~$r$ and the spectral decay exponent~$b$, illustrating the trade-off between regularity and complexity in the inverse learning problem. The parameter choice $\lambda_* = n^{-(1-r)/(1+b-br)}$ balances the bias and variance terms optimally under the prior $\PP_{r,b}$, yielding the uniform upper rate~(\ref{eq:limit.uniform}).  

\medskip

For each $p \in (0,2]$, we introduce a family of weights defined by  
\begin{equation}\label{eq:omega_t}
(\underline{u}_p)_j := w_j^{\frac{2p - 2}{p}}.
\end{equation}
This construction yields a continuous scale of weighted sequence spaces that interpolate between the unweighted and weighted cases.  
In particular, for $p = 1$, we recover the standard space $\ell^1$, whereas for $p = 2$, we obtain the weighted Hilbert space $\ell_{\underline{w}}^2$.

\medskip

By combining the previously derived bounds with Assumption~\ref{Ass:A-Lip} (ii) and applying Proposition~\ref{prop:interpolation} with the parameters $\theta = 2\!\left(1 - \frac{1}{p}\right)$, $q = 1$, $s = 2$, and $f = \fz - \fp$, we establish the following convergence bounds expressed in the interpolation norms.

\begin{corollary}
\label{cor:err.upper.weighted}
Assume the hypotheses of Theorem~\ref{err.upper.bound.p.1} and Assumption~\ref{Ass:A-Lip} (ii) hold, in particular that the variational source condition holds with an index function $\phi(x)=x^r$ for some $0<r<1$.
Choosing 
\[
\lambda_* = n^{-\frac{1 - r}{1 + b - br}}
\]
and $p \in (0,2]$, we have, for all $0 < \eta < 1$, with probability at least $1 - \eta$,
\begin{equation}\label{bound.up}
  \| \fz - \fp \|_{\underline{u},p}
\le
C\, n^{-\frac{2r - pr + p - 1}{p(1 + b - br)}} 
\log^{\frac{2}{p}}\!\left(\frac{4}{\eta}\right),  
\end{equation}
where $C$ depends on $\kappa$, $M$, $\Sigma$, $\ella$, and $\|\fp\|_{\Ho}$.
\end{corollary}

We give a choice of the regularization parameter $\lambda$ as a function of the sample size~$n$, which provides a rate of decay of the interpolation norm uniform on the prior class~$\PP_{r,b}$. 

\begin{corollary}[Uniform upper rate under data-driven $\lambda_*$]\label{thm:uniform.rate.lambda*}
Assume the hypotheses of Corollary~\ref{cor:err.upper.weighted} with $\lambda = n^{-\frac{1 - r}{1 + b - br}}$, $p \in (0,2]$ and $\pp \in [1,\infty)$. Let
\(
a_n := n^{-\frac{2r-pr+p-1}{p(1+b-br)}}.
\)
Then the sequence of estimators $(\fz)_n$ satisfies
\begin{equation}\label{eq:limit.uniform}
\limsup_{n\to\infty} \sup_{\rho \in \mathcal{P}_{r,b}}
\frac{\paren{\mathbb{E}_{\rho^n} 
\left[
\|\fz - \fp\|_{\underline{u},p}^\pp
\right]}^{1/\pp}}{a_n}
\, 
< \infty.
\end{equation}
\end{corollary}

\begin{proof}
Let
\[
X := \|\fz - \fp\|_{\underline{u},p}.
\]
From Corollary~\ref{cor:err.upper.weighted}, for any $0<\eta<1$, we have
\[
\mathbb{P}_{\rho^n}\!\left(
X \le C a_n \log^{2/p}\!\left(\frac{4}{\eta}\right)
\right)
\ge 1-\eta.
\]

Fix $t>0$ and set
\[
\eta = 4 \exp\!\left(-\frac{t^{p/2}}{C^{p/2} a_n^{p/2}}\right).
\]
Then, it follows that
\[
\mathbb{P}_{\rho^n}(X > t)
< 4 \exp\!\left(-c \frac{t^{p/2}}{a_n^{p/2}}\right),
\]
for some constant $c>0$ depending only on $p$.

We now use the identity
\[
\mathbb{E}_{\rho^n}[X^\pp]
=
\int_0^\infty \pp t^{\pp-1} \mathbb{P}_{\rho^n}(X>t)\,dt.
\]

Split the integral at $t_0 := a_n$,
\[
\mathbb{E}_{\rho^n}[X^\pp]
=
\int_0^{t_0} \pp t^{\pp-1} \mathbb{P}(X>t)\,dt
+
\int_{t_0}^\infty \pp t^{\pp-1} \mathbb{P}(X>t)\,dt.
\]

For the first term,
\[
\int_0^{t_0} \pp t^{\pp-1} \mathbb{P}(X>t)\,dt \le \int_0^{t_0} \pp t^{\pp-1} dt = a_n^\pp.
\]

For the second term, using the tail bound,
\[
\int_{t_0}^\infty \pp t^{\pp-1} \mathbb{P}(X>t)\,dt <\int_{t_0}^\infty 4 \pp t^{\pp-1}
\exp\!\left(-c \frac{t^p}{a_n^p}\right)\,dt.
\]

Perform the change of variables $u = (t/a_n)^p$, so that
\[
t = a_n u^{1/p}, \qquad dt = \frac{a_n}{p} u^{\frac{1}{p}-1} du.
\]

This yields
\[
\int_{t_0}^\infty 4 \pp t^{\pp-1}
\exp\!\left(-c \frac{t^p}{a_n^p}\right)\,dt
=
4\pp a_n^\pp
\int_1^\infty
u^{\frac{\pp}{p}-1}
e^{-cu}\,du.
\]

Since $\pp>0$ and $c>0$, the integral
\[
\int_1^\infty u^{\frac{\pp}{p}-1} e^{-cu}\,du
\]
is finite and bounded by a constant $C_{\pp,p,c}$ independent of $n$.

Hence,
\[
\int_{t_0}^\infty \pp t^{\pp-1}
\mathbb{P}_{\rho^n}(X>t)\,dt
\le
C\, a_n^\pp.
\]

Combining the two parts of the decomposition, we obtain
\[
\mathbb{E}_{\rho^n}[X^\pp]
\le
C a_n^\pp.
\]

Taking $\pp$-th roots yields
\[
\left(
\mathbb{E}_{\rho^n}[X^\pp]
\right)^{1/\pp}
\le
C a_n.
\]

Finally, taking the supremum over $\rho \in \mathcal{P}_{r,b}$ and the
$\limsup$ as $n \to \infty$ gives
\[
\limsup_{n\to\infty}
\sup_{\rho \in \mathcal{P}_{r,b}}
\frac{
\left(
\mathbb{E}_{\rho^n}
\big[
\|\fz - \fp\|_{\underline{u},p}^{\pp}
\big]
\right)^{1/\pp}
}{a_n}
< \infty.
\]

This completes the proof.

\end{proof}




The exponent $\frac{2r-pr+p-1}{p(1+b-br)}$ reflects the interplay between the smoothness index $r$, the spectral decay parameter $b$, and the interpolation norm index $p$.

\begin{remark}[Comparison with existing inverse learning results]
Compared with RKHS-based analyses such as Blanchard and Mücke \cite{BlaMuc16} and Rastogi et al. \cite{Rastogi20}, the present result replaces Hilbert-space regularization with sparsity-promoting $\ell^1$ regularization in a Banach-space setting. Moreover, convergence rates are derived under variational source conditions and are shown to be minimax optimal via matching lower bounds.
\end{remark}


\section{Lower Convergence Rates and Optimality}
\label{sec:lower.rates}

The convergence rates established in Section~\ref{Sec:convergence.rates} provide upper bounds on the statistical accuracy achievable by the proposed regularization method. A natural question is whether these rates are optimal, or whether faster convergence can be attained by alternative learning algorithms. To answer this question, we derive minimax lower bounds for the considered class of nonlinear statistical inverse problems. These lower bounds establish fundamental limitations that apply to every estimator, irrespective of the reconstruction procedure employed.

Our analysis follows an information-theoretic approach based on carefully constructed packing sets, the Kullback--Leibler divergence, and Fano's inequality. The lower bounds are obtained under the same approximation-space assumptions that characterize the sparsity of the unknown solution and are later shown to coincide with the variational source conditions used in the upper-bound analysis. Consequently, the resulting lower convergence rates match the upper rates established in the previous section, thereby demonstrating the minimax optimality of the proposed regularization method over the considered model class.

\subsection{Lower Minimax Bounds} \label{ssec:lower.rates}
To establish the lower convergence rate, we construct a family of probability measures~$\rho_f$ parameterized by suitable vectors~$f \in \DD(A)$. 
We assume that~$\YY$ is a finite-dimensional Hilbert space with an orthonormal basis~$\{v_j\}_{j=1}^d$.  
For each~$f \in \DD(A)$, define a probability measure on~$\XX \times \YY$ as
\begin{equation}\label{p.f}
\rho_f(dx,dy)
:= \frac{1}{2dJ}\sum_{j=1}^d
\Big(a_j(x)\,\delta_{y+dJv_j} + b_j(x)\,\delta_{y-dJv_j}\Big)\,\mu(dx),
\end{equation}
where
\[
a_j(x) = J - \langle A(f), K_x v_j \rangle_{\HH},
\qquad
b_j(x) = J + \langle A(f), K_x v_j \rangle_{\HH},
\qquad
J = 4\kappa \|A(f)\|_{\HH},
\]
and~$\delta_{y-\xi}$ denotes the Dirac measure at~$y=\xi$.  
It is straightforward to verify that the marginal distribution of~$\rho_f$ on~$\XX$ coincides with~$\mu$.

The following proposition is an adaptation of Proposition~4 in~\cite{Caponnetto07} to the setting of nonlinear statistical inverse problems.

\begin{proposition}
For the probability measure~$\rho_f$ defined in~\eqref{p.f} and parameterized by~$f \in \DD(A)$, the following hold:
\begin{enumerate}[(i)]
    \item The regression function corresponding to~$\rho_f$ is~$f$, that is, $\fp = f$.
    \item The probability measure~$\rho_f$ satisfies Assumption~\ref{ass:noise} if the following condition holds
    \begin{equation}\label{condition}
        dJ + \tfrac{J}{4} \le M 
        \quad \text{and} \quad 
        2dJ \le \Sigma.
    \end{equation}
\end{enumerate}
\end{proposition}

Therefore, Assumptions~\ref{ass:fp} and~\ref{ass:noise} hold for the family of probability measures~$\{\rho_f\}_{f \in \DD(A)}$ defined above.  
Furthermore, assume that the eigenvalues~$s_j$ of the covariance operator~$\tp$ corresponding to the marginal distribution~$\mu$ satisfy the polynomial decay condition
\[
s_j \le \beta\, j^{-\frac{1}{b}},
\]
which ensures that Assumption~\ref{ass:polydecay} is also fulfilled~\cite{Caponnetto07}.

For $t \in (0,1)$, the \textit{approximation space} $k_t$ is defined as (see \cite{Miller21})
\begin{equation}
k_t := \Big\{ f : \|f\|_{k_t} < \infty \Big\},
\qquad
\|f\|_{k_t} := \sup_{\alpha > 0} 
\alpha \Big( \sum_{j=1}^\infty w_j^{-2}\, \mathbbm{1}_{\{\alpha \le w_j^2 |f_j|\}} \Big)^{1/t}.
    \label{eq:k_t.space}
\end{equation}
From Lemma~\ref{kt.character}, if $f \in k_t$ for some $t \in (0,1)$ and Assumption~\ref{Ass:A-Lip} (ii) holds, then the variational source condition in Assumption~\ref{Ass:var.sour} is satisfied with $\phi(s) = s^{\frac{1-t}{2-t}}$.

\medskip
Information-theoretic tools, such as the Kullback–Leibler (KL) divergence and Fano’s inequality (Lemma~3.3 in~\cite{DeVore06}), play a central role in deriving the lower bounds.  
For two probability measures~$\rho_1$ and~$\rho_2$, the KL divergence is defined by
\[
\mathcal{K}(\rho_1, \rho_2)
:= \int_Z \log\!\left( \frac{d\rho_1}{d\rho_2} \right)\, d\rho_1
= \int_Z g(z)\,\log g(z)\, d\rho_2(z),
\]
provided that~$\rho_1$ is absolutely continuous with respect to~$\rho_2$, where~$g = \frac{d\rho_1}{d\rho_2}$ denotes the Radon–Nikodym derivative of~$\rho_1$ with respect to~$\rho_2$.  
If~$\rho_1$ is not absolutely continuous with respect to~$\rho_2$, we set~$\mathcal{K}(\rho_1, \rho_2) := \infty$.  
By definition, for all measurable sets~$E \subset Z$, $\rho_1(E) = \int_E g(z)\, d\rho_2(z)$.

\medskip
The proof of the lower bound theorem relies on a combinatorial packing result concerning binary strings.

\begin{proposition}\label{prop:VG.signpacking}
Let $\ell \in \mathbb{N}$ with $\ell > 16$. Then there exist
$N_\ell \in \mathbb{N}$ and vectors
\[
\sigma_1,\dots,\sigma_{N_\ell} \in \{\pm 1\}^{\ell}
\]
such that
\begin{equation}\label{eq:VG.hamming}
\mathrm{Ham}(\sigma_i,\sigma_j) \;\ge\; \frac{\ell}{4},
\qquad \text{for all } i\ne j,
\end{equation}
and
\begin{equation}\label{eq:VG.size}
N_\ell \;\ge\; \exp\left(\frac{\ell}{24}\right).
\end{equation}
In particular, any two distinct vectors satisfy
\[
\sum_{n=1}^{\ell} |\sigma_i^n - \sigma_j^n|
\;\ge\;
\frac{\ell}{2},
\qquad \text{for all } i\ne j,
\]
where $\sigma_i = (\sigma_i^{1}, \ldots, \sigma_i^{\ell})$ and 
$\sigma_j = (\sigma_j^{1}, \ldots, \sigma_j^{\ell})$.
\end{proposition}

\begin{proof}
Let $\sigma,\sigma' \in \{\pm1\}^{\ell}$ be independent and uniformly distributed, and
let $h = \mathrm{Ham}(\sigma,\sigma')$, the number of coordinates where they differ.  Then $h \sim \mathrm{Bin}(\ell,1/2)$ and
$\mathbb{E}h = \ell/2$. Hoeffding's inequality gives
\[
\mathbb{P}\!\left(h \le \frac{\ell}{4}\right)
= \mathbb{P}\!\left(h - \frac{\ell}{2} \le -\frac{\ell}{4}\right)
\le \exp\left(-2\paren{\frac{\ell}{4}}^2/\ell\right) = \exp\left(-\frac{\ell}{8}\right)
\]

Set $M := \lceil e^{\ell/24} \rceil$. Then $M - 1 \le e^{\ell/24}$.  
The assumption $\ell > 16$ ensures that $M < (M - 1)^2$.  

Draw $\sigma_1, \dots, \sigma_M$ independently and uniformly from $\{\pm 1\}^\ell$.  
The number of unordered pairs is at most $\binom{M}{2}$. Hence, by the union bound,
\[
\mathbb{P}\!\left(\exists\, i \ne j : \mathrm{Ham}(\sigma_i, \sigma_j) < \frac{\ell}{4}\right)
\le \frac{M(M - 1)}{2}\, \exp\!\left(-\frac{\ell}{8}\right).
\]
Using the inequalities above, we obtain
\[
\frac{M(M - 1)}{2}\, \exp\!\left(-\frac{\ell}{8}\right)
\le \frac{M(M - 1)}{2}\, \frac{1}{(M - 1)^3}
< \frac{1}{2}.
\]

Hence, there exists a choice of $\{\sigma_i\}_{i=1}^{M}$ with
$\mathrm{Ham}(\sigma_i,\sigma_j) \ge \ell/4$ for all $i\neq j$.
Since 
\[
|\sigma_i^n - \sigma_j^n| = 
\begin{cases}
0, & \sigma_i^n=\sigma_j^n,\\
2, & \sigma_i^n\ne \sigma_j^n,
\end{cases}
\]
we have
\[
\sum_{n=1}^{\ell}|\sigma_i^n - \sigma_j^n| = 2 \,\mathrm{Ham}(\sigma_i,\sigma_j)
\ge 2 \cdot \frac{\ell}{4}
= \frac{\ell}{2}.
\]
Thus, $N_\ell = M \ge \exp\left(\frac{\ell}{24}\right)$, concluding the proof.
\end{proof}

To derive the minimax lower bounds, we briefly recall the general reduction scheme described in Chapter~2 of~\cite{Tsybakov2008}.  
For each~$0 < \epsilon \le \epsilon_0$, we construct a finite family 
$\{f^i\}_{i=1}^{N_\epsilon} \subset D(A) \cap k_t$ satisfying
\[
N_\epsilon \ge \exp(\ell_\epsilon / 24), 
\qquad
\|f^i - f^j\|_{\underline{u},p} \ge 2\epsilon, 
\quad i \ne j,
\]
and
\[
\mathcal{K}\!\left(\rho_{f^i}^n, \rho_{f^j}^n\right)
\le c\, n\, \epsilon^{\beta_0}\, \ell_\epsilon^{\beta_1},
\qquad i \ne j.
\]
Applying Lemma~3.3 of~\cite{DeVore06} (Fano’s method) to this family yields the corresponding minimax lower bound for the estimation error.

\medskip

\begin{theorem}\label{err.lower.bound}
Let $\zz$ be i.i.d.~samples drawn from a probability measure
$\rho\in \mathcal{P}_{r,b}$ with $\dim(\YY)=d<\infty$,
$r=\frac{1-t}{2-t}$, $t\in(0,1)$, and $p\in(t,2]$.
Assume that $\ell_\epsilon$ is a positive decreasing function of $\epsilon$.
Suppose further that there exist constants $c_0,\epsilon_0>0$ (with $\ell_{\epsilon_0}>16$) and $q>1$
such that, for every $\zeta\in(0,c_0)$ and every
$\epsilon\in(0,\epsilon_0)$, there exists a set of
$\ell_\epsilon$ indices
\(
M=\{m_1,\ldots,m_{\ell_\epsilon}\}
\)
for which
\[
\zeta \Bigl(\epsilon^p\ell_\epsilon^{-1}\Bigr)^{\!\frac{t}{2(p-t)}}
    \;\le\; w_m
    \;\le\; q\,\zeta \Bigl(\epsilon^p\ell_\epsilon^{-1}\Bigr)^{\!\frac{t}{2(p-t)}},
\qquad
\forall\, m\in M.
\]
Then, for every learning algorithm
$l:\zz\mapsto f_{\zz}^{\,l}$,
there exists a probability measure
$\rho_*\in\mathcal{P}_{r,b}$ with regression function
$f_{\rho_*}$ such that, for all
$0<\epsilon<\epsilon_0$,
\[
\mathbb P_{\rho_*^n}\!\left\{
\|f_{\zz}^{\,l}-f_{\rho_*}\|_{\underline u,p}
>\epsilon
\right\}
\ge
\min\!\left\{
\frac12,\;
\vartheta
\exp\!\left(
\frac{\ell_\epsilon}{48}
-
c\,n\,
\epsilon^{\frac{p(2-t)}{p-t}}
\ell_\epsilon^{\frac{p-2}{p-t}}
\right)
\right\},
\]
where $\vartheta=e^{-3/e}$.
\end{theorem}

\begin{proof} 

Choose $\epsilon_0>0$ as above such that $\ell_{\epsilon_0}>16$.
Then, by Proposition~\ref{prop:VG.signpacking}, for every $0 < \epsilon < \epsilon_0$ (so that $\ell_\epsilon > \ell_{\epsilon_0}$), there exist $N_\epsilon \in \NN$ and sign vectors
\[
\sigma_i = (\sigma_i^1, \dots, \sigma_i^{\ell_\epsilon}) \in \{\pm1\}^{\ell_\epsilon}, \qquad i = 1, \dots, N_\epsilon,
\]
such that
\begin{equation}\label{sigma.i.j}
\mathrm{Ham}(\sigma_i,\sigma_j) \;\ge\; \frac{\ell_{\epsilon}}{4},
\qquad \text{for all } i\ne j,
\end{equation}
and
\begin{equation}\label{N.epsilon.k}
N_\epsilon \ge \exp(\ell_\epsilon/24).
\end{equation}


For any $\varrho > \varrho_0$ such that $(4\varrho_0^{-p})^{\frac{t}{2(p-t)}} < c_0$, setting
\[
    \zeta \;=\; \Bigl(4\varrho^{-p}\Bigr)^{\!\frac{t}{2(p-t)}},
\]
we have that

there exists a set of indices $M = \{m_1, \dots, m_{\ell_\epsilon}\}$ such that
every $m \in M$ satisfies

\begin{equation}\label{w.scale}
    \Bigl(4\varrho^{-p}\epsilon^p\ell_\epsilon^{-1}\Bigr)^{\!\frac{t}{2(p-t)}}
    \;\le\; w_m
    \;\le\; q\,\Bigl(4\varrho^{-p}\epsilon^p\ell_\epsilon^{-1}\Bigr)^{\!\frac{t}{2(p-t)}}.
\end{equation}

For each $i=1,\dots,N_\epsilon$ define $f^i\in\RR^\infty$ by
\[
(f^i)_{m_s} := \sigma_i^s\,\varrho\, w_{m_s}^{\frac{2-2t}{t}},
\qquad s=1,\dots,\ell_\epsilon,
\]
and set $(f^i)_m = 0$ for $m \notin M$.

We now show that $f^i\in k_t$. By construction,
\[
w_{m_s}^2 |(f^i)_{m_s}|
=
\varrho\, w_{m_s}^{2/t},
\qquad s=1,\ldots,\ell_\epsilon,
\]
and $(f^i)_m=0$ for $m\notin M$. Therefore,
\[
\|f^i\|_{k_t}
=
\sup_{\alpha>0}
\alpha
\left(
\sum_{s=1}^{\ell_\epsilon}
w_{m_s}^{-2}
\mathbbm{1}_{\{\alpha\le \varrho w_{m_s}^{2/t}\}}
\right)^{1/t}.
\]
Moreover,
\[
\alpha \le \varrho\, w_{m_s}^{2/t}
\quad\Longrightarrow\quad
\alpha^t w_{m_s}^{-2}\le \varrho^t .
\]
Hence,
\begin{align*}
\|f^i\|_{k_t}
&=
\sup_{\alpha>0}
\left(
\sum_{s=1}^{\ell_\epsilon}
\alpha^t w_{m_s}^{-2}
\mathbbm{1}_{\{\alpha\le \varrho w_{m_s}^{2/t}\}}
\right)^{1/t} \\
&\le
\sup_{\alpha>0}
\left(
\sum_{s=1}^{\ell_\epsilon}
\varrho^t
\mathbbm{1}_{\{\alpha\le \varrho w_{m_s}^{2/t}\}}
\right)^{1/t}
\le
\varrho\,\ell_\epsilon^{1/t}
< \infty.
\end{align*}
Thus, $f^i\in k_t$. In view of Theorem~\ref{kt.vsc}, the functions $f^i$ satisfy the variational source condition in Assumption~\ref{Ass:var.sour} with index function $\phi(s)\asymp s^r$, where \(r=\frac{1-t}{2-t}\). Consequently, the corresponding probability measures $\rho_{f^i}$ belong to $\mathcal{P}_{r,b}$.

Fix $i \neq j$, and let $\Delta = \{s : \sigma_i^s \neq \sigma_j^s\}$ with $h :=\mathrm{Ham}(\sigma_i,\sigma_j) = |\Delta|$.  
For each $s \in \Delta$, we have
\[
|(f^i-f^j)_{m_s}| = 2\varrho\, w_{m_s}^{\frac{2-2t}{t}}.
\]
Since $u_{m}=w_m^{\frac{2p-2}{p}}$, the $p$-norm satisfies
\begin{equation}\label{p.norm.sum}
\|f^i-f^j\|_{\underline u,p}^p
= \sum_{s\in\Delta} u_{m_s}^p \, |(f^i-f^j)_{m_s}|^p
= 2^p \varrho^p \sum_{s\in\Delta} w_{m_s}^{\frac{2(p-t)}{t}}.
\end{equation}

Using the lower bound in \eqref{w.scale} we obtain
\[
\|f^i-f^j\|_{\underline u,p}^p
\ge 2^p\varrho^p \cdot h \cdot
\big(4\varrho^{-p}\epsilon^p\ell_\epsilon^{-1}\big)
= 2^p\epsilon^p\paren{\frac{4 h}{\ell_\epsilon}}.
\]
Since \(h\ge\ell_\epsilon/4\) this implies the convenient lower bound
\begin{equation}\label{f.i.j.bound.l}
  \|f^i-f^j\|_{\underline u,p} \ge 2\epsilon.  
\end{equation}

Now, we compute explicit upper bound for \(\|f^i-f^j\|_{\underline w,2}^2\):
\begin{equation}\label{2.norm.sum}
\|f^i-f^j\|_{\underline w,2}^2
= \sum_{s\in\Delta} w_{m_s}^2 \, |(f^i-f^j)_{m_s}|^2
= 4 \varrho^2 \sum_{s\in\Delta} w_{m_s}^{\frac{2(2-t)}{t}}.
\end{equation}

Set
\[
\alpha_0 := \frac{2(2-t)}{t},\qquad \alpha_1 := \frac{2-t}{p-t}.
\]
Using the upper bound in \eqref{w.scale} we get for each \(s\in\Delta\),
\[
w_{m_s}^{\alpha_0}
\le q^{\alpha_0}
\big(4\varrho^{-p}\epsilon^p\ell_\epsilon^{-1}\big)^{\alpha_0\frac{t}{2(p-t)}}
= q^{\alpha_0}\big(4\varrho^{-p}\epsilon^p\ell_\epsilon^{-1}\big)^{\alpha_1}.
\]
Summing over \(s\in\Delta\) (there are \(h\) terms) yields
\[
\|f^i-f^j\|_{\underline w,2}^2
\le 4\varrho^2\,h\,q^{\alpha_0}\,
\big(4\varrho^{-p}\epsilon^p\ell_\epsilon^{-1}\big)^{\alpha_1}.
\]
Rewrite this inequality as
\[
\|f^i-f^j\|_{\underline w,2}^2
\le 4^{1+\alpha_1}\, q^{\alpha_0}\, h\,
\varrho^{\,2 - p\alpha_1}\,\epsilon^{\,p\alpha_1}\,\ell_\epsilon^{-\alpha_1}.
\]
A short algebraic simplification of the exponent of \(\varrho\) gives
\[
2 - p\alpha_1 \;=\; \frac{t(p-2)}{p-t},
\qquad p\alpha_1 \;=\; \frac{p(2-t)}{p-t}.
\]
Therefore the explicit bound becomes
\begin{equation}\label{eq:explicit-w2bound}
\|f^i-f^j\|_{\underline w,2}^2
\le
\underbrace{4^{\,1+\frac{2-t}{p-t}}\,q^{\frac{2(2-t)}{t}}\,
\varrho^{\frac{t(p-2)}{p-t}}}_{=:A(p,t,q,\varrho)}
\;
\epsilon^{\frac{p(2-t)}{p-t}}\;
\ell_\epsilon^{-\frac{2-t}{p-t}}\;h.
\end{equation}

Since $\varrho \geq \varrho_0$ and $t <p\leq 2$, we have $\varrho^{\frac{t(p-2)}{p-t}} \leq \varrho_0^{\frac{t(p-2)}{p-t}}$, hence $A(p,t,q,\varrho) \leq A(p,t,q,\varrho_0)=: A(p,t,q)$

Using \(h\le\ell_\epsilon\) we also obtain the simpler bound
\[
\|f^i-f^j\|_{\underline w,2}^2
\le A(p,t,q)\,
\ell_\epsilon^{\,\frac{p-2}{p-t}}\,
\epsilon^{\frac{p(2-t)}{p-t}}.
\]

Define the sets
\[
A_i = \Bigl\{\,\zz : \|f_{\zz}^l - f^i\|_{\underline u,p} < \epsilon\Bigr\}, \qquad 1 \le i \le N_\epsilon.
\]
From~\eqref{f.i.j.bound.l}, the sets $A_i$ are pairwise disjoint.

Applying Fano's lemma (Lemma~3.3 from~\cite{DeVore06}) to the family of product measures $\rho_{f^i}^n$ ($1 \le i \le N_\epsilon$), we obtain that either
\begin{equation}\label{p.either.bound}
\mathrm{p} := \max_{1\le i \le N_\epsilon} \rho_{f^i}^n(A_i^c) \geq \frac{N_\epsilon}{N_\epsilon + 1},
\end{equation}
or
\begin{equation}\label{leibler.l.bound}
\min_{1 \le j \le N_\epsilon} \frac{1}{N_\epsilon}\sum_{\substack{i=1 \\ i \neq j}}^{N_\epsilon}
\mathcal{K}(\rho_{f^i}^n, \rho_{f^j}^n)
\geq \Psi_{N_\epsilon}(\mathrm{p}),
\end{equation}
where
\[
\Psi_{N_\epsilon}(\mathrm{p})
= \log(N_\epsilon) + (1-\mathrm{p})\log\!\left(\frac{1-\mathrm{p}}{\mathrm{p}}\right)
- \mathrm{p}\log\!\left(\frac{N_\epsilon - \mathrm{p}}{\mathrm{p}}\right).
\]

We now distinguish two cases. In the first one, we simply assume that \( \mathrm{p} \geq \tfrac{1}{2} \).  If instead \( \mathrm{p} < \tfrac{1}{2} \), we derive an alternative lower bound for \( \mathrm{p} \). From the dichotomy given by Fano’s lemma, since \( \tfrac{N_\epsilon}{N_\epsilon+1} \geq \tfrac{1}{2} \) for \( N_\epsilon \geq 1 \),   we conclude that in this case the bound \eqref{leibler.l.bound} holds.
Thus, we obtain
\[
\Psi_{N_\epsilon}(\mathrm{p})
\ge (1-\mathrm{p})\log N_\epsilon
+ (1-\mathrm{p})\log(1-\mathrm{p})
- \log \mathrm{p}
+ 2\mathrm{p}\log \mathrm{p}.
\]
Applying \(x\log x \ge -1/e , ~  \text{for } x \in (0,1]\) to the negative terms yields
\[
\Psi_{N_\epsilon}(\mathrm{p})
\ge (1-\mathrm{p})\log N_\epsilon - \log \mathrm{p} - \frac{3}{e}.
\]
Since we consider \(\mathrm{p} \le 1/2\), we further simplify \((1-\mathrm{p})\log N_\epsilon \ge \tfrac{1}{2}\log N_\epsilon\), hence
\begin{equation}\label{psi.p.bound}
\Psi_{N_\epsilon}(\mathrm{p})
\ge \log(\sqrt{N_\epsilon}) - \log \mathrm{p} - \frac{3}{e}.
\end{equation}

For the joint measures $\rho_{f^i}^n$ and $\rho_{f^j}^n$, $1\le i,j\le N_\epsilon$,  
Proposition~4 in~\cite{Caponnetto07}, together with Assumption~\ref{Ass:A-Lip} (i) and~\eqref{f.i.j.bound.l}, gives
\begin{equation*}
\mathcal{K}(\rho_{f^i}^n, \rho_{f^j}^n)
= n\,\mathcal{K}(\rho_{f^i}, \rho_{f^j}) 
\le \frac{16n}{15dJ^2}\norm{\ip\sbrac{A(f_i)-A(f_j)}}_{\Hu}^2
\le \frac{16nL^2}{15dJ^2}\|f^i - f^j\|_{\underline w,2}^2.
\end{equation*}

Using the explicit bound \eqref{eq:explicit-w2bound} we get, for every \(i\ne j\),
\[
\mathcal{K}(\rho_{f^i}^n,\rho_{f^j}^n)
\le \frac{16L^2}{15dJ^2} \, A(p,t,q)\cdot 
n\,
\epsilon^{\frac{p(2-t)}{p-t}}\,
\ell_\epsilon^{\,\frac{p-2}{p-t}}.
\]
Denote \(c := \frac{16L^2}{15dJ^2} \, A(p,t,q)\) and hence we get
\begin{equation}\label{leibler.u.bound}
\mathcal{K}(\rho_{f^i}^n,\rho_{f^j}^n) \le c \cdot 
n\,
\epsilon^{\frac{p(2-t)}{p-t}}\,
\ell_\epsilon^{\,\frac{p-2}{p-t}}.
\end{equation}

Combining \eqref{leibler.l.bound}, 
\eqref{psi.p.bound}, and \eqref{leibler.u.bound}, we obtain
\[
\mathrm{p} := \max_{1\le i\le N_\epsilon}
\mathbb{P}_{\rho^n_{f^i}}\!\left\{\zz : \|f_{\zz}^l - f^i\|_{\underline u,p} > \epsilon\right\}
\ge
\min\!\left\{
\frac{1}{2},
\sqrt{N_\epsilon}\, \exp\left(-\frac{3}{e} - c n \epsilon^{\frac{p(2-t)}{p-t}}\,
\ell_\epsilon^{\,\frac{p-2}{p-t}}\right)
\right\}.
\]
Finally, using~\eqref{N.epsilon.k}, for the probability measure $\rho_*$ satisfying $\mathrm{p} = \rho_*^n(A_i^c)$, the stated result follows.
\end{proof}


\medskip
We now refine this general lower bound to obtain an explicit convergence rate.

\begin{theorem} [Minimax lower rate in expectation]\label{err.lower.bound.p.para.corrected}
Assume the hypotheses of Theorem~\ref{err.lower.bound}, and $\pp \in [1,\infty)$.
Let $\mathcal{A}$ denote the class of all such learning algorithms 
$l : \zz \mapsto f_{\zz}^l$.  Let $a_n:=n^{-\frac{2r-pr+p-1}{p(1+b-br)}}.$
Then the following minimax lower bound holds:
\[
\liminf_{n\to\infty}
\inf_{l\in\mathcal A}
\sup_{\rho\in\mathcal P_{r,b}}
\frac{
\Big(
\mathbb E_{\rho^n}
\big[
\norm{f_{\zz}^l-\fp}_{\underline u,p}^{\pp}
\big]
\Big)^{1/\pp}
}{a_n}
> 0.
 \]
Consequently, the sequence $(a_n)_{n\ge1}$ is a minimax lower rate of convergence in $L^\pp$ over the model class $\mathcal P_{r,b}$.
\end{theorem}

\begin{proof}[Proof of Theorem~\ref{err.lower.bound.p.para.corrected}]

The argument follows by adapting the general lower bound established in 
Theorem~\ref{err.lower.bound} to the present parametric setting of the packing exponent.

\medskip
\noindent

Assume that the packing exponent satisfies $\ell_\epsilon = \Gamma\,\epsilon^{-pb(1-r)}>16$ for some $\Gamma$. Balancing the two exponential terms in the bound of 
Theorem~\ref{err.lower.bound} with $t=\frac{2r-1}{r-1}$ yields the equality
\begin{equation}\label{eq:balance}
    \frac{\Gamma}{48}\, \epsilon^{-pb(1-r)}
    = c\, n\, \epsilon^{\frac{p}{2r-pr+p-1}}\,\paren{\Gamma\,\epsilon^{-pb(1-r)}}^{\frac{(p-2)(1-r)}{2r-pr+p-1}}.
\end{equation}

Solving~\eqref{eq:balance} for $\epsilon$ gives the critical choice
\begin{equation}\label{eq:eps.n}
\epsilon=\tau a_n,
\qquad
a_n=n^{-\frac{2r-pr+p-1}{p(1+b-br)}},
    \qquad
    \tau = \left(\frac{1}{48c}\right)^{\!\frac{(2r-pr+p-1)}{p(1+b-br)}}\Gamma^{\frac{1}{p(1+b-br)}}.
\end{equation}

\medskip
\noindent

For this choice of $\epsilon$, Theorem~\ref{err.lower.bound} guarantees that for any estimator $f_{\zz}^l$ produced by an arbitrary learning algorithm
$l\in\mathcal{A}$, there exists $\rho_* \in \mathcal{P}_{r,b}$ such that
\begin{equation}\label{eq:prob.lower.balance}
    \mathbb P_{\rho^n_*}\!\left\{
        \| f_{\zz}^l - f_{\rho_*} \|_{\underline{u},p}
        > \tau a_n
    \right\}
    \ge
    \min\!\left\{
        \frac{1}{2},
        \;
        \vartheta\, 
    \right\}=\vartheta.
\end{equation}



Let
$X=
\norm{f_{\zz}^l-f_{\rho_*}}_{\underline u,p}.$ For every $\epsilon>0$,
$\mathbb E_{\rho_*^n}[X^\pp]
\ge
\epsilon^\pp
\mathbb P_{\rho_*^n}(X \ge \epsilon).$
Hence
\[
\Big(
\mathbb E_{\rho_*^n}[X^\pp]
\Big)^{1/\pp}
\ge
\epsilon \,
\Big(\mathbb P_{\rho_*^n}(X\ge \epsilon)\Big)^{1/\pp}.
\]

For $\epsilon = \tau a_n$, it follows from \eqref{eq:prob.lower.balance} that
\[
\begin{gathered}
    \forall l \in \mathcal{A},\ \exists \rho_* \in \mathcal{P}_{r,b} \ \text{s.t.} \ \Big(\mathbb E_{\rho_*^n}[X^\pp] \Big)^{1/\pp} \geq \tau a_n \vartheta^{1/\pp} \\
    \Rightarrow \forall l \in \mathcal{A}, \quad \sup_{\rho \in \mathcal{P}_{r,b}} \Big(\mathbb E_{\rho^n}[X^\pp] \Big)^{1/\pp} \geq \tau a_n \vartheta^{1/\pp}
     \\
    \Rightarrow \inf_{l \in \mathcal{A}}  \sup_{\rho \in \mathcal{P}_{r,b}} \Big(\mathbb E_{\rho^n}[X^\pp] \Big)^{1/\pp} \geq \tau a_n \vartheta^{1/\pp}
    \end{gathered}
\]
Finally, we obtain
\[
\liminf_{n\to\infty}
\inf_{l\in\mathcal A}
\sup_{\rho\in\mathcal P_{r,b}}
\frac{
\Big(
\mathbb E_{\rho^n}[X^\pp]
\Big)^{1/\pp}
}{a_n}
\ge
\tau \vartheta^{1/\pp}
> 0.
\]



\medskip
\noindent

This establishes the lower rate of convergence claimed in
the theorem.
 \end{proof}

\begin{remark}[On the verification of the hypothesis of Theorem \ref{err.lower.bound}]
We verify that the following condition, appearing in the statement of Theorem \ref{err.lower.bound},
\begin{equation}
\label{eq:bounds.thm6.3}
\begin{gathered}
\exists c_0,\epsilon_0 > 0, q>1 \ \text{such that} \ \forall \zeta \in (0,c_0), \forall \epsilon \in (0,\epsilon_0), \exists M=\{m_1, \dots, m_{\ell_\epsilon}\}: \\
\zeta \Bigl(\epsilon^p\ell_\epsilon^{-1}\Bigr)^{\!\frac{t}{2(p-t)}}
\;\le\; w_m
\;\le\; q\,\zeta \Bigl(\epsilon^p\ell_\epsilon^{-1}\Bigr)^{\!\frac{t}{2(p-t)}},
\qquad \forall m\in M
\end{gathered}
\end{equation}
is satisfied for the choice $\ell_\epsilon = \Gamma \epsilon^{-pb(1-r)}$ used in Theorem \ref{err.lower.bound.p.para.corrected}, under mild assumptions on the weight sequence $\underline{w}$.
\\
In particular, we assume that $w_m = m^{-a}$ for some $a>0$, which is asymptotically equivalent to $2^{-a \lfloor\log m\rfloor}$, as considered in Section \ref{ssec:smoothing}. For simplicity, set $\Gamma = 1$ and define
\(
\gamma = (p + pb(1-r))\frac{t}{2(p-t)} = \frac{pt(2-t+b)}{2(2-t)(p-t)}.
\)
Then \eqref{eq:bounds.thm6.3} is equivalent to finding $\epsilon^{-pb(1-r)}$ integers satisfying
\(
\zeta \epsilon^\gamma \le m^{-a} \le q\,\zeta \epsilon^\gamma,
\)
that is, integers belonging to the interval
\(
\big[(q\zeta\,\epsilon^\gamma)^{-1/a},\, (\zeta\,\epsilon^\gamma)^{-1/a}\big].
\)
The length of this interval is $\zeta^{-1/a}(1-q^{-1/a})\,\epsilon^{-\gamma/a}$, where $Q:=1-q^{-1/a}>0$. Moreover, since $\zeta<c_0$, this length is bounded from below by $Q c_0^{-1/a}\epsilon^{-\gamma/a}$.
\\
To ensure that this interval contains at least $\epsilon^{-pb(1-r)}$ integers, it suffices that
\(
\epsilon^{-\gamma/a} \gtrsim \epsilon^{-pb(1-r)}.
\)
Since $\epsilon<\epsilon_0<1$, this is equivalent to
\(
\frac{\gamma}{a} \ge pb(1-r),
\quad \text{i.e.} \quad
a \le \frac{\gamma}{pb(1-r)}.
\)
This shows that polynomially decaying weights $w_m=m^{-a}$ are compatible with the assumptions of Theorem \ref{err.lower.bound}.
\end{remark}

\medskip
Therefore, no learning algorithm can achieve a convergence rate faster than 
\(\mathcal{O}\!\big(n^{-\frac{(2r-pr+p-1)}{p(1+b-br)}}\big)\)
in the interpolation norm.  
Hence, the optimal lower rate of convergence for the class of nonlinear inverse problems considered here is of order 
\(\mathcal{O}\!\big(n^{-\frac{(2r-pr+p-1)}{p(1+b-br)}}\big)\), 
which matches the corresponding upper rate obtained under the same variational source and capacity conditions.

\subsection{Optimality of Convergence Rates and Connections}
\label{sec:optimality}

The goal of this subsection is to show that the upper and lower convergence rates established in Sections~\ref{Sec:convergence.rates} and~\ref{ssec:lower.rates} are optimal. Although this follows directly from Corollary~\ref{thm:uniform.rate.lambda*} and Theorem~\ref{err.lower.bound.p.para.corrected}, it is important to emphasize that these two results rely on different structural assumptions. In particular, the upper rates in Corollary~\ref{thm:uniform.rate.lambda*} are expressed in terms of the \emph{variational source condition} (Assumption~\ref{Ass:var.sour}) with index function $\phi(t)=t^{r}$. In contrast, the proof of the lower bound in Theorem~\ref{err.lower.bound.p.para.corrected} is based on the assumption that $\fp$ belongs to the \emph{approximation space} $k_t$. The first goal of this subsection is therefore to prove Theorem~\ref{kt.vsc}, which establishes the equivalence of these two requirements (under Assumption~\ref{Ass:A-Lip}) and which was already used in the proof of Theorem~\ref{err.lower.bound.p.para.corrected}. We then derive the optimality of the convergence rates. Finally, we provide an approximation-theoretic characterization of $k_t$ in terms of the polynomial decay rate of the best $n$-term approximation error, thereby establishing a complete connection between the sparsity (or compressibility) of the true solution $\fp$ and the optimal minimax rate of convergence.


\paragraph{From Approximation Spaces to Variational Source Conditions.}
We begin by establishing that $\fp\in k_t$ (together with Assumption~\ref{Ass:A-Lip}) implies the variational source condition with a suitable index $r$. 
The key intermediate step is the following equivalence, which shows that membership in $k_t$ can be reformulated as a nonlinear variational inequality involving the $\ell^1$ norm and the weighted $\ell^2_{\underline{w}}$ norm.

\begin{lemma}\label{kt.character}
Assume that $\fp \in \DD(A)\cap \Ho$, $t \in (0,1)$, and let $(w_j)_{j\ge 1}$ be a
sequence of positive weights satisfying $w_j \to 0$ as $j\to\infty$.
Then the following two statements are equivalent:
\begin{enumerate}[(i)]
    \item $\fp \in k_t$ (with respect to the weights $(w_j)_{j\ge 1}$).
    \item There exists a constant $\gamma > 0$ such that, for all
    $f \in \DD(A)\cap \Ho$,
    \begin{equation}\label{eq:kt-ineq}
        \norm{\fp - f}_{\Ho} + \norm{\fp}_{\Ho} - \norm{f}_{\Ho}
        \leq \gamma\, \norm{\fp - f}_{\underline{w},2}^{\frac{2-2t}{2-t}}.
    \end{equation}
\end{enumerate}
\end{lemma}

\noindent
Lemma~\ref{kt.character} is a variant of Lemma~13 in~\cite{Miller21}, adapted to the $\ell^1$-space. It shows that the smoothness of $\fp$ in the approximation space $k_t$ can be equivalently expressed through a power-type inequality between the Hilbert norm and a weighted $\ell^2$-type norm. This equivalence provides a bridge between geometric smoothness assumptions and variational inequalities used in convergence rate analysis. Combining the above characterization with the lower estimate in Assumption \eqref{Ass:A-Lip}, we immediately obtain a variational source condition. 

\begin{theorem}[Variational source condition from $k_t$]\label{kt.vsc}
Suppose that $A$ satisfies Assumption~\ref{Ass:A-Lip}.
Let $t \in (0,1)$ and $\fp \in D(A)\cap \Ho$ be such that
$\|\fp\|_{k_t} \le \varrho$, where $k_t$ is the approximation space defined
in~\eqref{eq:k_t.space} with weights $(w_j)_{j\ge 1}$ satisfying $w_j \to 0$.
Then, for all $f \in D(A) \cap \Ho$,
\begin{equation}\label{eq:explicit-varsource}
\|\fp - f\|_{\Ho} + \|\fp\|_{\Ho} - \|f\|_{\Ho}
\le C\,\|A(f)-A(\fp)\|_{\HH_\mu}^{\frac{2-2t}{2-t}},
\end{equation}
where $C>0$ is a constant depending only on $\varrho$, $t$, and the Lipschitz
constant $L$ in Assumption~\ref{Ass:A-Lip}.
In particular, $\fp$ satisfies the variational source condition
(Assumption~\ref{Ass:var.sour}) with index function
\begin{equation}\label{eq:phi-from-kt}
\phi(s) = C\,s^{\,r}, \qquad r := \frac{1-t}{2-t} \in \left(0,\tfrac{1}{2}\right).
\end{equation}
\end{theorem}

\begin{proof}
By Lemma~\ref{kt.character}, the condition $\fp\in k_t$ (with $\|\fp\|_{k_t}\le\varrho$) implies
\[
\|\fp - f\|_{\Ho} + \|\fp\|_{\Ho} - \|f\|_{\Ho}
\le \gamma\,\|\fp - f\|_{\underline{w},2}^{\frac{2-2t}{2-t}}
\]
for all $f\in\DD(A)\cap\Ho$, where $\gamma=\gamma(\varrho,t)$.
By part (ii) of Assumption~\ref{Ass:A-Lip} (lower Lipschitz bound),
\[
\|\fp - f\|_{\underline{w},2}
\le L\,\|A(\fp)-A(f)\|_{\Hu}.
\]
Substituting and setting $C:=\gamma L^{\frac{2-2t}{2-t}}$ gives~\eqref{eq:explicit-varsource}. The exponent in~\eqref{eq:phi-from-kt} follows directly by comparing $\phi(s)=Cs^r$ with~\eqref{eq:explicit-varsource}.  Since $t\in(0,1)$, one checks that $r=(1-t)/(2-t)$ maps $(0,1)$ bijectively onto $(0,1/2)$, confirming $r\in(0,1/2)\subset(0,1)$ as required by Assumption~\ref{Ass:var.sour}.
\end{proof}

\noindent

Theorem~\ref{kt.vsc} provides the fundamental link between the smoothness-space framework and the variational source condition framework. Specifically, under Assumption~\ref{Ass:A-Lip}(ii), if $\fp$ belongs to the smoothness space $k_t$, then it satisfies the variational source condition  of Assumption~\ref{Ass:var.sour} with index function $\phi(s)\asymp s^r$, where
\(r=\frac{1-t}{2-t}.\)
Conversely, again under Assumption~\ref{Ass:A-Lip}(i), if $\fp$ satisfies the variational source condition of Assumption~\ref{Ass:var.sour} with $\phi(s)=s^r$, then $\fp$ belongs to the smoothness space $k_t$ with
\(t=\frac{2r-1}{r-1}.\)
This means that the prior class $\mathcal{P}_{r,b}$
is indeed compatible with the approximation-space hypothesis used in the lower bound construction.


\paragraph{Matching Upper and Lower Rates.}
We now state explicitly that the upper rate of Corollary~\ref{thm:uniform.rate.lambda*}
and the lower rate of Theorem~\ref{err.lower.bound.p.para.corrected} coincide.

\begin{corollary}[Minimax optimality]\label{cor:minimax-optimality}
Let $r \in (0,\frac{1}{2})$ and $b \in (0,1)$.
Under Assumptions~\ref{ass:fp}--\ref{Ass:A-Lip} and the polynomial spectral
decay condition (Assumption~\ref{ass:polydecay}), the $\ell^1$-regularized
estimator~\eqref{eq:l1reg} with regularization parameter
\[
\lambda_* = n^{-\frac{1-r}{1+b-br}}
\]
achieves the rate
\[
\|\fz - \fp\|_{\underline{u},p}
= \mathcal{O}\!\left(n^{-\frac{2r-pr+p-1}{p(1+b-br)}}\right)
\]
in probability over the prior class $\mathcal{P}_{r,b}$, and no estimator can
achieve a faster rate over this class.  Hence, the rate
\[
n^{-\frac{2r-pr+p-1}{p(1+b-br)}}
\]
is minimax optimal over $\mathcal{P}_{r,b}$.
\end{corollary}



\paragraph{Approximation-theoretic Characterization of $k_t$.}
To make the connection with concrete sparsity models fully explicit, we recall
and prove the characterization of $k_t$ in terms of best $n$-term approximation
errors.  Given a sequence $f=(f_j)_{j=1}^\infty$ and weights $(w_j)$, define
the \emph{weighted sparsity measure}
\[
S(h):=\sum_{j=1}^{\infty} w_j^{-2}\, \mathbbm{1}_{\{h_j \neq 0\}}
\]
and the \emph{best $n$-term approximation error}
\[
\sigma_n(f):=\inf\Big\{ \|f-h\|_{\underline{w},2} : S(h)\leq n \Big\}.
\]
When $w_j\equiv 1$, $S(h)$ equals the number of nonzero coefficients of $h$,
recovering the standard compressed-sensing notion of sparsity.

\begin{lemma}\label{lem:sigma.kt}
Let $t \in (0,2)$. Then $f \in k_t$ if and only if the approximation error satisfies
\[
\sigma_n(f) = \mathcal{O}\!\big(n^{\frac{1}{2}-\frac{1}{t}}\big).
\]
\end{lemma}

\begin{proof}
Let us introduce the following notation:
\[
A(\alpha): = \sum_{j=1}^\infty w_j^{-2} \mathbbm{1}_{\{|f_j|w_j^{2} > \alpha\}} =\sum_{j: |f_j|w_j^2 >\alpha} w_j^{-2}.
\]
Suppose $f \in k_t$. Then, the membership condition gives $A(\alpha)\le C\alpha^{-t}$ for some $C>0$ and for every $\alpha>0$. Let us define $\alpha_n : = \operatorname{sup}\{\alpha>0: A(\alpha)\leq n\}$ and the truncation $h^{(n)}$ such that
\[
h^{(n)}_j = \left\{
\begin{aligned}
f_j, \quad &\text{if }\ |f_j|w_j^2 > \alpha_n \\
0, \quad &\text{otherwise.}
\end{aligned}
\right.
\]
Then, 
\[
S(h^{(n)}) = \sum_{j=1}^\infty w_j^{-2} \mathbbm{1}_{\{h_j^{(n)} \neq 0\}} = \sum_{j:|f_j|w_j^2 >\alpha_n}w_j^{-2} = A(\alpha_n) \leq n,
\]
we have that $\sigma_n(f) \leq \| f-h^{(n)}\|_{\underline{w},2}$. We easily observe that
\[
\| f-h^{(n)}\|_{\underline{w},2}^2 = \sum_{j: |f_j|w_j^2\leq \alpha_n} w_j^2 f_j^2\ .
\]
To bound the last term, we use a layer estimate, relying on the following trick:
\[
\begin{aligned}
w_j^2 f_j^2 &= w_j^{-2} (|f_j|w_j^2)^2 = 
w_j^{-2} \int_0^{|f_j|w_j^2}  2\beta \operatorname{d}\!\beta  = 2w_j^{-2} \int_\RR \beta \mathbbm{1}_{\{\beta < |f_j|w_j^2\}}(\beta)\operatorname{d}\!\beta.
\end{aligned}
\]
As a consequence,
\[
\begin{aligned}
    \sum_{j: |f_j|w_j^2\leq \alpha_n} w_j^2 f_j^2 &=
    2 \sum_{j: |f_j|w_j^2\leq \alpha_n} w_j^{-2} \int_\RR \beta \mathbbm{1}_{\{\beta < |f_j|w_j^2\}}\operatorname{d}\!\beta \\
    &\leq 2 \sum_{j=1}^\infty w_j^{-2} \int_0^{\alpha_n} \beta \mathbbm{1}_{\{\beta < |f_j|w_j^2\}}\operatorname{d}\!\beta \\
    & = 2 \int_{0}^{\alpha_n} \beta A(\beta) \operatorname{d}\!\beta.
\end{aligned}
\]

Using $A(\beta) \le C \beta^{-t}$,
\[
\| f-h^{(n)}\|_{\underline{w},2}^2
\le 2C \int_0^{\alpha_n} \beta^{1-t} d\beta.
\]

Since $t<2$, this yields
\[
\| f-h^{(n)}\|_{\underline{w},2}^2
\le 2C \alpha_n^{2-t}.
\]

By the definition of $\alpha_n$ and using the bound $A(\alpha) \leq C \alpha^{-t}$, we obtain
\[
n \leq C \alpha_n^{-t},
\]
which implies
\[
\alpha_n^{2-t} \leq C^{\frac{2-t}{t}} n^{-\frac{2-t}{t}}.
\]
Therefore,
\[
\sigma_n(f) = \mathcal{O}(n^{1-\frac{t}{2}}).
\]

\smallskip

Conversely, assume that $\sigma_n(f) \le C n^{\frac{1}{2}-\frac{1}{t}}$ for all $n$.  
Fix $\alpha>0$ and let $n=\lfloor \alpha^{-t}\rfloor$.  
By definition of $\sigma_n(f)$, there exists $h$ with $S(h)\le n$ such that 
\(\|f-h\|_{\underline{w},2}\le C n^{\frac{1}{2}-\frac{1}{t}} = C \alpha^{1-\frac{t}{2}}\).
Then,
since $|f_j|w_j^2 > \alpha$ implies $w_j^{-2} < w_j^2 \alpha^{-2}f_j^2$, we have
\[
\begin{aligned}
    \sum\limits_{j=1}^\infty w_j^{-2} \mathbbm{1}_{\{|f_j|w_j^2 > \alpha\}} & = 
    \sum_{j : h_j\neq 0} w_j^{-2} \mathbbm{1}_{\{|f_j|w_j^2 > \alpha\}} + 
    \sum_{j : h_j= 0} w_j^{-2} \mathbbm{1}_{\{|f_j|w_j^2 > \alpha\}} 
    \\
    &\leq
    S(h) +
    \alpha^{-2} \sum_{j : h_j =0} w_j^2 (f_j-h_j)^2  \\
    &\leq S(h) + \alpha^{-2} \| f-h\|_{\underline{w},2}^2 \leq (1+C^2) \alpha^{-t},
\end{aligned}
\]
which implies
\[
\|f\|_{k_t}
= \sup_{\alpha>0}
\alpha \Big(
\sum_{j=1}^{\infty} w_j^{-2} \, 
\mathbbm{1}_{\{w_j^{-2} \alpha < |f_j|\}}
\Big)^{1/t}
\le (1 + C^2)^{1/t} < \infty,
\]
hence $f \in k_t$.
\end{proof}

\noindent
Lemma~\ref{lem:sigma.kt} shows that the rate $n^{1/2-1/t}$ for best $n$-term
approximation is the right measure of sparsity underlying the convergence rates
in Theorems~\ref{thm:uniform.rate.lambda*}
and~\ref{err.lower.bound.p.para.corrected}.  Taken together,
Theorem~\ref{kt.vsc} and Lemma~\ref{lem:sigma.kt} establish a complete chain of
implications:

\begin{equation}\label{eq:chain}
\begin{aligned}
&\sigma_n(\fp)=\mathcal{O}\!\big(n^{\frac{1}{2}-\frac{1}{t}}\big)
\;\Longleftrightarrow\;
\fp\in k_t \quad \text{(under Assumption~\ref{Ass:A-Lip})}
 \\
\;\Longrightarrow\;&
\text{VSC with } \phi(s)=s^{\frac{1-t}{2-t}}
\;\Longrightarrow\;
\text{convergence rate } n^{-\frac{2r-pr+p-1}{p(1+b-br)}}.
\end{aligned}
\end{equation}
where $r=(1-t)/(2-t)$ and the last implication is provided by
Theorem~\ref{thm:uniform.rate.lambda*}.  The corresponding lower bound
(Theorem~\ref{err.lower.bound.p.para.corrected}) shows that no estimator can
improve upon this rate, so the $\ell^1$-regularized estimator is minimax
optimal over the class $\mathcal{P}_{r,b}$.

\begin{remark}
The mapping $t\mapsto r=(1-t)/(2-t)$ is strictly increasing on $(0,1)$
with range $(0,1/2)$.  Thus, for the present $\ell^1$ setting, the variational
source condition index $r$ is confined to the interval $(0,1/2)$.
This is consistent with the known limitations of $\ell^1$ regularization
compared to $\ell^2$ regularization: the $\ell^1$ estimator extracts sparsity
information efficiently but cannot exploit smoothness beyond a certain degree.
The boundary case $t\to 0$ ($r\to 0$) corresponds to very sparse but
almost unstructured solutions, while $t\to 1$ ($r\to 1/2$) corresponds to
solutions with the maximal sparsity exploitable within this framework (e.g.,
functions with moderately fast coefficient decay in a wavelet basis).
\end{remark}

\begin{remark}
The convergence rate~\eqref{eq:chain} unifies three distinct complexity
parameters: the sparsity exponent $t$ (or equivalently the source smoothness $r$
from the variational source condition), the spectral decay exponent $b$ of the
covariance operator $T_\mu$, and the interpolation norm index $p$.  In the
special case $p=1$ (the $\ell^1$ reconstruction norm), the exponent reduces to
\[
\frac{2r - r + 1 - 1}{1\cdot(1+b-br)} = \frac{r}{1+b-br},
\]
recovering the rate obtained in Corollary~\ref{cor:err.upper.power}.
In the case $p=2$ (the weighted $\ell^2_{\underline{w}}$ norm), one obtains
\[
\frac{2r - 2r + 2 - 1}{2(1+b-br)} = \frac{1}{2(1+b-br)}.
\]
\end{remark}

\section{Motivation from Examples}
\label{sec:examples}

In this section, we discuss two main examples that satisfy all the hypotheses introduced in the theoretical discussion, particularly focusing on Assumptions \ref{ass:kernel}-\ref{ass:polydecay}. We first consider a general strategy to verify Assumptions \ref{ass:Lipschitz}, \ref{Ass:var.sour}, and \ref{Ass:A-Lip} for operators mapping on smoothness spaces (a generalization of Besov spaces) with suitable smoothing properties. We then present the two main examples falling in this category, namely the identification of a conductivity coefficient in a PDE and the inversion of the (filtered) Radon transform, carefully checking the remaining assumptions. Finally, we observe that all the hypothesis can be verified also when the forward is just the synthesis operator $S \colon \ell^1 \to \HH$. As a consequence of this simple example, we can connect our results to classical results of statistical learning outside the context of inverse problems.

\subsection{Finitely Smoothing Operators and Smoothness Spaces}
\label{ssec:smoothing}

We assume that the forward operator $A$ can be written as a composition,
\[A = G \circ S,\]
where $G\colon L^2(\XX)\rightarrow \HH$, $\XX$ is a bounded domain in $\RR^d$ or the $d$-dimensional torus $(\RR/\mathbb{Z})^d$, and $S \colon \ell^2 \rightarrow L^2(\XX)$ is the synthesis operator of a multi-resolution system, such as a wavelet basis or a shearlet frame. In particular, we consider a family of functions $(\phi_{M,k})_{(M,k) \in G \times \RR^d}$, where $G$ is a subset of $GL_d(\RR)$, the group of $d$-dimensional invertible matrices. The simplest case
is $G= \{2^j I : j \in \mathbb{Z} \}$, where $I$ is the identity matrix, corresponding to $G$ being and isotropic dilation group. This  provides a  discrete wavelet system. Instead, a discrete shearlet system is obtained when the matrix $M$ is the
product of the matrices $B_s A_j$, where $A_j$ is an anisotropic dilation matrix (whose scale is determined by $j$) and $B_s$ is a shear matrix. In both examples, the role of $k \in \RR^d$ is that of encoding translations.
For ease of notation we introduce the index
$\Lambda = \{(M,k) \; : \;  M \in G \subset GL_d(\RR), \ k \times \RR^d\}$, so that for an atom $\phi_\lambda = \phi_{M,k}$, we denote its scale as $|\lambda| = j$. We assume that $(\phi_\lambda)_{\lambda\in \Lambda}$ is a frame of $L^2(\XX)$, and define the associated synthesis operator $S$ as follows:
\[
(Sf)(x) = \sum_{\lambda \in \Lambda} f_\lambda \phi_\lambda(x).
\]
Such systems are usually employed to encode suitable smoothness of functions, providing an equivalent description of suitable smoothness spaces, depending on a smoothness level $s$ and on two integrability indices $p$ and $q$. For our purposes, we only consider the case $p=q$, $p\geq 1$, and, relying on the equivalence of atomic descriptions, we define the smoothness space $S_{p,p}^s$ associated with the frame $(\phi_\lambda)_{\lambda\in \Lambda}$ as the image, through the synthesis operator $S$, of the weighted sequence space $\ell_{p,\underline{w}}$, where the weight $\underline{w}$ depends on $s,d,p$, and on the frame itself. 
We consider two main examples:
\begin{enumerate}
    \item \textbf{Wavelet bases.} If $(\phi_\lambda)_{\lambda\in \Lambda}$ is a wavelet basis, we consider $w_\lambda = 2^{|\lambda|(s+\frac{d}{2}-\frac{d}{p})}$. In this case, the weighted space $\ell_{p,\underline{w}}$ is equivalent to the space $b_{p,p}^s$ defined in \cite{Miller21}. Moreover, it is well known (see, e.g., \cite[Chapter 7]{triebel1983theory}) that, provided that $p\geq 1$ and $|s| < s_{max}$ (depending on the regularity and vanishing moments of the wavelet $\phi$), the associated smoothness space is norm-isomorphic to the Besov space $B_{p,p}^s$. In particular, selecting $p=1$ and $s = \frac{d}{2}$, we also observe that the $\ell^1$ sequence space is isomorphic to $B_{1,1}^{d/2}$.
    \item \textbf{Shearlet frames.} If $(\phi_\lambda)_{\lambda\in \Lambda}$ is a cone-adapted shearlet smooth Parseval frame in 2D (see \cite{labate2013shearlet}), we consider $w_\lambda = 2^{|\lambda|(s+\frac{3}{2}-\frac{3}{p})}$. The resulting smoothness spaces $S_{p,p}^s$ is related to Besov spaces as follows (see \cite[Proposition 4.3]{labate2013shearlet}): $B_{p,p}^{s+\frac{1}{p}} \hookrightarrow  S_{p,p}^s \hookrightarrow B_{p,p}^{s+\frac{1}{p}-1}$.
In particular, selecting $p=1$ and $s =\frac{3}{2}$, the sequence space $\ell^1$ is isomorphic to $S_{1,1}^{\frac{3}{2}}$, which is not equivalent to a Besov space, but it satisfies $B_{1,1}^{\frac{5}{2}} \hookrightarrow S_{1,1}^{\frac{3}{2}} \hookrightarrow B_{1,1}^{\frac{3}{2}}$
\end{enumerate}

As in \cite{Miller21}, we require the operator $G$ to satisfy the following requirement: there exists $a>0$ such that the domain $D_G$ of $G$ is a subset of $H^{-a}(\XX)$ and, for some constant $L\geq 1$, 
\begin{equation}
\frac{1}{L}\|u_1-u_2\|_{H^{-a}} \leq \| G(u_1) - G(u_2)\|_{\Hu}\leq L \|u_1-u_2\|_{H^{-a}} \qquad \forall u_1, u_2 \in D_G.
    \label{eq:equiv_Besov}
\end{equation}
Since $\HH_\mu = L^2(\XX,\mu;\YY)$, condition \eqref{eq:equiv_Besov} can be interpreted as a finitely-smoothing property of $G$. 

We easily observe that \eqref{eq:equiv_Besov} implies that Assumption \ref{Ass:A-Lip}(ii) is verified. Indeed,
\[
\| A(f)-A(\tilde{f}) \|_{\Hu} = \| G(S(f))-G(S(\tilde{f})) \|_{\Hu}
\leq \| S(f)-S(\tilde{f}) \|_{H^{-a}}.
\]
Since $H^{-a}(\XX)=B_{2,2}^{-a}(\XX)$, it is enough to find a $s$ such that $S_{2,2}^s \hookrightarrow B_{2,2}^{-a}$ to conclude that
\[
\| A(f)-A(\tilde{f}) \|_{\Hu}
\leq \| S(f)-S(\tilde{f}) \|_{S^s_{2,2}} = \| f - \tilde{f}\|_{\ell_{2,\underline{w}}}.
\]
In the wavelet case, since $S_{2,2}^{-a}\cong H^{-a}$, we verify Assumption \ref{Ass:A-Lip}(ii) with $w_\lambda = 2^{-|\lambda|a}$. 
In the shearlet case, we use $S_{2,2}^{-a+\frac{1}{2}}\hookrightarrow H^{-a}$, thus and Assumption \ref{Ass:A-Lip}(ii) is verified with $w_\lambda = 2^{-|\lambda|(a-\frac{1}{2})}$.

Analogously, \eqref{eq:equiv_Besov} allows us to verify Assumption \ref{Ass:A-Lip}(i), by considering $s$ s.t. $H_2^{-a} \hookrightarrow B_{2,2}^s$:
\[
\| A(f) - A(\tilde{f})\|_{\Hu} \geq \| S(f) - S(\tilde{f})\|_{H^{-a}} \geq \| S(f)-S(\tilde{f})\|_{S_{2,2}^{s}} = \| f-\tilde{f}\|_{\ell_{2,\underline{w}}}.
\]
In the wavelet case, Assumption \ref{Ass:A-Lip}(i) is verified with the same weights as Assumption \ref{Ass:A-Lip}(ii). In the shearlet case, using $H^{-a} \hookrightarrow S_{2,2}^{-a-\frac{1}{2}}$, thus and Assumption \ref{Ass:A-Lip}(ii) is verified with $w_\lambda = 2^{-|\lambda|(a+\frac{1}{2})}$.

\begin{remark}
Recall that $\Hu = L^2(\XX,\mu;\YY)$ and that $\XX$ is a bounded subset in $\RR^d$ (or the $d$-dimensional torus). If $\mu$ is absolutely continuous  with respect to $d-$dimensional the Lebesgue measure and its density is bounded from below and above on its support $\XX'$, then it is equivalent to verify \eqref{eq:equiv_Besov} on $\Hu$ or on $L^2(\XX';\YY)$. A prominent example of this scenario is the uniform probability on $\XX$.   
\end{remark}

Finally, we recall a strategy (derived from \cite{Miller21} and outlined in Section \ref{sec:optimality}) through which it is possible to verify Assumption \ref{Ass:var.sour} for operators satisfying \eqref{eq:equiv_Besov}.
Combining Lemma~\ref{lem:sigma.kt} and Theorem \ref{kt.vsc}, we can ensure that, if a sequence $f = (f_\lambda)_{\lambda \in \Lambda}$ shows a polynomial decay of the best $n$–term approximation error with respect to $\| \cdot \|_{\underline{w},2}$
of order $\frac{1}{2}-\frac{1}{t}$ (with $0<t<1$), then $f$ also satisfies the variational source condition
with index function $\phi(s) = s^{\frac{1-t}{2-t}}$. Equivalently, if the decay of the approximation error is $\sigma_n(\fr) = \mathcal{O}(n^{-\beta})$ (with $\beta > \frac{1}{2}$), then $t=\frac{2}{2\beta+1}$ and $\phi(s)=s^{\frac{1}{2}-\frac{1}{4\beta}}$. \\
Ultimately, to satisfy Assumption \ref{Ass:var.sour}, it is sufficient to have \textit{a priori} guarantees on the approximation error decay of the ground truth $\fp$. 

Using classical results from nonlinear approximation of signals and images, we can derive some insightful examples (all referred to the case $\underline{w} \equiv 1$):
\begin{itemize}
    \item In dimension $d=1$, if $S$ is a wavelet synthesis operator associated to a $C^q$ wavelet with $q$ vanishing moments ($q>1$) and the unknown signal $S\fr$ is piecewise $C^1$ with a finite number of discontinuities, we know that $\sigma_n(\fr) = \mathcal{O}(n^{-1})$ up to a logarithmic factor (see, e.g., \cite[Theorem 9.12]{mallat1999wavelet}).
    \item In dimension $d=2$, if $S$ is a wavelet synthesis operator as above and the image $S\fr$ is of bounded variation (i.e., piecewise $W^{1,1}$ on regions of finite perimeter), then $\sigma_n(\fr) = \mathcal{O}(n^{-\frac{1}{2}})$ up to a logarithmic factor (see \cite[Theorem 9.17]{mallat1999wavelet}, \cite{cohen1999nonlinear}). The order does not improve if we consider further local regularity of the image and of the regions.
    \item In dimension $d=2$, if $S$ is a shearlet synthesis and the image $S\fr$ is cartoon-like (i.e., piecewise $C^2$ on regions of $C^2$ boundary), then $\sigma_n(\fr) = \mathcal{O}(n^{-1})$ up to a logarithmic factor (see \cite{guo2007optimally}). The same holds if a curvelet dictionary is used (see \cite{candes2004new}).
\end{itemize}

In the next two subsections, we describe two possible applications in which the operator $A=G\circ S$ satisfies \eqref{eq:equiv_Besov} and image space $\HH$ verifies Assumptions \ref{ass:kernel} and \ref{ass:polydecay}. Both those cases are related with bounded linear operators from $\ell^2$ to $\HH$, which directly implies that Assumption \ref{ass:Lipschitz} is satisfied.

\subsection{Identification of a Reaction Coefficient}
\label{ssec:coefficient}
Let $\XX \subset \mathbb{R}^{d}$ with $d \in \{1,2,3\}$ be a bounded domain with $C^{2}$ boundary, and select $\varphi_1 \in L^2(\XX)$ and $\varphi_2 \in H^{\frac{3}{2}}(\partial \XX)$. Let $c \in L^\infty(\XX)$ s.t. $c\geq 0$ and consider the unique weak solution $h$ of 
\[
\begin{aligned}
-\Delta h +c h &= \varphi_1 \quad \text{in }\XX\\
    h &=\varphi_2 \quad \text{on }\partial \XX
\end{aligned}
\]
Let $G: L^\infty(\XX) \rightarrow H^1(\XX)$ be operator associating the coefficient $c$ to the solution $h$. We consider the inverse problem of recovering $c$ from the knowledge of $y$.

In \cite[Example 2]{Miller21}, it is argued that $G$ satisfies \eqref{eq:equiv_Besov} with $a=2$ in a sufficiently small $L^2$ ball of a reference solution $c_0$.

Moreover, standard elliptic regularity results (see \cite[Theorem 8.12]{gilbarg1998elliptic}) ensure that $h \in H^2(\XX)$, and Sobolev embeddings guarantee that $H^2(\XX) \subset C(\XX)$ for $d \leq 3$. As a consequence, we let $\HH = \{ h \in H^2(\XX): h|_{\partial\XX}=\varphi_2\}$ is a (real-valued) reproducing kernel Hilbert space. We consider for simplicity the homogenous case, i.e., $\varphi_2 = 0$, where $\HH = H^2(\XX) \cap H^1_0(\XX)$. In this case, it is easy to show that the kernel of $\HH$ is 
\[
K(x,y) = \sum_{k=1}^\infty \frac{\phi_k(x)\phi_k(y)}{(1+\lambda_k)^2},
\]
where $(\lambda_k, \phi_k)$ are the eigenvalues and eigenfunctions of the Laplace operator on $\XX$ with homogenous Dirichlet boundary conditions. 
 
In order to verify Assumption \ref{ass:kernel}, we need to show that $\| K_x\|_{HS}$ is bounded uniformly in $x$. To do so, first notice that, for each $x$, 
\[
\| K_x\|_{HS}^2 = \sum_{k=1}^{\infty} \left|K_x\left(\frac{\phi_k}{1+\lambda_k}\right)\right|^2 = \sum_{k=1}^{\infty} \frac{\phi_k^2(x)}{(1+\lambda_k)^2} =  K(x,x),
\]
where we have used that $\left(\frac{\phi_k}{1+\lambda_k}\right)_k$ form an orthonormal basis of $\HH$. To get a more explicit expression of the eigenpairs, consider for simplicity the case $\XX = [0,1]^d$. In this case, for each multi-index $\kk=(k_1,\ldots,k_d)$,
\[
\phi_\kk(x) = \prod_{i=1}^d \sqrt{2}\sin(k_i \pi x_i), \qquad \lambda_\kk = \pi^2 |\kk|^2 = \pi^2 (k_1^2 +\ldots +k_d^2).
\]
Now, since $|\phi_\kk(x)|^2\leq 2^d$, we have
\[
\| K_x\|_{HS}^2 =  K(x,x)\lesssim \sum_{\kk \in \NN^d} |\kk|^{-4},
\]
which is a convergent series if $d<4$.

We now aim at verifying Assumption \ref{ass:polydecay}. 
Consider for simplicity a uniform distribution $\mu$ on $\XX=[0,1]^d$: in this case, the covariance operator is given by
\[
(Tf)(y) = \int_\XX f(x) K(x,y)dx = \sum_{k=1}^\infty \left(\int_\XX f(x)\phi_k(x)dx\right) \frac{\phi_k(y)}{(1+\lambda_k)^2}.
\]
As a consequence, $T$ is diagonal with respect to the basis $(\phi_k)_k$, with eigenvalues 
\[
\mu_\kk = \frac{1}{(1+\lambda_\kk)^2} \sim \lambda_\kk^{-2} \sim |\kk|^{-4}
\]
Let us now order the eigenvalues decreasingly with a single index $k$: since the number of indices satisfying $|\kk|\leq M$ is of order $M^d$, we get that 
\[
\mu_\kk \sim |\kk|^{-4} \quad \Leftrightarrow \quad \mu_k \sim k^{-\frac{4}{d}}.
\]
As a result,
\[
    \mathcal{N}(\lambda) = \operatorname{tr}((T+\lambda I)^{-1} T) = \sum_{k=1}^\infty \frac{\mu_k}{\mu_k + \lambda},
\]
and it is easy to show that $\mu_k \sim k^{-\frac{1}{b}}$ implies $\mathcal{N}(\lambda) \sim \lambda^{-b}$, whenever $0<b<1$. Indeed, for $k$ greater than a sufficiently large $K_0$,
    \[
    \mathcal{N}(\lambda) \sim \int_{K_0}^\infty\frac{Cx^{-\frac{1}{b}}}{Cx^{-\frac{1}{b}} + \lambda}dx  \leq C^{b}\lambda^{-b} \int_{0}^\infty\frac{v^{-\frac{1}{b}}}{ v^{-\frac{1}{b}} + 1} dv,
    \]
    where we employed the change of variable $x= C^{{b}}\lambda^{-{b}} v$. The latter quantity is bounded since, as $v \rightarrow\infty$, the integrated function is asymptotic to $v^{-\frac{1}{b}}$, which is integrable provided that $\frac{1}{b}>1$.

As a result, since $\mu_k \sim \lambda_k^{-4/d}$, we have that $\mathcal{N}(\lambda) \sim \lambda^{-\frac{d}{4}}$, thus $b= \frac{d}{4}$ (which verifies $0<b<1$ for $d=1,2,3$).

\subsection{Filtered Radon Transform}

Let us consider an application in Computed Tomography: set $A = R \circ S$, where $R\colon L^2(\Omega)\rightarrow L^2(\XX;\YY)$, where $\Omega = [-1,1]^2$, $\mathcal{X} = \mathbb{S}^1 \cong [0,2\pi)$ and $\mathcal{Y} = L^2(-\sqrt2,\sqrt{2})$, and
$R$ is the Radon transform in 2D. In particular, for all $u \in L^2(\Omega)$, $Ru$ is a function in $L^2(\XX; \YY)$ that associates an angle $\theta \in \mathcal{X}$ to a function in $\mathcal{Y}$ as follows:
\[
(Ru)(\theta) = (Ru)(\theta,\cdot)\in \YY, \qquad (Ru)(\theta,s) = \int_\Omega \delta(x \cdot \omega_\theta - s) u(x) dx,
\]
where $\delta$ is the Dirac distribution and $\omega_\theta = (\cos(\theta),\sin(\theta))$. 

Notice that, thanks to \cite[Theorem 5.1]{natterer2001mathematics}, condition \eqref{eq:equiv_Besov} is verified for $a = \frac{1}{2}$

Let us now consider the image $\mathcal{H}$ of $L^2(\Omega)$ through $R$: by the injectivity of the Radon transform, if $h,h' \in \mathcal{H}$, there exist unique $u,u' \in L^2(\Omega)$ such that $h=Ru$ and $h'=Ru'$.
Then, the following scalar product and its induced norm are well-defined on $\mathcal{H}$:
\[
\langle h,h'\rangle_\mathcal{H} = \langle u,u' \rangle_{L^2(\Omega)}, \quad \| h\|_{\mathcal{H}} = \| u \|_{L^2(\Omega)}.
\]
\begin{proposition}
    The space $\mathcal{H}$ is a vector-valued RKHS with kernel
\[
K(\theta',\theta) = R_{\theta'}R_\theta^*,
\]
being $R_\theta$ the operator mapping a function $u \in L^2(\Omega)$ to $Ru(\theta,\cdot)$.
\end{proposition}

\begin{proof}
We first show that the (linear) operators $R_\theta$ are bounded form $L^2(\Omega)$ to $\mathcal{Y}$, uniformly in $\theta \in \mathcal{X}$:
\[
\begin{aligned}
    \| R_\theta u  \|_{\mathcal{Y}}^2 
    &= \int_{-\sqrt{2}}^{\sqrt{2}} |Ru(\theta,s)|^2ds 
    = \int_{-\sqrt{2}}^{\sqrt{2}} \left| \int_{L(\theta,s) \cap \Omega} u(x) dx \right|^2 ds \\
    &\leq \int_{-\sqrt{2}}^{\sqrt{2}}  |L(\theta,s) \cap \Omega| \left(\int_{\Omega} u(x)^2dx \right) ds \leq 8 \| u \|_{L^2(\Omega)}^2,
\end{aligned}
\]
where $L(\theta,s)$ denotes the line of direction $\theta$ and offset $s$ from the origin, thus $ |L(\theta,s) \cap [0,1]^2| \leq 2\sqrt{2}$.
Consider now any $h \in \mathcal{H}$ and let $h=Ru$ for $u \in L^2(\Omega)$: then, for all $\theta \in \mathcal{X}, y \in \mathcal{Y}$, the evaluation functional $F_{\theta,y}\colon \mathcal{H} \rightarrow \mathbb{R}$, defined as \[
F_{\theta,y}(h) = \langle h(\theta),y \rangle_{\mathcal{Y}} = \int_{-\sqrt{2}}^{\sqrt{2}} Ru(\theta,s)y(s) ds = \langle R_\theta u,y\rangle_\mathcal{Y},
\]
is bounded for any $\theta, y$:
\[
|F_{\theta,y}(h)| = |\langle R_\theta u ,y \rangle_\mathcal{Y}| \leq L \| u \|_{L^2(\Omega)} \| y \|_{\mathcal{Y}} \leq L \| y \|_\mathcal{Y} \| h \|_{\mathcal{H}}.
\]
This guarantees that $\mathcal{H}$ is an RKHS.
We now wish to provide an expression for its kernel. On the one hand, we have
\[
\langle h, K(\cdot,\theta)y\rangle_\mathcal{H} := F_{\theta,y}(h) = \langle R_\theta u,y\rangle_\mathcal{Y} = \langle u, R_\theta^* y \rangle_{L^2(\Omega)}.
\]
On the other hand, by the continuity of the evaluation operator, $K(\cdot,\theta)y \in \mathcal{H}$, and there exists $u' \in L^2(\Omega)$ such that $Ru' = K(\cdot,\theta)y$, which also implies that
\[
\langle h, K(\cdot,\theta)y\rangle_\mathcal{H} = \langle u, u'\rangle_{L^2(\Omega)}. 
\]
This allows to conclude that $u' = R_\theta^*y$, hence $K(\theta',\theta)y = (Ru')(\theta')=R_{\theta'}R_\theta^*y$, from which we finally deduce that $K(\theta',\theta)=R_{\theta'}R_{\theta}^*$.
\end{proof}

Although $\mathcal{H}$ possesses the desired RKHS structure, it does not fulfill the assumptions for the theoretical discussion, and in particular the following condition holds on $K_\theta = K(\cdot,\theta)$ as a map from $y \in \mathcal{Y}$ to $K(\cdot,\theta)y \in \mathcal{H}$. 
\begin{proposition}
For every $\theta \in \mathcal{X}$, the kernel $K_\theta \colon \mathcal{Y} \rightarrow \mathcal{H}$ is not a Hilbert-Schmidt operator.
\end{proposition}

\begin{proof}
    From $K(\theta',\theta)=R_{\theta'}R_{\theta}^*$ we deduce that $K(\cdot,\theta)y= R R_{\theta}^*y$. As a consequence, $\| K(\cdot,\theta)y \|_{\mathcal{H}} = \| R_\theta^* y \|_{L^2(\Omega)}$, hence $\| K(\cdot,\theta) \|_{\operatorname{HS}(\mathcal{Y};\mathcal{H})} = \| R_\theta^*\|_{\operatorname{HS}(\mathcal{Y};L^2(\Omega))}$.
    Notice now that
    \[
    R_\theta R_\theta^*y (s) = \int_{L(\theta,s)\cap \Omega} y(x\cdot \omega_\theta) dx = y(s) |L(\theta,s)\cap\Omega|,
    \]
    thus $R_\theta R_\theta^*$ is a multiplication operator with the (non-zero) multiplier $|L(\theta,s)\cap\Omega|$, hence it cannot be compact nor Hilbert-Schmidt.
\end{proof}

In order to verify Assumptions \ref{ass:kernel} and \ref{ass:polydecay}, we instead consider $A = R_\varphi\circ S$, where, for a fixed function $\varphi \in L^2(-\sqrt{2},\sqrt{2})$, $R_\varphi$ is a smoothed version of the Radon transform:
\[
R_\varphi u = \varphi \ast_s Ru = \int_{\mathbb{R}} \varphi(s-\tau) \left(\int_{L(\theta,\tau)} u(x)dx\right) \ d\tau = \int_\Omega u(x)\varphi(s-x\cdot\omega_\theta).
\]
We can also denote it by $R_\varphi = C_\varphi \circ R$, where $C_\varphi$ is the convolution operator with the filter $\varphi$ with respect to the $s$ variable. We consider two main possible choices: a boxcar function $\varphi(\tau)=\frac{1}{2a}\chi_{[-a,a]}(\tau)$ for $a>0$ and a Gaussian kernel. 

It is immediate to verify that the image $\mathcal{H}$ of $L^2(\Omega)$ through $R_\varphi$ is a RKHS with evaluation functional $F_{\theta,y}(h)=\langle C_\varphi R_\theta u,y\rangle_{\mathcal{Y}}$ and kernel $K(\theta',\theta) = C_\varphi R_{\theta'}R_\theta^*C_\varphi^*$.
It is also possible to express the action of $K(\theta',\theta)$ by means of a scalar kernel as follows:
\[
\begin{gathered}
    \big(K(\theta',\theta)y\big)(s') = \int_{S} k(s',s,\theta',\theta)y(s) ds,\\ k(s',s,\theta',\theta) = \int_\Omega \varphi(s-x\cdot\omega_\theta)\varphi(s'-x\cdot\omega_{\theta'})dx
    \end{gathered}
\]

\begin{proposition}
For every $\theta \in \mathcal{X}$, the kernel $K_\theta \colon \mathcal{Y} \rightarrow \mathcal{H}$ is a Hilbert-Schmidt operator and its $\operatorname{HS}$ norm is bounded uniformly in $\theta \in \mathcal{X}$.
\end{proposition}

\begin{proof}
 From $K(\theta',\theta) = C_\varphi R_{\theta'}R_\theta^*C_\varphi^*$ we deduce that $K_\theta y = C_\varphi R R_\theta^* C_\varphi^* y$, and in view of the isometric isomorphism induced by $R_\varphi = C_\varphi R$ from $L^2(\Omega)$ to $\mathcal{H}$ we can restrict ourselves to compute $\| R_\theta^* C_\varphi^* \|_{\operatorname{HS}(\mathcal{Y}; L^2(\Omega))}$.
 We immediately observe that
 \[
 (R_\theta^* C_\varphi^* y)(x) = (C_\varphi^*y)(x\cdot \omega_\theta) = \int_{S} \varphi(\tau - x\cdot\omega_\theta)y(\tau)d\tau \qquad \forall x\in \Omega,
 \]
 thus $R_\theta^* C_\varphi^*$ is an integral operator with kernel $\kappa_\theta(x,\tau) = \varphi(\tau - x\cdot\omega_\theta)$. By the isomorphism theorem for Hilbert-Schmidt operators,
 \[
 \|R_\theta^* C_\varphi^*\|_{\operatorname{HS}(\YY; L^2(\Omega))}^2 = \| \kappa_\theta \|_{L^2([-\sqrt{2},\sqrt{2}] \times \Omega)}^2 = \int_{-\sqrt{2}}^{\sqrt{2}} \int_\Omega |\varphi(\tau - x\cdot\omega_\theta)|^2 dx \ ds \leq |\Omega| \| \varphi\|_{\YY}^2.
 \]
\end{proof}

 We now want to investigate the spectral properties of the covariance operator $T_\mu$. In this example, let us consider a uniform probability distribution $\mu$ on $\mathcal{X}$, from which we know 
 \[
 T_\mu = \frac{1}{2\pi}\int_0^{2\pi} K_\theta K_\theta^* d\theta.
 \]
 Actually, in view of the isomorphism induced by $R_\varphi$ between $L^2(\Omega)$ and $\mathcal{H}$, the eigenvalues of $T_\mu$ coincide with the ones of 
 \[
  T = \frac{1}{2\pi}\int_0^{2\pi} R_\theta^* C_\varphi^* C_\varphi R_\theta \ d\theta.
 \]
 Via the Fourier slice theorem, we can easily observe that, for $u \in L^2(\Omega)$,
 \[
 T u(x) = \frac{1}{2\pi} \int_0^{2\pi}\int_{\mathbb{R}} |\hat{\varphi}(\sigma)|^2 e^{i\sigma(x\cdot \omega_\theta)}\hat{u}(\sigma \omega_\theta)d\sigma \ d \theta,
 \]
 and by the change of coordinates $\xi =\xi(\sigma,\theta) = \sigma \omega_\theta$ we have
 \[
 T u(x) = \int_{\mathbb{R}^2} \frac{|\hat{\varphi}(|\xi|)|^2}{|\xi|} e^{i\sigma(x\cdot \omega_\theta)}\hat{u}(\xi)d\xi.
 \]
 As a result, we know that $T$ is a pseudodifferential operator with symbol $\sigma_T(\xi) = \frac{|\hat{\varphi}(|\xi|)|^2}{|\xi|}$.
\begin{proposition}{(Asymptotic behavior of the effective dimension)}
\begin{itemize} 
    \item Let $\varphi = \frac{1}{2a}\chi_{[-a,a]}$ (boxcar filter). Then, the eigenvalues $\mu_k$ of $T$ have the asymptotic decay $\mu_k \sim k^{-\frac{3}{2}}$ and the effective dimension of $\mathcal{H}$ grows as $\mathcal{N}(\lambda) \sim \lambda^{-\frac{2}{3}}$ as $\lambda \rightarrow 0$. 
    \item Let $\varphi$ be a Gaussian filter. Then, the eigenvalues $\mu_k$ of $T$ decay faster than any polynomial ($\mu_k = o(k^{-\alpha})$ $\forall \alpha>0$) and the effective dimension of $\mathcal{H}$ grows slower than any power ($\mathcal{N}(\lambda) = o(\lambda^{-\frac{1}{\alpha}})$ $\forall \alpha>0$) as $\lambda \rightarrow 0$.
\end{itemize}
\label{prop:asymp}
\end{proposition}
\begin{proof}
    In the first case, we have $\hat{\varphi}(\sigma) = \frac{\sin(\sigma)}{\sigma}$, which entails that the symbol $\sigma_T(\xi) \sim |\xi|^{-3}$ for $|\xi|\rightarrow \infty$. As a consequence of the Weyl law, it is well-known (see \cite[Volume IV, Chapter 29]{hormander1983analysis}) that the sorted eigenvalues $\mu_k$ of a positive, self-adjoint, compact pseudodifferential operator of order $-m$ on bounded domain in $\mathbb{R}^d$ have an asymptotic decay of order $k^{-m/d}$ as $k \rightarrow \infty$. As a consequence, the eigenvalues of $T$ (and of $T_\mu$) decay polynomially as $k^{-3/2}$. \\
    We can now employ this result to the analysis of the effective dimension $\mathcal{N}(\lambda)$: 
    indeed, as shown in Section \ref{ssec:coefficient}, if $\mu_k \sim k^{-\frac{1}{b}}$ with $0<b<1$, then $\mathcal{N}(\lambda) \sim \lambda^{-b}$. This allows to conclude that, in this case, $\mathcal{N}(\lambda) \sim \lambda^{-\frac{2}{3}}$.
    
    In the second case, we have $\hat{\varphi}(\sigma) = e^{-c\sigma^2}$, entailing an exponential decay of the symbol $\sigma_T(\xi)\sim e^{-c|\xi|^2}$. Thus, the Weyl law can be employed for all $\alpha>0$, entailing a decay faster that any power $k^{-\alpha}$ (it is actually possible to prove that the eigenvalues reduce exponentially). As a consequence, we conclude that $\mathcal{N}(\lambda)$ grows slower than any power $\lambda^{-1/\alpha}$.
\end{proof}
As a consequence of Proposition \ref{prop:asymp}, the operator $R_\varphi$ satisfies Assumption \ref{ass:polydecay} with $b = \frac{2}{3}$ in the boxcar case, and with any $0<b<1$ in the Gaussian case. We claim that, despite the unfiltered Radon filter is not compliant with all the requested assumptions, in the context of application we may always consider the presence of a small convolution along the $s$ variable. In Remark \ref{rem:ASTRA} we also show that this phenomenon also occurs in the numerical approximation of the Radon transform. 
We carefully point out that, whenever in applications we fix a finite resolution for objects in the output space $\HH$, the covariance $T_\mu$ of the kernel operator becomes a finite-rank operator (of rank $N$), thus
\[
\mathcal{N}(\lambda) =\operatorname{tr}((T+\lambda I)^{-1}T) = \sum_{k=1}^N \frac{\mu_k}{\mu_k+\lambda} \leq N, 
\]
thus Assumption \ref{ass:polydecay} is satisfied by any $b>0$. Nevertheless, according to the size of the discretization and to the level of noise on the measurement, it might be relevant to evaluate $\lambda$ only in a range sufficiently far away from $0$. As a consequence, the behaviour of $\mathcal{N}(\lambda)$ might be interesting also at a pre-asymptotic regime, which might be similar to the asymptotic regime of the non-discretized version of the forward operator. 

\begin{remark}
\label{rem:ASTRA}
To compute the Radon transform there is a vast variety of algorithms which use different approximations of the computation for the line integral. Among the many software available, ASTRA~\cite{astra} allows to simulate the parallel beam geometry (corresponding to the Radon transform) using a so called \textit{strip} model in place of the \textit{line} one. This means that the weight of a ray/pixel pair is given by the area of the intersection of the pixel and the ray, considered as a strip with the same width as a detector pixel, effectively averaging this value. This essentially amounts to convolving the Radon transform with a boxcar filter. 
To numerically inspect the asymptotic behaviour of $\mathcal{N}(\lambda)$, we compute the SVD of the matrix associated with the filtered Radon operator (generated with the ASTRA toolbox using parallel beam geometry and the \textit{strip} model) for a $128 \times 128$ target with angular views in $[0,\pi)$. Because in such a case $A$ is already a large scale object, the SVD is computed using sketching techniques, which uses a low-rank approximation and might be slightly inaccurate for smaller singular values.
We show the decay rate in Figure~\ref{fig:SVDdecay} (red curve).
\\
As expected, due to the finite resolution of the considered discrete setting, the singular values abruptly vanish as $k$ increases. In the pre-asymptotic regime, though, a polynomial decay $k^{\beta}$ can be observed. In the unfiltered case, we expect a decay $\sigma_k \sim k^{-0.25}$, associated with the asymptotics $\mu_k \sim k^{-0.5}$ for the eigenvalues of the covariance, which would entail $\mathcal{N}(\lambda) \sim \lambda^{-2}$, where $b=2$ is not compliant with Assumption \ref{ass:polydecay}. We can observe that the slope of the red curve is slightly steeper than $\beta=-0.25$, and such a behaviour is emphasized if the binning effect is increased, namely, if the discretization of the $s$ variable gets rougher, as reported in the blue curve.
\end{remark}

\begin{figure}\label{fig:SVDdecay}
\begin{center}
\includegraphics[width=0.65\textwidth]{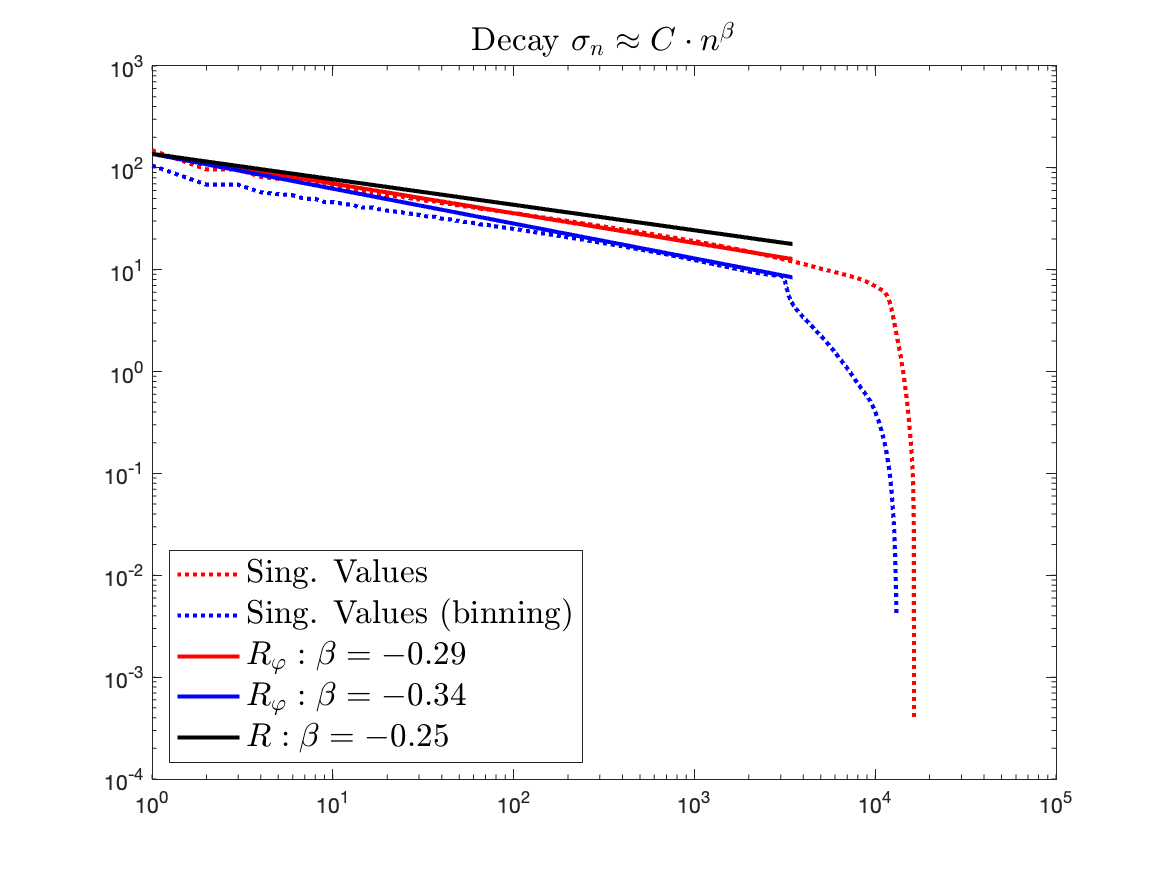} 
\end{center}
\caption{SVD decay for a $128 \times 128$ target with angular views in $[0,\pi)$ using the ASTRA toolbox using parallel beam geometry and the \textit{strip} model.}
\end{figure}

We can finally detail the decay obtained in Corollary \ref{cor:err.upper.power} for some specific examples, in which a specific decay of $\sigma_n(\fr)$ is expected. Recall that if $\sigma_n(\fr) = \mathcal{O}(n^{-\beta})$ with $\beta > \frac{1}{2}$, than, as described in \ref{ssec:smoothing}, Assumption \ref{Ass:var.sour} is satisfied with $\phi(s) = s^r$, where $r=\frac{1}{2}-\frac{1}{4\beta}$. Moreover, Corollary \ref{cor:err.upper.power} prescribes that, if $\lambda \sim n^{-\frac{1-r}{1+b-br}}$, then $\| f_{\mathbf{z},\lambda} -\fr\|_{\ell^1} \sim n^{-\frac{r}{1+b-br}}$ with high probability.
\begin{itemize}
    \item Let $S$ be a wavelet synthesis operator and assume the image $S\fr$ has bounded variation: in this case, $\beta = \frac{1}{2}$ (which implies $r=0$); thus, we do not get explicit convergence rates.
    \item Let $S$ be a shearlet synthesis operator and assume that $S\fr$ is a cartoon-like image: in this case, $\beta = 1$ (which implies $r=\frac{1}{4}$); then, if the boxcar filter is used on the Radon transform ($b = \frac{2}{3}$), choosing $\lambda \sim n^{-\frac{1}{2}}$ leads to $\| f_{\mathbf{z},\lambda} -\fr\|_{\ell^1} \sim n^{-\frac{1}{6}}$; if a Gaussian filter is used on the Radon transform ($b \approx 0$), choosing $\lambda \sim n^{-\frac{3}{4}}$ leads to $\| f_{\mathbf{z},\lambda} -\fr\|_{\ell^1} \sim n^{-\frac{1}{4}}$.
    \item Let $S$ be any synthesis operator such that image representation $\fr$ is sparse (namely, only finitely many coefficients of $\fr$ are different from $0$). In this case, $\sigma_n(\fr)$ decays faster than $n^{-\beta}$ for any $\beta>0$, which means that we can consider $r \approx \frac{1}{2}$. In the boxcar-filter case ($b = \frac{2}{3}$), choosing $\lambda \sim n^{-\frac{3}{8}}$, we get $\| f_{\mathbf{z},\lambda} -\fr\|_{\ell^1} \sim n^{-\frac{3}{8}}$; in the Gaussian filter case ($b \approx 0$), choosing $\lambda \sim n^{-\frac{1}{2}}$, we get $\| f_{\mathbf{z},\lambda} -\fr\|_{\ell^1} \sim n^{-\frac{1}{2}}$.
\end{itemize}

\subsection{Direct Learning Through a Synthesis Operator}
\label{ssec:direct}

We now examine the degenerate case $G = I$ of the operator
$A = G \circ S$ discussed in Section~\ref{ssec:smoothing}, in which $A$ reduces to the synthesis operator $S:\ell^1 \to \HH$. 
Although this setting involves no operator-induced ill-posedness, it is far from trivial. Rather, it isolates the precise role of the spectral decay of the embedding $\ip:\HH \hookrightarrow \Hu$ and shows that Assumption~\ref{Ass:A-Lip} reduces to an \emph{equality}. Consequently, it provides the cleanest possible verification of the abstract framework and serves as a direct point of comparison with classical (vector-valued) kernel ridge regression~\cite{Caponnetto07} and Lasso-type sparse learning~\cite{CandesTao07,BRT09}.

\paragraph{Diagonalizing the Embedding.}

Assume the canonical embedding $\ip:\HH\hookrightarrow\Hu$ is injective (no nonzero element of $\HH$ vanishes $\mu$-a.e.). By Assumption~\ref{ass:kernel}, $\ip$ is a Hilbert--Schmidt operator; in particular, $\ip$ is compact.  By the singular value decomposition of a compact injective operator, there exist an orthonormal basis $(\phi_{j})_{j\ge1}$ of $\HH$, an orthonormal system $(e_{j})_{j\ge1}$ in $\LL$, and numbers $w_{1}\ge w_{2}\ge\cdots>0$ with $w_{j}\to0$, such that
\begin{equation}\label{eq:svd}
   \ip\phi_{j}=w_{j}\,e_{j},
   \qquad
   \ip^{*}e_{j}=w_{j}\,\phi_{j}.
\end{equation}
Consequently, the uncentered covariance operator $\tp:=\ip^{*}\ip$ is diagonalized by the same basis,
\begin{equation}\label{eq:Tmu}
   \tp\phi_{j}=s_{j}\,\phi_{j},
   \qquad
   s_{j}:=w_{j}^{2},
\end{equation}
so that the eigenvalues of $\tp$ are precisely the squared singular values
$s_{j}=w_{j}^{2}$.

We define the
forward operator $A$ as the \emph{synthesis identity} $S$:
\begin{equation}\label{eq:syn.op}
    A = S :\ell^1 \to \HH, \qquad
    A(f) = \sum_j f_j\phi_j, \quad \text{with} \quad f=(f_j)_{j\ge1}\in\ell^1.   
\end{equation}

This is precisely the choice $G=I$ in the composite operator $A=G\;\circ \;S$
of Section~\ref{ssec:smoothing}, with $S$ the synthesis map associated with the eigenbasis $(\phi_j)$.

Using \eqref{eq:syn.op} we obtain
\begin{equation}\label{eq:recon.norm-direct}
\norm{A(f)}_{\HH}=\norm{\sum_j f_j\phi_j}_{\HH}=\paren{\sum_j f_j^2}^{\frac{1}{2}}=\norm{f}_{\ell^2} \le \norm{f}_{\ell^1}, 
\end{equation}
by orthonormality of $(\phi_j)$ in $\HH$. Hence, $A:\ell^1 \to \HH$ is bounded.

Using \eqref{eq:svd}, \eqref{eq:syn.op} and the linearity of $A$,
\begin{equation}\label{eq:prednorm-direct}
    \norm{A(f)}_{\Hu}^2
    = \norm{\ip A(f)}_{\Hu}^2
    = \Big\|\sum_j f_j\, w_j\, e_j\Big\|_{\Hu}^2
    = \sum_{j\ge1} w_j^2\, f_j^2
    = \norm{f}_{\underline w,2}^2,
\end{equation}
by orthonormality of $(e_j)$ in $\Hu$. Hence the prediction norm coincides
\emph{exactly} with the weighted norm of Section~\ref{sec:setting}, with weight sequence $w_j = \sqrt{s_j}$. 

\paragraph{Verification of the Structural Assumptions on $A$.}

Since $A$ is linear, \eqref{eq:recon.norm-direct} gives, for \emph{all}
$f,\tilde f\in\ell^1$,
\begin{equation*}
\norm{A(f)-A(\tilde f)}_{\HH}=\norm{A(f-\tilde f)}_{\HH}\le\norm{f-\tilde f}_{\ell^1}
\end{equation*}
The map $A=S$ is injective and Lipschitz, so Assumption~\ref{ass:Lipschitz} holds with  $\ella=1$. 

Since $A$ is linear, \eqref{eq:prednorm-direct} gives, for \emph{all}
$f,\tilde f\in\ell^1$,
\begin{equation}\label{eq:bilip-direct}
    \norm{A(f)-A(\tilde f)}_{\Hu} = \norm{f-\tilde f}_{\underline w,2}.
\end{equation}
Both parts of Assumption~\ref{Ass:A-Lip} therefore hold simultaneously, with the \emph{same} weight $\underline w$ and constant $L=1$. As discussed after Assumption~\ref{Ass:A-Lip}, this two-sided bound with a common weight is exactly what is needed to obtain matching upper and lower convergence rates via Corollary~\ref{cor:minimax-optimality}; here it is automatic. Further, Theorem~\ref{kt.character} implies that, whenever $\fp\in k_t$ for some $t\in(0,1)$, the variational source condition (Assumption~\ref{Ass:var.sour}) holds with
\begin{equation}\label{eq:r-direct}
    \phi(\theta)=\theta^r, \qquad r=\frac{1-t}{2-t}\in\Big(0,\tfrac12\Big).
\end{equation}

Hence, the structural assumptions on the operator A are verified. Consequently, the error bounds established in Corollary~\ref{cor:err.upper.power} remain applicable when A is the synthesis operator defined in \eqref{eq:syn.op}.

\section{Discussion}
\label{sec:discussion}

The results of this paper establish a complete statistical theory for 
$\ell^1$-regularized nonlinear inverse learning in vector-valued reproducing 
kernel Hilbert spaces. We discuss the significance of the main contributions, 
place them in the context of related literature, and outline directions for 
future work.

\paragraph{Optimality of the rates.}
The upper convergence rates of Corollary~\ref{cor:err.upper.power} and 
Theorem~\ref{thm:uniform.rate.lambda*} and the minimax lower bounds of 
Theorem~\ref{err.lower.bound.p.para.corrected} together show that the rate
\[
n^{-\frac{2r - pr + p - 1}{p(1 + b - br)}}
\]
in the interpolated norm $\|\cdot\|_{\underline{u},p}$ is minimax optimal 
over the prior class $\mathcal{P}_{r,b}$. This optimality is achieved by 
the $\ell^1$-regularized estimator~\eqref{eq:l1reg} with the explicit, 
closed-form regularization parameter $\lambda_* = n^{-(1-r)/(1+b-br)}$. 
The matching of upper and lower bounds confirms that neither the 
regularization scheme, nor the choice of $\lambda_*$ introduces unnecessary 
suboptimality: the rates reflect the fundamental information-theoretic 
difficulty of the problem as captured by $r$ and $b$.

\paragraph{Role of the smoothness index $r$.}
The source smoothness parameter $r \in (0,1)$ enters through the variational 
source condition (Assumption~\ref{Ass:var.sour}) with $\phi(t) = t^r$.
This condition is broadly verifiable: as Theorem~\ref{kt.vsc} 
shows, membership in the approximation space $k_t$ implies a variational 
source condition with $r = (1-t)/(2-t)$, and the space $k_t$ is itself 
characterized by a polynomial decay of the best $n$-term approximation errors 
(Lemma~\ref{lem:sigma.kt}). 


\paragraph{Role of the effective dimension $b$.}
The effective dimension exponent $b \in (0,1)$ encodes the polynomial spectral 
decay $\mathcal{N}(\lambda) \lesssim \lambda^{-b}$ of the covariance operator 
$T_\mu$ (Assumption~\ref{ass:polydecay}). A smaller value of $b$ indicates 
faster spectral decay, meaning that the learning problem has fewer effective 
degrees of freedom at each regularization scale. Consequently, smaller $b$ 
leads to faster convergence: the variance term $C''^2 \lambda^{-b}/n$ in the 
bounds decreases more rapidly as $\lambda \to 0$. In the CT example, the 
Gaussian filter yields super-polynomially decaying eigenvalues, so 
Assumption~\ref{ass:polydecay} holds for every $b \in (0,1)$, and in the limit 
$b \to 0$ the rate approaches $n^{-r}$. By contrast, the boxcar filter 
produces eigenvalue decay $\mu_k \sim k^{-3/2}$ (for $d=2$), leading to 
$b = 2/3$ and strictly slower rates. This illustrates a fundamental trade-off 
between spectral preservation (boxcar) and statistical efficiency (Gaussian).


\paragraph{Nonlinearity and ill-posedness.}
A key feature of the present framework is that it accommodates 
\emph{nonlinear} forward operators while retaining quantitative 
convergence guarantees. The Lipschitz condition 
(Assumption~\ref{ass:Lipschitz}) is the structural requirement on $A$; it ensures the local stability of the inverse map and enables the error-bounding arguments in Lemma~\ref{lemma:err}. 
The weighted Lipschitz property (Assumption~\ref{Ass:A-Lip}) 
further connects the sequence-space distance $\|\cdot\|_{\underline{w},2}$ 
to the prediction error $\|A(f) - A(\fp)\|_{\Hu}$, 
which is essential for the lower bound construction and the consistency results. The framework 
excludes infinitely smoothing operators (e.g., the backward heat equation 
or electrical impedance tomography), as the lower Lipschitz bound in 
Assumption~\ref{Ass:A-Lip} fails in those cases. Extending the theory 
to such severely ill-posed problems, possibly under logarithmic source 
conditions or Sobolev-type variational inequalities, is an important 
direction for future work.

\paragraph{Relation to compressed sensing and LASSO.}
In the finite-dimensional setting with $A$ linear and $\XX$ a fixed design, 
$\ell^1$-regularization reduces to the LASSO~\cite{Tibshirani96}, 
and well-known results give exact support recovery and $\ell^2$ estimation rates under restricted isometry or irrepresentability conditions. Recently, similar results have been derived also in an infinite-dimensional setting for ill-posed problems, assuming quasi-diagonalizability of the forward operator and suitable coherence bounds (\cite{alberti2025compressed,alberti2025compressed2}).
The present paper extends this program to the nonparametric, 
infinite-dimensional, nonlinear, and random-design setting, where the 
relevant complexity measure is the effective dimension rather than the 
sparsity level alone. In particular, our rates depend on both $r$ 
(solution smoothness) and $b$ (spectral complexity), and the analysis 
does not require incoherence or restricted isometry conditions, relying 
instead on the variational source condition and the probabilistic bounds 
of Proposition~\ref{main.bound}.

\paragraph{Comparison with RKHS regularization.}
The classical Tikhonov estimator with RKHS regularization, analyzed in the statistical inverse learning setting by~\cite{BlaMuc16,Rastogi20}, achieves convergence rates governed by Hilbert-space source conditions and spectral regularization theory. In contrast, the present work considers an $\ell^1$ penalty and variational source conditions formulated directly in a Banach-space setting. The resulting reconstruction rate
$n^{-r/(1+b-br)}$
depends explicitly on the sparsity parameter $r$ and the effective-dimension exponent $b$, and is minimax optimal over the model class $\mathcal P_{r,b}$. This reflects the fact that $\ell^1$ regularization exploits sparse structure in the unknown solution and therefore yields a substantially different statistical behavior than classical Hilbert-space regularization.

\paragraph{Parameter choice.}
The regularization parameter $\lambda_* = n^{-(1-r)/(1+b-br)}$ depends 
on the unknown parameters $r$ and $b$. In practice, one can estimate $b$ 
from data via spectral approximations of $T_\mu$, and $r$ can be 
calibrated using cross-validation or Lepskii-type balancing principles. 
Corollary~\ref{err.upper.bound} provides an alternative, 
dual-function-based characterization of the optimal $\lambda_*$ that 
is amenable to data-driven selection without explicit knowledge of $r$; 
this is analogous to the discrepancy principle in classical regularization 
theory. A rigorous adaptive procedure, along with finite-sample 
guarantees, would be a valuable contribution to the practical implementation 
of the method.

\paragraph{Extensions.}
Several natural extensions of the present framework merit investigation.
\begin{enumerate}[(i)]
  \item \textbf{Linearization and Fréchet derivatives.} 
    When $A$ is Fréchet differentiable, one may consider iteratively 
    reweighted $\ell^1$ schemes or proximal gradient methods, 
    and the present convergence analysis could serve as a reference 
    for analyzing each linearized step.
  \item \textbf{Online and streaming algorithms.} 
    The batch estimator~\eqref{eq:l1reg} requires access to all $n$ 
    observations simultaneously. Stochastic gradient or online 
    proximal algorithms adapted to the $\ell^1$ setting would be 
    practically important, and one would need to analyze whether the optimal rates remain achievable in an online fashion.
  \item \textbf{Stochastic noise models.} 
    The present analysis assumes sub-Gaussian noise. 
    Extending to heavier-tailed or correlated noise, or to Gaussian 
    white noise in infinite-dimensional output spaces (currently excluded), 
    would broaden the applicability of the theory.


\end{enumerate}



\section{Conclusion}

We developed a statistical learning theory for $\ell^1$-regularized nonlinear inverse problems in vector-valued reproducing kernel Hilbert spaces. Under variational source conditions and polynomial effective-dimension growth, we established almost-sure consistency, non-asymptotic high-probability convergence rates, and matching minimax lower bounds. The resulting rates are shown to be minimax optimal over a broad family of probability distributions $\PP_{r,b}$ characterized by the smoothness parameter $r$ and the effective-dimension exponent $b$.

A further contribution of the paper is the connection between sparse approximation theory and statistical inverse learning. Through the approximation spaces $k_t$, we showed that polynomial decay of best $n$-term approximation errors implies the variational source conditions required for the statistical analysis, thereby linking sparsity models directly to convergence-rate exponents. Applications to coefficient identification problems and sparse computed tomography demonstrate that the abstract assumptions can be verified in concrete inverse problems and lead to explicit convergence guarantees.

These results provide a rigorous bridge between deterministic sparsity regularization, statistical inverse learning, and vector-valued kernel methods, and establish a general framework for analyzing sparse nonlinear inverse problems under random sampling.

\appendix

\section{Probabilistic Upper Bound}

In this section, we present high-probability estimates that characterize the stochastic perturbation behavior of empirical quantities arising in the regularized learning framework. The results stated below are standard in the analysis of statistical inverse problems and are derived from \citep{Rastogi20}. These probabilistic upper bounds play a key role in controlling deviations between empirical and population-level operators under random sampling.

\begin{proposition}\label{main.bound}
Suppose Assumptions~\ref{ass:fp}--\ref{ass:kernel} hold. Then, for any~$n \in \NN$ and~$0 < \eta < 1$, each of the following inequalities holds with confidence at least~$1 - \eta$:
\begin{equation*}
\Xi_{\zz} := \norm{\sx^*\VE}_{\Ht} 
\leq 2\!\left(\frac{\kappa M}{n} + \sqrt{\frac{\kappa^2 \Sigma^2}{n}}\right)\!
\log\!\left(\frac{2}{\eta}\right),
\end{equation*}
\begin{equation*}
\Theta_{\zz} := \norm{(\tp + \la I)^{-1/2} \sx^* \VE}_{\Ht} 
\leq 2\!\left(\frac{\kappa M}{n\sqrt{\la}} + \sqrt{\frac{\Sigma^2 \mathcal{N}(\la)}{n}}\right)\!
\log\!\left(\frac{2}{\eta}\right),
\end{equation*}
and
\begin{equation*}
\Psi_{\xx} := \norm{(\tp + \la I)^{-1/2} (\tx - \tp)}_{\mathcal{L}(\Ht)} 
\leq 2\!\left(\frac{\kappa^2}{n\sqrt{\la}} + \sqrt{\frac{\kappa^2 \mathcal{N}(\la)}{n}}\right)\!
\log\!\left(\frac{2}{\eta}\right).
\end{equation*}
\end{proposition}

\begin{corollary}\label{cor:psi.bound}
Suppose Assumptions~\ref{ass:fp}--\ref{ass:kernel} and condition~\eqref{l.la.condition} hold. Then, the following bound holds with confidence at least~$1 - \eta$:
\begin{equation}\label{Psi.bound}
\Psi_{\xx} + \Theta_{\zz} \;\lesssim\; \sqrt{\frac{\la^{-b}}{n}}\log\!\left(\frac{4}{\eta}\right).
\end{equation}
\end{corollary}

\begin{proof}
Since~$\mathcal{N}(\la)$ is a decreasing function of~$\la$ and~$\la \leq 1$, the condition~\eqref{l.la.condition} implies that
\begin{equation*}
\mathcal{N}(1) \leq \mathcal{N}(\la) \leq n\la.
\end{equation*}
Consequently,
\begin{equation*}
\frac{1}{n\sqrt{\la}} 
\leq \frac{1}{n\sqrt{\la}}\frac{\mathcal{N}(\la)}{\mathcal{N}(1)}
= \frac{1}{\mathcal{N}(1)} 
\sqrt{\frac{\mathcal{N}(\la)}{n\la}} 
\sqrt{\frac{\mathcal{N}(\la)}{n}} 
\leq \frac{1}{\mathcal{N}(1)} \sqrt{\frac{\mathcal{N}(\la)}{n}}.
\end{equation*}
Substituting this bound into the inequalities of Proposition~\ref{main.bound}, we obtain with probability at least~$1 - \eta$:
\begin{align*}
\Theta_{\zz} 
&\leq 2\!\left(\frac{\kappa M}{\mathcal{N}(1)} + \Sigma\right)\!
\sqrt{\frac{\mathcal{N}(\la)}{n}}\log\!\left(\frac{4}{\eta}\right),\\[4pt]
\Psi_{\xx} 
&\leq 2\!\left(\frac{\kappa^2}{\mathcal{N}(1)} + \kappa\right)\!
\sqrt{\frac{\mathcal{N}(\la)}{n}}\log\!\left(\frac{4}{\eta}\right).
\end{align*}
Combining these bounds yields inequality~\eqref{Psi.bound} with
\begin{equation*}
C' = 2\!\left[\!\left(\frac{\kappa^2}{\mathcal{N}(1)} + \kappa\right)
+ \left(\frac{\kappa M}{\mathcal{N}(1)} + \Sigma\right)\!\right].
\end{equation*}
\end{proof}

The above results establish uniform probabilistic control over the empirical quantities~$\Xi_{\zz}$, $\Theta_{\zz}$, and~$\Psi_{\xx}$. These estimates are instrumental in deriving high-probability error bounds for the regularized estimator and in analyzing the stability of the inverse learning scheme.

\section{Boundedness of the Regularized Solution}

In this section, we establish that the regularized solution~$\fz$ remains bounded in the hypothesis space~$\Ho$. Specifically, we show that the deviation~$\norm{\fz - \fp}_{\Ho}$ between the regularized estimator~$\fz$ and the true function~$\fp$ is upper bounded by a constant multiple of~$\norm{\fp}_{\Ho}$ under suitable conditions.  

\begin{proposition}\label{prop:bounded.solution}
Suppose Assumptions~\ref{ass:fp}--\ref{ass:Lipschitz} hold. Then, for any~$n \in \NN$ and~$0 < \eta < 1$, the following bound holds with confidence at least~$1 - \eta$:
\begin{equation*}
\norm{\fz - \fp}_{\Ho} \leq 4 \norm{\fp}_{\Ho},
\end{equation*}
provided that
\begin{equation*}
8 \ella \kappa (M + \Sigma) \frac{1}{\la \sqrt{n}} \log\!\left(\frac{4}{\eta}\right) \leq 1.
\end{equation*}
\end{proposition}

\begin{proof}
By the Lipschitz continuity of the nonlinear operator~$A$ (Assumption~\ref{ass:Lipschitz}), we have for all~$f \in \DD(A)\cap \Ho$,
\begin{equation*}
\norm{A(f) - A(\fp)}_{\Ht} \leq \ella \norm{f - \fp}_{\Ho}.
\end{equation*}

Using this property in inequality~\eqref{ineq:fz-bound}, we obtain
\begin{equation}\label{eq:fl}
\norm{\fz}_{\Ho} \leq \frac{2\ella}{\la} \norm{\fz - \fp}_{\Ho} \norm{\sx^* \VE}_{\Ht} + \norm{\fp}_{\Ho}.
\end{equation}

Applying the triangle inequality yields
\begin{align}\label{fz.fl.bd}
\norm{\fz - \fp}_{\Ho}
&\leq \norm{\fz}_{\Ho} + \norm{\fp}_{\Ho} \nonumber \\
&\leq \frac{2\ella}{\la} \norm{\sx^* \VE}_{\Ht} \norm{\fz - \fp}_{\Ho} + 2 \norm{\fp}_{\Ho}.
\end{align}

From Proposition~\ref{main.bound}, we know that with confidence~$1 - \eta$,
\begin{equation*}
\norm{\sx^* \VE}_{\Ht} \leq 2 \kappa (M + \Sigma) \frac{1}{\sqrt{n}} \log\!\left(\frac{4}{\eta}\right).
\end{equation*}

Substituting this into the previous inequality gives
\begin{align*}
\norm{\fz - \fp}_{\Ho}
&\leq 4 \ella \kappa (M + \Sigma) \frac{1}{\la \sqrt{n}} \log\!\left(\frac{4}{\eta}\right) \norm{\fz - \fp}_{\Ho}
+ 2 \norm{\fp}_{\Ho}.
\end{align*}

If the condition
\begin{equation*}
8 \ella \kappa (M + \Sigma) \frac{1}{\la \sqrt{n}} \log\!\left(\frac{4}{\eta}\right) \leq 1
\end{equation*}
is satisfied, then the first term on the right-hand side can be absorbed into the left-hand side, yielding
\begin{equation*}
\frac{1}{2} \norm{\fz - \fp}_{\Ho} \leq 2 \norm{\fp}_{\Ho},
\end{equation*}
which simplifies to the desired bound
\begin{equation*}
\norm{\fz - \fp}_{\Ho} \leq 4 \norm{\fp}_{\Ho}.
\end{equation*}
\end{proof}

The result confirms that the regularized estimator~$\fz$ remains bounded in~$\Ho$, provided that the regularization parameter~$\la$ and sample size~$n$ satisfy a mild condition. This boundedness plays a key role in ensuring the stability and consistency of the regularized inverse learning algorithm.

\section{Interpolation Scale of Weighted Sequence Spaces}

We begin by introducing a scale of spaces that interpolates between the classical 
$\ell^1$ space and the weighted Hilbert space $\ell^2_{\underline{w}}$ appearing in our setting. 
For $p \in (0,2]$, we define the weights
\begin{equation*}
(\underline{u}_p)_j := w_j^{\frac{2p-2}{p}}.
\end{equation*}
The corresponding weighted sequence space $\ell^p_{\underline{u}_p}$ consists of all 
$f = (f_j)_{j \ge 1}$ such that
\[
\|f\|_{\underline{u}_p, p}
:= \left( \sum_{j=1}^\infty (\underline{u}_p)_j^{\,p} |f_j|^p \right)^{1/p}
< \infty.
\]

The next result provides an interpolation inequality that connects these spaces.

\begin{proposition}[Interpolation inequality]\label{prop:interpolation}
Let $p, q, s \in (0,2]$ and $\theta \in (0,1)$ satisfy
\[
\frac{1}{p} = \frac{1-\theta}{q} + \frac{\theta}{s}.
\]
Then, for all $f \in \ell^q_{\underline{u}_q} \cap \ell^s_{\underline{u}_s}$, the following inequality holds:
\begin{equation}\label{eq:interp.ineq}
\|f\|_{\underline{u}_p, p}
\le 
\|f\|_{\underline{u}_q, q}^{\,1-\theta}
\|f\|_{\underline{u}_s, s}^{\,\theta}.
\end{equation}
\end{proposition}

\begin{proof}
We apply Hölder’s inequality with conjugate exponents 
$\frac{q}{(1-\theta)p}$ and $\frac{s}{\theta p}$.  
Indeed, we have
\[
\|f\|_{\underline{u}_p, p}^p
= \sum_{j=1}^\infty w_j^{\,2p-2}\,  |f_j|^p
\le 
\left( \sum_{j=1}^\infty w_j^{\,2q-2}\,  |f_j|^q \right)^{\frac{(1-\theta)p}{q}}
\left( \sum_{j=1}^\infty w_j^{\,2s-2}\,  |f_j|^s \right)^{\frac{\theta p}{s}}.
\]
Taking the $p$-th root on both sides yields the desired result~\eqref{eq:interp.ineq}.
\end{proof}

\bibliographystyle{plain}
\bibliography{bib_ell1}
\end{document}